\documentclass{article}
\usepackage{arxiv_style,times}

\usepackage{amsmath,amsfonts,bm}

\def\eqref#1{Eq.~\ref{#1}}

\def\1{\bm{1}}

\DeclareMathAlphabet{\mathsfit}{\encodingdefault}{\sfdefault}{m}{sl}
\SetMathAlphabet{\mathsfit}{bold}{\encodingdefault}{\sfdefault}{bx}{n}

\usepackage{amsthm}
\usepackage{algorithm}
\usepackage{algpseudocode}
\usepackage{hyperref}
\usepackage{url}
\usepackage{xcolor}
\newcommand{\rev}[1]{{\color{black}#1}}

\newtheorem{theorem}{Theorem}
\newtheorem{assumption}{Assumption}
\newtheorem{lemma}{Lemma}
\newtheorem{corollary}{Corollary}
\newtheorem{remark}{Remark}

\title{Constrained Bilevel Reinforcement Learning via Approximate Lower-Level Responses}

\author{
  Naman Saxena \\
  Purdue University \\
  West Lafayette, IN 47907, USA \\
  \texttt{saxen147@purdue.edu} \\
  \And
  Vaneet Aggarwal \\
  Purdue University \\
  West Lafayette, IN 47907, USA \\
  \texttt{vaneet@purdue.edu} \\
}

\iclrfinalcopy 

\begin{document}

\maketitle

\begin{abstract}
\rev{We study constrained bilevel reinforcement learning in which the
lower-level problem is a constrained Markov decision process with a generally
parameterized policy. The lower-level solution map can be set-valued and need
not be Lipschitz continuous in the upper-level variable, so standard
hypergradient analyses do not apply. We therefore define the upper-level
objective as the expected upper-level loss of the policy returned by a fixed
stochastic lower-level algorithm, which we call the approximate response. This
objective is well defined without uniqueness of the lower-level solution. To
compute its gradient, we randomize the reward and constraint values used during
lower-level training without changing their means, so that the upper-level
variable affects training only through known probabilities. This yields an
unbiased estimator of the upper-level objective's gradient that does not
differentiate through the lower-level updates. We establish lower-level reward
and feasibility guarantees, upper-level stationarity, and a total sample
complexity of $\widetilde{\mathcal{O}}(\delta_\ell^{-6}\epsilon_o^{-2})$ over
both levels, where $\delta_\ell$ is the lower-level accuracy and $\epsilon_o$ is
the upper-level stationarity error. Setting
$\delta_\ell=\epsilon_o=\epsilon$ gives
$\widetilde{\mathcal{O}}(\epsilon^{-8})$.}
\end{abstract}

\section{Introduction}
\label{sec:introduction}

Bilevel optimization provides a natural framework for learning and
decision-making problems in which an upper-level decision influences a nested
learning or control problem. In reinforcement learning (RL), this structure
arises in hyperparameter optimization, meta-learning, reward design, inverse
RL, and preference-based learning. We study the more challenging setting of
constrained bilevel RL, where the lower-level policy must additionally satisfy
safety, resource, fairness, or performance constraints and the upper-level
variable modifies the reward or constraint specification of a lower-level
constrained Markov decision process (CMDP).
\rev{This setting covers several practical problems. In safe alignment from
human feedback, the upper level fits reward and safety-cost models to human
preferences evaluated on the policy that maximizes the learned reward subject
to a cost budget, combining policy-aligned reward learning
\citep{chakraborty2024parl} with the constrained formulation of Safe RLHF
\citep{dai2024safe}. In inverse constrained RL, the upper level infers an
unknown constraint so that the resulting constrained-optimal policy matches
expert demonstrations \citep{scobee2020maximum,malik2021inverse}. In constrained reward design, the upper level chooses reward parameters to
improve a designer-specified objective, while the lower-level agent optimizes
the resulting reward subject to fixed safety or resource constraints. Appendix~\ref{app:examples} states each example in our
notation.}

\rev{In each of these problems, computing the upper-level gradient} is
particularly difficult with generally parameterized policies because \rev{it
requires} \rev{the lower-level solution map to be Lipschitz continuous in the
upper-level variable}. Existing nonconvex bilevel
analyses obtain hypergradient formulas under Polyak--\L{}ojasiewicz \rev{(PL)} conditions,
error bounds, or related \rev{regularity} assumptions
\citep{chen2024finding,kwon2024on}. Such regularity is difficult to establish
in parameterized constrained RL: the primal optimizer and Lagrange multiplier
may be nonunique, \rev{the solution map need not be Lipschitz continuous in the upper-level
variable}, and differentiating the \rev{Karush--Kuhn--Tucker (KKT)} or primal--dual system requires strong
regularity of the constrained solution map.

We therefore replace the exact constrained optimizer with \rev{an approximately
optimal and approximately feasible lower-level policy, characterized directly
by its expected} reward suboptimality and constraint violation. We \rev{compute this
policy} using a finite-sample stochastic primal--dual natural policy-gradient
\rev{algorithm} whose natural-gradient direction is approximated by vanilla
stochastic gradient descent \rev{(SGD)}. \rev{The returned policy parameter} is random
because of trajectory sampling, stochastic gradient estimation, and random
iterate selection; this randomness belongs to the \rev{lower-level algorithm} rather
than the bilevel formulation itself.

To differentiate the \rev{expected upper objective at the policy returned by the lower-level algorithm}, we introduce a \rev{randomized binary encoding, which replaces} every sampled reward and constraint signal used during lower-level training \rev{by a random variable taking two values}. This encoding preserves the conditional means required by the lower-level algorithm while transferring the dependence on the upper-level parameter to the \rev{lower-level training distribution, i.e., the probability distribution} of the complete stochastic training trace. Once the \rev{training} trace is fixed, the returned policy parameter is fixed as well. Applying the standard score-function identity \citep{schulman2015gradient} to this \rev{training distribution}, Lemma~\ref{lem:training-gradient} (in Appendix) derives an upper gradient without differentiating through the lower-level policy, critic, and dual iterates. To make this hypergradient usable for finite-sample optimization, Lemma~\ref{lem:upper-smoothness} (in Appendix) proves that \rev{this expected upper objective at the returned lower-level policy} is smooth, while Lemma~\ref{lem:upper-estimator} (in Appendix) shows that the hypergradient admits an implementable unbiased estimator with controlled second moment. Together, these results turn the \rev{randomized binary encoding} construction into a tractable stochastic first-order method for the induced upper objective, without requiring differentiability of an exact constrained lower-level solution map.

Our main contributions are summarized as follows.
\begin{itemize}
    \item \textbf{Approximate constrained bilevel formulation.} We formulate
    constrained bilevel RL through \rev{an approximately optimal and approximately
    feasible lower-level policy returned by a stochastic lower-level algorithm}, characterized directly by reward suboptimality and constraint violation,
    avoiding the need to differentiate an exact constrained solution map.


    \item \textbf{\rev{Randomized binary encoding} and \rev{unbiased upper-level gradient}.}
    \rev{We replace each sampled reward and constraint value used during lower-level
    training by a random variable taking two values with the same mean, so the
    lower-level algorithm still receives unbiased signals while the upper-level
    variable affects training only through known probabilities. A likelihood-ratio
    identity then gives an unbiased estimator of the upper-level gradient without
    differentiating through the lower-level updates,} and requires neither a
    lower-level PL condition nor inverse-Hessian computation.

    \item \textbf{\rev{Sample complexity.}} We establish upper-level stationarity
    together with lower-level reward \rev{and feasibility} guarantees, \rev{with a
    total sample complexity of
    $\widetilde{\mathcal{O}}(\delta_\ell^{-6}\epsilon_o^{-2})$ over both levels,
    where $\delta_\ell$ is the lower-level accuracy and $\epsilon_o$ is the
    upper-level stationarity error}. It becomes
    $\widetilde{\mathcal{O}}(\epsilon^{-8})$ when $\delta_\ell=\epsilon_o=\epsilon$.
\end{itemize}

\section{Related Work}
\label{sec:related-work}

\subsection{Bilevel Optimization and Reinforcement Learning}

Bilevel methods obtain upper gradients through implicit differentiation with
Hessian or inverse-Hessian information
\citep{ghadimi2018approximation,chen2021closing,chen2022single,
yang2021provably}, or through iterative differentiation \rev{(ITD)} of a finite lower
trajectory \citep{ji2021bilevel}. Hessian-free, fully first-order,
penalty-, and value-function-based methods reduce or avoid this computation
\citep{li2022fully,sow2022convergence,kwon2023fully,liu2022bome,
liu2021value,shen2023penalty,kwon2024on}. These ideas have been extended to
bilevel RL and \rev{RL from human feedback (RLHF)} through stochastic approximation, actor--critic, penalty,
and sample-complexity analyses
\citep{hong2023two,shen2025principled,NEURIPS2024_e66309ea,
gaur2025sample,yang2025bilevel,chakraborty2024parl}, but primarily for
unconstrained lower-level policy optimization. Appendix~\ref{sc:related}
provides the detailed literature discussion.

\subsection{Constrained Bilevel Optimization}

Constrained bilevel methods use implicit hypergradients, projected stochastic
updates, Clarke-subdifferential approximations, and first-order or
proximal-Lagrangian reformulations
\citep{tsaknakis2022implicit,xiao2023alternating,xu2023efficient,
kornowski2024first,yao2024constrained,yao2025overcoming}. Their nonasymptotic
guarantees generally retain convexity or solution-map regularity; the
nonconvex value-function approach of \citet{liu2023value} is asymptotic. In
constrained RL, nonuniqueness and changing active constraints obstruct implicit
hypergradients. ITD instead differentiates through the primal--dual trajectory,
including projected dual updates, tying the derivative to that unrolled
procedure. \rev{We instead control the lower-level policy through reward and
feasibility guarantees and differentiate the training distribution of the
algorithm that produces it}.

\section{Problem Formulation}
\label{sec:problem-formulation}

\subsection{Reinforcement Learning and Constrained MDPs}

We consider a discounted constrained Markov decision process written as the
tuple $\mathcal M_v=(\mathcal S,\mathcal A,P,\rho,\gamma,r_x,c_y)$, where
$r_x(s,a):=r(s,a;x)$ and $c_y(s,a):=c(s,a;y)$. Here, $\mathcal S$ is the state
space, $\mathcal A$ is the action space, $P$ is the transition kernel, $\rho$
is the initial-state distribution, and $\gamma\in[0,1)$ is the discount
factor. The upper-level variable is $v=(x,y)\in\mathbb R^{d_v}$, where
$x\in\mathbb R^{d_x}$ parameterizes the reward function and
$y\in\mathbb R^{d_y}$ parameterizes the constraint function. Importantly,
$P$ and $\rho$ do not depend on $v$. For a stationary randomized policy
$\pi$, the discounted reward and constraint returns are defined jointly in
\eqref{eq:returns}.
\begin{equation}
\label{eq:returns}
J_r(\pi,x):=\mathbb E_\pi\!\left[\sum_{t=0}^{\infty}\gamma^t
r(s_t,a_t;x)\right],\qquad
J_c(\pi,y):=\mathbb E_\pi\!\left[\sum_{t=0}^{\infty}\gamma^t
c(s_t,a_t;y)\right].
\end{equation}

For a fixed upper-level parameter $v=(x,y)$, the lower-level CMDP is
$J_r^\star(v):=\max_\pi\{J_r(\pi,x):J_c(\pi,y)\geq0\}$. We optimize over a
differentiable parameterized policy class
$\{\pi_\theta:\theta\in\mathbb R^{d_\theta}\}$, which is not assumed to
contain an exact optimizer of the CMDP. For brevity, write
$J_r(\theta,x):=J_r(\pi_\theta,x)$ and
$J_c(\theta,y):=J_c(\pi_\theta,y)$.

\subsection{Approximate Constrained Bilevel Reinforcement Learning}

\rev{Let $G(v,\theta)$ denote the upper-level objective. The exact constrained
bilevel problem is given in \eqref{eq:exact-bilevel}.
\begin{equation}
\label{eq:exact-bilevel}
\min_{v\in\mathbb R^{d_v}} G\big(v,\theta^\star(v)\big)
\quad\text{s.t.}\quad
\theta^\star(v)\in\mathcal S(v):=\arg\max_{\theta\in\mathbb R^{d_\theta}}
\big\{J_r(\theta,x):J_c(\theta,y)\geq 0\big\}.
\end{equation}}
\rev{The objective in \eqref{eq:exact-bilevel} depends on $v$ through the
solution map $v\mapsto\mathcal S(v)$.} For parameterized constrained RL, this
\rev{map} can be \rev{set-valued and need not be Lipschitz continuous in $v$}.
Rather than imposing \rev{such regularity on} this exact solution map, we
define the bilevel problem through \rev{the approximately optimal and
approximately feasible policy returned by a stochastic lower-level algorithm}.

For each upper-level parameter $v$, let $Z$ denote the complete random trace
generated by \rev{one execution of the lower-level
algorithm constructed in Section~\ref{sec:approach}}, including all training
randomness and the \rev{random selection of the returned iterate. We call the probability distribution
of this trace the \emph{lower-level training distribution} and write
$Z\sim P_v$; for brevity, we refer to it as the training distribution
hereafter. The training distribution} depends on $v$, whereas $\theta(Z)$ \rev{is}
the policy parameter returned \rev{by the lower-level run, which is determined by the
realized training trace $Z$} and has no additional explicit dependence on $v$. \rev{We
refer to $\theta(Z)$ as the \emph{approximate lower-level response}, or the
approximate response for brevity.} Thus the
lower-level dependence \rev{on $v$ enters only through $Z\sim P_v$}, rather than
\rev{through} a deterministic differentiable response map. For every
$v\in\mathbb R^{d_v}$, we characterize this \rev{approximate} response through the
reward-suboptimality and feasibility requirements in
\eqref{eq:approx-reward} and \eqref{eq:approx-feasibility}. Here, $\epsilon_\ell$ is \rev{a response accuracy}\rev{, instantiated for our lower-level algorithm in Section~\ref{sec:analysis},} and $C_r,C_c>0$ are
constants.
\begin{align}
\label{eq:approx-reward}
\mathbb E_{Z\sim P_v}\left[
J_r^\star(v)-J_r\left(\pi_{\theta(Z)},x\right)
\right]
&\leq C_r\epsilon_\ell,
\\
\label{eq:approx-feasibility}
\mathbb E_{Z\sim P_v}\left[
-J_c\left(\pi_{\theta(Z)},y\right)
\right]
&\leq C_c\epsilon_\ell,
\end{align}
The approximate constrained
bilevel problem induced by this \rev{approximate} response is defined in
\eqref{eq:approx-bilevel}.
\begin{equation}
\label{eq:approx-bilevel}
\min_{v\in\mathbb R^{d_v}}
\Phi_{\epsilon_\ell}(v)
:=
\mathbb E_{Z\sim P_v}
\left[G(v,\theta(Z))\right].
\end{equation}
\rev{Accordingly, $v$ affects the upper objective directly through $G(v,\cdot)$ and
indirectly through the training distribution $P_v$ of the trace $Z$ that
determines $\theta(Z)$.}
The requirements in \eqref{eq:approx-reward}--\eqref{eq:approx-feasibility}
characterize the lower-level response and do not restrict the feasible domain
of the upper-level variable $v$.

The formulation does not require the \rev{lower-level algorithm} and parameterized
policy class to attain every arbitrarily prescribed $\epsilon_\ell$;
\rev{Section~\ref{sec:analysis} characterizes the response accuracy attainable by
the lower-level algorithm, including an approximation error determined by the
policy class. We next construct the lower-level algorithm} and the
corresponding upper-level gradient estimator.
\section{Proposed Approach}
\label{sec:approach}

\rev{This section constructs the lower-level algorithm that produces the
approximate response, together with an upper-level gradient for the induced
objective. We first introduce a \rev{randomized binary encoding} of the sampled reward and
constraint signals, and then specify the lower-level algorithm that operates on
the encoded signals. We show that this encoding allows the upper gradient to be
obtained from the score of the training distribution, rather than from
differentiation through the lower-level algorithm iterates.}

\subsection{\texorpdfstring{\rev{Randomized Binary Encoding}}{Randomized Binary Encoding} of Reward and Constraint Signals}
\label{subsec:two-point-encoding}

\begin{assumption}[Upper-level signal regularity]
\label{ass:main-outer-signal}
\label{ass:outer-signal}
There exist finite $R,C,R_1,C_1,R_2,C_2>0$ such that, uniformly in the
arguments, the bounds in \eqref{eq:main-outer-signal-bounds} hold.
\begin{equation}
\label{eq:main-outer-signal-bounds}
\begin{gathered}
    |r(s,a;x)|\leq R,\quad |c(s,a;y)|\leq C,\quad
    \|\nabla_xr(s,a;x)\|\leq R_1,\quad
    \|\nabla_yc(s,a;y)\|\leq C_1,\\
    \|\nabla_{xx}^2r(s,a;x)\|\leq R_2,\qquad
    \|\nabla_{yy}^2c(s,a;y)\|\leq C_2.
\end{gathered}
\end{equation}
\end{assumption}
Boundedness and smooth dependence of parameterized reward and constraint
signals on the upper-level variables are standard regularity conditions in
bilevel and configurable reinforcement-learning analyses
\citep{NEURIPS2024_e66309ea,shen2025principled,
gaur2025sample}. The bounded signals make the \rev{encoding introduced below} well
defined, while the uniform first- and second-order bounds control \rev{its
score and the score's derivatives}.
Appendix~\ref{app:outer-signal-assumption-justification} and
Remark~\ref{rem:normalized-signal-architecture} provide a concrete normalized
function class satisfying these conditions.

Direct differentiation \rev{of the upper objective} would require differentiating
\rev{through the iterates of the lower-level algorithm in
Section~\ref{subsec:lower-level-npg}}. We instead encode the sampled scalar
signals and differentiate the \rev{training distribution $P_v$}, avoiding
differentiation through \rev{these iterates}. For any scalar signal $a(v)$ satisfying
$|a(v)|\leq A$, generate $B\in\{-1,+1\}$ according to the distribution in
\eqref{eq:encoding-probability}.
\begin{equation}
\label{eq:encoding-probability}
    \mathbb P_v(B=b)=\frac12+\frac{b\,a(v)}{4A},
    \qquad b\in\{-1,+1\}.
\end{equation}
The corresponding decoded signal is $\widetilde a=2AB$, which satisfies
$\mathbb E_v[\widetilde a]=a(v)$. The score of the encoded signal is
given in \eqref{eq:encoding-score}.
\begin{equation}
\label{eq:encoding-score}
    s_a(B;v)=\nabla_v\log\mathbb P_v(B)
    =\frac{B\nabla_v a(v)}{2A+B a(v)}.
\end{equation}
We conditionally independently encode every sampled reward and constraint signal used by the
lower-level algorithm \rev{(Section~\ref{subsec:lower-level-npg})}. Under Assumption~\ref{ass:main-outer-signal}, the
decoded signals are
$\widetilde r(s,a;x)=2RB^r$ and $\widetilde c(s,a;y)=2CB^c$, where $B^r$
and $B^c$ are generated by \eqref{eq:encoding-probability} with $A=R$ and
$A=C$, respectively.
For the $j$-th encoded signal, we denote the corresponding encoding
signs by $B_j^r$ and $B_j^c$.
Because the decoded signals are unbiased, this transformation does not change
the conditional means of the stochastic critic and constraint estimators used
by the lower-level algorithm.

\subsection{Lower-Level Primal-Dual Natural Policy Gradient}
\label{subsec:lower-level-npg}

\begin{assumption}[Uniform Slater condition]
\label{ass:main-uniform-slater}
\label{ass:slater-lower}
There exist $c_0:\mathbb R^{d_v}\to(0,\infty)$ and a constant
$\underline c>0$, independent of $v$, such that every $v=(x,y)$ admits a
stationary policy $\pi_{\mathrm{sl}}(v)$ satisfying
$J_c(\pi_{\mathrm{sl}}(v),y)\geq c_0(v)\geq\underline c$.
\end{assumption}
Strict feasibility through a Slater condition is standard in primal--dual
analyses of CMDPs and is used to control the optimal dual multiplier
\citep{bai2023zero,mondal2024constrained,mondal2026last}. Because the
constraint function in our bilevel formulation depends on the upper-level
parameter $y$, we require the Slater margin to hold uniformly over $y$\rev{, because the upper-level parameter changes during optimization}. This
uniform margin permits a fixed interval $\Lambda=[0,\lambda_{\max}]$, with
$\lambda_{\max}\geq4R/((1-\gamma)\underline c)$, for every $v$;
Appendix~\ref{app:slater-assumption-justification} gives the multiplier bound and a sufficient
construction.

For a fixed upper-level parameter $v=(x,y)$, define
$J_L(\theta,\lambda;v):=J_r(\theta,x)+\lambda J_c(\theta,y)$ for
$\lambda\in\Lambda$. Write
$\psi_\theta(s,a):=\nabla_\theta\log\pi_\theta(a\mid s)$. Let $d_\rho^\pi$
denote the normalized discounted state-occupancy measure of a policy $\pi$, let
$\nu_\rho^\pi(s,a):=d_\rho^\pi(s)\pi(a\mid s)$ denote its state--action
occupancy measure, and let $A_r^\pi$ and $A_c^\pi$ denote its reward and
constraint advantages, respectively. Define the Lagrangian advantage by
$A_{L,\lambda}^{\pi_\theta}:=A_r^{\pi_\theta}+\lambda A_c^{\pi_\theta}$.
With \rev{the} discounted Fisher matrix $F_\rho(\theta)\rev{:=\mathbb E_{(s,a)\sim\nu_\rho^{\pi_\theta}}[\psi_\theta(s,a)\psi_\theta(s,a)^\top]}$, the natural-gradient direction
$\omega_{\theta,\lambda}^\star:=F_\rho(\theta)^{-1}
\nabla_\theta J_L(\theta,\lambda;v)$ equivalently minimizes the compatible
regression objective
\begin{equation}
\label{eq:critic-objective}
\mathcal L_\rho(\omega;\theta,\lambda,v)
:=\frac12\mathbb E_{(s,a)\sim\nu_\rho^{\pi_\theta}}
\left[
\left(
\psi_\theta(s,a)^\top\omega
-\frac{A_{L,\lambda}^{\pi_\theta}(s,a;v)}{1-\gamma}
\right)^2
\right].
\end{equation}

The critic estimator uses reset access\rev{: the simulator can start a trajectory from $\rho$, reset to a state visited by that trajectory, and from there} generate the independent
continuations required by Appendix Algorithm~\ref{alg:encoded-critic}.
At lower iteration $k$, initialize $\omega_{k,0}=0$ and apply $N$ vanilla-SGD
critic steps using the encoded stochastic gradient:
\begin{equation}
\label{eq:critic-sgd}
    \omega_{k,n+1}=\omega_{k,n}-\beta\widehat\nabla_\omega
    \mathcal L_\rho(\omega_{k,n};\theta_k,\lambda_k,v),
    \qquad n=0,\ldots,N-1.
\end{equation}
With $\widehat\omega_k:=\omega_{k,N}$, the primal and dual updates are
$\theta_{k+1}=\theta_k+\eta\widehat\omega_k$ and
$\lambda_{k+1}=\Pi_{[0,\lambda_{\max}]}
[\lambda_k-\zeta\widehat J_{c,\rho}(\theta_k;y)]$\rev{, where $\widehat J_{c,\rho}(\theta_k;y)$ is an independent encoded estimate of the constraint return $J_c(\theta_k,y)$ from Appendix Algorithm~\ref{alg:encoded-critic}}. After $K$ iterations,
$R\sim\mathrm{Unif}\{0,\ldots,K-1\}$ is sampled independently and the
algorithm returns the approximate response
$\widehat\theta=\theta_R=\theta(Z)$. Appendix
Algorithms~\ref{alg:lower-pd-npg-sgd}--\ref{alg:encoded-critic} give the
complete lower-level \rev{algorithm}, encoded oracle, score accumulation, and
sampling details.

\subsection{\texorpdfstring{\rev{Training-Distribution Score}}{Training-Distribution Score}}

Recall that $Z\sim P_v$ denotes the complete random trace of \rev{one lower-level}
\rev{run}, including all environment trajectories, encoded
reward and constraint signals, critic samples, and the final random
iterate selection. The \rev{training-distribution score}, formalized in
Lemma~\ref{lem:score-moments} (in Appendix), is defined in
\eqref{eq:complete-score}.
Let $\mathcal I_r(Z)$ and $\mathcal I_c(Z)$ \rev{index every instance in which a
sampled reward $r(s,a;x)$ or constraint signal $c(s,a;y)$ is replaced by its
randomized encoding $2RB^r$ or $2CB^c$, respectively, across all trajectory
samples used by the lower-level algorithm. For $j\in\mathcal I_r(Z)$ and
$j\in\mathcal I_c(Z)$,} let $s_{r,j}$ and $s_{c,j}$ denote the corresponding
instances of \eqref{eq:encoding-score} \rev{for $B_j^r$ and $B_j^c$}, embedded in the
$x$- and $y$-coordinate blocks of $v$, respectively.
\begin{equation}
\label{eq:complete-score}
    S(Z;v)=\sum_{j\in\mathcal I_r(Z)}s_{r,j}
    +\sum_{j\in\mathcal I_c(Z)}s_{c,j}.
\end{equation}

\begin{assumption}[Upper-objective regularity]
\label{ass:outer-objective}
\label{ass:upper-objective}
The objective satisfies $G_{\min}\leq G(v,\theta)\leq G_{\max}$,
$\|\nabla_vG(v,\theta)\|\leq L_{G,0}$, and
$\|\nabla_\theta G(v,\theta)\|\leq M_G$, while $\nabla_vG(v,\theta)$ is
$L_G$-Lipschitz in $v$. Fresh upper-gradient oracles satisfy the unbiasedness
and variance conditions in \eqref{eq:main-upper-oracle-conditions}.
\begin{equation}
\label{eq:main-upper-oracle-conditions}
\begin{gathered}
    \mathbb E[\widehat\nabla_vG(v,\theta)\mid v,\theta]
    =\nabla_vG(v,\theta),\qquad
    \mathbb E[\widehat\nabla_\theta G(v,\theta)\mid v,\theta]
    =\nabla_\theta G(v,\theta),\\
    \mathbb E[\|\widehat\nabla_vG-\nabla_vG\|^2\mid v,\theta]
    \leq\sigma_v^2,\qquad
    \mathbb E[\|\widehat\nabla_\theta G-\nabla_\theta G\|^2\mid v,\theta]
    \leq\sigma_\theta^2.
\end{gathered}
\end{equation}
\end{assumption}
Smoothness of the upper-level objective together with unbiased,
bounded-variance stochastic gradient oracles is standard in stochastic
bilevel and nonconvex optimization analyses
\citep{hu2023contextual,gong2024accelerated}. The additional uniform
boundedness conditions control the direct-gradient and score-function terms in
the upper estimator \rev{constructed in Section~\ref{subsec:stochastic-upper-gradient}}; a concrete oracle construction and detailed verification
are given in Appendix~\ref{app:upper-objective-assumption-justification}.

Once the complete \rev{training} trace $Z$ is fixed, the \rev{approximate response}
$\theta(Z)$ is fixed and has no additional explicit dependence on $v$.
The \rev{training distribution} $P_v$ retains the dependence on the upper-level parameter, and
$\mathbb E_v$ denotes expectation under this \rev{distribution}. This gives the following
exact gradient expression for the upper objective induced by the approximate
lower-level response.

\begin{theorem}[Training-distribution gradient]
\label{thm:training-distribution-gradient}
Under Assumptions~\ref{ass:main-outer-signal}
and~\ref{ass:outer-objective},
$\Phi_{\epsilon_\ell}(v):=\mathbb E_{Z\sim P_v}[G(v,\theta(Z))]$ is
differentiable and, for any deterministic scalar baseline $c(v)$, satisfies the
identity in \eqref{eq:score-gradient-identity}.
\begin{equation}
\label{eq:score-gradient-identity}
    \nabla\Phi_{\epsilon_\ell}(v)
    =\mathbb E_v\left[\nabla_v G(v,\theta(Z))
    +(G(v,\theta(Z))-c(v))S(Z;v)\right].
\end{equation}
Here, $S(Z;v)$ is the sum of the score functions of all reward and constraint
encoding variables generated during the \rev{lower-level run}.
\end{theorem}
\noindent The identity follows from the training-distribution score argument
formalized in Lemma~\ref{lem:training-gradient} (in Appendix), using the same notation
$Z\sim P_v$, $\theta(Z)$, and $S(Z;v)$.

\subsection{Stochastic Upper-Level Gradient}
\label{subsec:stochastic-upper-gradient}

For the estimator, choose the baseline $c(v)=G(v,\theta_0)$, where $\theta_0$
is the fixed initialization common to all lower-level runs. To estimate
$G(v,\theta(Z))-G(v,\theta_0)$ using only upper-gradient oracle calls, define
$d=\theta(Z)-\theta_0$. Under Assumption~\ref{ass:outer-objective}, write
$V_\theta:=M_G^2+\sigma_\theta^2$ and set the implementable batch size
$B_G(Z):=\max\{1,\lceil\|d\|^2\rceil\}$.
Draw conditionally independent $U_m\sim\mathrm{Unif}[0,1]$ and fresh
upper-gradient oracle samples for $m=1,\ldots,B_G(Z)$.
The fundamental theorem of calculus gives the identity in
\eqref{eq:fundamental-value-difference}.
\begin{equation}
\label{eq:fundamental-value-difference}
    G(v,\theta(Z))-G(v,\theta_0)
    =\int_0^1\left\langle\nabla_\theta
    G(v,\theta_0+ud),d\right\rangle du.
\end{equation}
Accordingly, we define the value-difference estimator in
\eqref{eq:value-difference-estimator}.
\begin{equation}
\label{eq:value-difference-estimator}
    \widehat{\delta}_G
    :=\frac{1}{B_G(Z)}\sum_{m=1}^{B_G(Z)}
    \left\langle
        \widehat\nabla_\theta
        G(v,\theta_0+U_md),d
    \right\rangle.
\end{equation}
Conditionally on $Z$, this estimator is unbiased for the value difference and
satisfies $\operatorname{Var}(\widehat{\delta}_G\mid Z)\leq V_\theta$.
Moreover, Lemma~\ref{lem:response-displacement} gives
$\mathbb E[B_G(Z)]\leq2+\mathbb E\|d\|^2=O(1+K)$, so under the \rev{choice
$K=\lceil\delta_\ell^{-2}\rceil$ in Theorem~\ref{thm:end-to-end}, where
$\delta_\ell$ is the numerical lower-level accuracy}, each upper-level iteration uses
$O(\delta_\ell^{-2})$ upper $\theta$-gradient calls in expectation, in
addition to one $v$-gradient call. \rev{For the trajectory-based Bradley--Terry
upper objective of Appendix~\ref{app:upper-objective-assumption-justification}, every oracle call uses two trajectories of length at most $H$};
this cost is included in Theorem~\ref{thm:end-to-end}. The resulting
stochastic upper gradient, using an independent fresh unbiased sample
$\widehat\nabla_vG(v,\theta(Z))$, is given in
\eqref{eq:upper-gradient-estimator}.
\begin{equation}
\label{eq:upper-gradient-estimator}
    \widehat g(v)=\widehat\nabla_vG(v,\theta(Z))
    +\widehat{\delta}_G\,S(Z;v).
\end{equation}
The upper-level parameter is then updated as
$v_{t+1}=v_t-\alpha\widehat g_t$, as summarized in
Algorithm~\ref{alg:bilevel-approx-response}. A key feature of
\eqref{eq:upper-gradient-estimator} is that it requires neither an inverse
Hessian of a constrained saddle system nor differentiation through the full
sequence of lower-level updates, including the projected dual update.

\begin{algorithm}[H]
\small
\setlength{\parskip}{0pt}
\caption{Constrained Bilevel Reinforcement Learning with an Approximate Lower-Level Response}
\label{alg:bilevel-approx-response}
\label{alg:bilevel-nonpert}
\begin{algorithmic}[1]
\Require Initial upper-level parameter $v_0$, upper-level stepsize $\alpha$, number of
upper-level iterations $T_o$, common lower-level
primal initialization $\theta_0$, fixed dual initialization $\lambda_0=0$,
and lower-level parameters
$K,N,\eta,\zeta,\beta,\lambda_{\max}$.
\For{$t=0,\ldots,T_o-1$}
    \State Run the lower-level \rev{algorithm} in
    Section~\ref{subsec:lower-level-npg} at $v_t$.
    \State Obtain its \rev{approximate response}
    $\widehat\theta_t=\theta(Z_t)$ and the \rev{training-distribution score}
    $S_t=S(Z_t;v_t)$.
    \State Set $d_t=\widehat\theta_t-\theta_0$ and
    $B_{G,t}=\max\{1,\lceil\|d_t\|^2\rceil\}$.
    \For{$m=1,\ldots,B_{G,t}$}
        \State Sample $U_{t,m}\sim\operatorname{Unif}[0,1]$ independently.
        \State Obtain a fresh upper-gradient sample
        $\widehat G_{\theta,t,m}=\widehat\nabla_\theta
        G(v_t,\theta_0+U_{t,m}d_t)$.
        \State Set $Y_{t,m}=\langle\widehat G_{\theta,t,m},d_t\rangle$.
    \EndFor
    \State Form
    $\widehat{\delta}_{G,t}=B_{G,t}^{-1}\sum_{m=1}^{B_{G,t}}Y_{t,m}$.
    \State Obtain an independent upper partial-gradient sample
    $\widehat G_{v,t}=\widehat\nabla_vG(v_t,\widehat\theta_t)$.
    \State Form $\widehat g_t=\widehat G_{v,t}+\widehat{\delta}_{G,t}S_t$.
    \State Update $v_{t+1}=v_t-\alpha\widehat g_t$.
\EndFor
\State Independently sample
$R_o\sim\operatorname{Unif}\{0,\ldots,T_o-1\}$.
\State \rev{Perform one fresh lower-level run} at $v_{R_o}$ to obtain
a trace $Z_{\mathrm{out}}$ and
$\widehat\theta_{\mathrm{out}}=\theta(Z_{\mathrm{out}})$.
\State \Return $(v_{R_o},\widehat\theta_{\mathrm{out}})$.
\end{algorithmic}
\end{algorithm}

\section{Convergence Analysis}
\label{sec:analysis}

We first establish that the proposed lower-level algorithm satisfies the
approximate-response requirements in
\eqref{eq:approx-reward}--\eqref{eq:approx-feasibility}. We then analyze the
upper objective $\Phi_{\epsilon_\ell}$ induced by this response, establish
upper-level stationarity, and combine the two guarantees. The proof sketch in
Section~\ref{subsec:complexity-proof-sketch} explains how the length of the
\rev{a lower-level run} enters the final environment-interaction
complexity. Full proofs and detailed interaction accounting are deferred to
the appendix.

\subsection{Remaining Assumptions}
\label{subsec:remaining-assumptions}

The convergence analysis operates under the upper-level signal, Uniform Slater, and
upper-objective regularity assumptions introduced in
Section~\ref{sec:approach}, together with the additional assumptions below.

\begin{assumption}[Continuous-state CMDP regularity]
\label{ass:main-cmdp-regularity}
\label{ass:cmdp-regularity}
The state space $\mathcal S\subseteq\mathbb R^{d_s}$ is \rev{a closed subset of $\mathbb R^{d_s}$},
and the action space $\mathcal A$ is finite. The initial distribution $\rho$
and each transition kernel $P(\cdot\mid s,a)$ admit densities $\rho(s)$ and
$p(\cdot\mid s,a)$, respectively, with respect to Lebesgue measure. The
transition kernel is weakly continuous: for every bounded continuous function
$f:\mathcal S\to\mathbb R$ and every $a\in\mathcal A$, the mapping
$s\mapsto\int_{\mathcal S}f(s')p(s'\mid s,a)\,ds'$ is continuous.
For every fixed upper-level parameter $v=(x,y)$ and every $a\in\mathcal A$, the
mappings $r(\cdot,a;x)$ and $c(\cdot,a;y)$ are continuous in the state.
\end{assumption}
Weak continuity of the transition kernel together with state regularity of the
one-stage signals is standard in continuous-space MDP and
reinforcement-learning theory
\citep{kara2023qlearning,kara2026approximations}, including constrained
discounted MDPs \citep{feinberg2020constrained}. These conditions support the
existence and occupancy-measure arguments used here. The additional density
requirement permits a density-valued presentation of discounted occupancies
and is satisfied by common continuous-state systems with nondegenerate
additive noise; Appendix~\ref{app:cmdp-assumption-justification} gives a
sufficient construction and detailed justification.

\begin{assumption}[Policy regularity]
\label{ass:main-policy-regularity}
\label{ass:score}
The policies have full action support, $\pi_\theta(a\mid s)>0$, and the common
initialization satisfies $\pi_{\theta_0}(a\mid s)\geq p_0$ uniformly for some
$p_0>0$. There exist $G_\pi,L_\pi>0$ satisfying the bounds in
\eqref{eq:main-policy-regularity}.
\begin{equation}
\label{eq:main-policy-regularity}
    \|\nabla_\theta\log\pi_\theta(a\mid s)\|\leq G_\pi,
    \qquad
    \|\nabla_\theta\log\pi_\theta(a\mid s)
    -\nabla_\theta\log\pi_{\theta'}(a\mid s)\|
    \leq L_\pi\|\theta-\theta'\|.
\end{equation}
\end{assumption}
\rev{For the finite action space in Assumption~\ref{ass:cmdp-regularity}, the
full-support condition holds, for example, under the softmax parameterization
$\pi_\theta(a\mid s)=\exp(f_\theta(s,a))/\sum_{a'\in\mathcal A}\exp(f_\theta(s,a'))$,
for which $\pi_\theta(a\mid s)>0$ for every action and every finite parameter
vector $\theta$.}
Full support and the score bounds control likelihood-ratio and policy-gradient
terms, while the initialization condition controls the initial KL comparison.
Bounded-score assumptions are common in finite-sample policy-gradient analyses
\citep{bedi2022hidden,papini2022smoothing}; bounded and Lipschitz score
conditions are used, for example, by \citet{fatkhullin2023stochastic} and
\citet{zhang2023modelbased}.
Appendix~\ref{app:score-assumption-justification} gives further discussion.

\begin{assumption}[Fisher non-degeneracy]
\label{ass:main-fisher-lower}
\label{ass:fisher-lower}
There exists $\mu_F>0$ such that $F_\rho(\theta)\succeq\mu_FI$ for every
$\theta$; under policy regularity, $F_\rho(\theta)\preceq G_\pi^2I$.
\end{assumption}
Uniform Fisher non-degeneracy is standard in analyses of natural and
general-parameterization policy-gradient methods
\citep{liu2020improved,ding2022global,fatkhullin2023stochastic,
mondal2024constrained}. It makes the natural-gradient
direction well-defined and makes the compatible-regression objective strongly
convex, which is needed for geometric critic-bias decay; see
Appendix~\ref{app:fisher-assumption-justification}.

\begin{assumption}[Compatible-function approximation]
\label{ass:main-compatibility}
\label{ass:compatibility}
For an optimal stationary policy $\pi^\star$, define
the transferred compatible-function loss in
\eqref{eq:compatible-transfer-loss}.
\begin{equation}
\label{eq:compatible-transfer-loss}
\mathcal L_{\nu_\rho^{\pi^\star}}(\omega,\theta,\lambda)
:=\frac12\int_{\mathcal S}d_\rho^{\pi^\star}(s)
\sum_{a\in\mathcal A}\pi^\star(a\mid s)
\left[\nabla_\theta\log\pi_\theta(a\mid s)^\top\omega
-\frac{A_{L,\lambda}^{\pi_\theta}(s,a;v)}{1-\gamma}\right]^2ds.
\end{equation}
Uniformly over $\theta$ and $\lambda\in[0,\lambda_{\max}]$, the exact
natural-gradient direction satisfies
$\mathcal L_{\nu_\rho^{\pi^\star}}
(\omega_{\theta,\lambda}^\star,\theta,\lambda)
\leq\epsilon_{\mathrm{bias}}/2$.
\end{assumption}
Compatible-function approximation assumptions are standard in analyses of
natural policy-gradient methods with general policy parameterization
\citep{agarwal2020optimality,liu2020improved} and have been used directly in
parameterized CMDPs
\citep{bai2023zero,mondal2024constrained,mondal2026last}. The term
$\epsilon_{\rm bias}$ is the resulting approximation floor and measures the
representational mismatch between the Lagrangian advantage and the
score-function feature class; see
Appendix~\ref{app:compatibility-assumption-justification}.
\rev{Appendix~\ref{app:joint-policy-construction} gives a finite-action policy
class for which Assumptions~\ref{ass:cmdp-regularity}--\ref{ass:compatibility}
hold simultaneously.}

\subsection{Lower-Level Convergence}

The \rev{randomized binary encoding} preserves the conditional means of the sampled reward
and constraint signals, so the stochastic critic gradient remains
conditionally unbiased. For the vanilla-SGD critic recursion in
\eqref{eq:critic-sgd}, let
$\omega_k^\star:=\omega_{\theta_k,\lambda_k}^\star$ and let
$\mathbb E_k[\cdot]$ denote expectation conditional on the history at the
start of lower iteration $k$. If $R_\omega$ uniformly bounds
$\|\omega_k^\star\|$, Lemma~\ref{lem:sgd-critic} (in Appendix) gives the
conditional-bias bound
\begin{equation}
\label{eq:critic-bias-rate}
    \|\mathbb E_k[\widehat\omega_k]-\omega_k^\star\|
    \leq R_\omega\exp(-\beta\mu_FN).
\end{equation}
The same lemma provides a uniform second-moment bound despite the stochastic
critic noise floor. Thus $N=\Theta(\log K)$ makes the first-order critic bias
$O(K^{-1/2})$, matching the primal--dual optimization scale. Combining this
estimate with Lemmas~\ref{lem:local-global-sgd}
and~\ref{lem:dual-bounds} (in Appendix) gives the response guarantee.

\begin{theorem}[Approximate lower-level response guarantee]
\label{thm:lower-convergence}
Suppose Assumptions~\ref{ass:main-outer-signal},
\ref{ass:main-uniform-slater}, and
\ref{ass:main-cmdp-regularity}--\ref{ass:main-compatibility} hold. Choose
$\eta=\eta_0/\sqrt K$, $\zeta=\zeta_0/\sqrt K$ for fixed
$\eta_0,\zeta_0>0$, choose the critic stepsize $\beta$ to satisfy the
condition in Lemma~\ref{lem:sgd-critic} (in Appendix), take
$N\geq\max\{1,\lceil(\beta\mu_F)^{-1}
\log(R_\omega\sqrt K)\rceil\}$, and let
$R\sim\mathrm{Unif}\{0,\ldots,K-1\}$. Let $Z\sim P_v$ denote the complete
\rev{training} trace of this \rev{lower-level} run, including the selection of $R$, so that
$\theta_R=\theta(Z)$. The bounds in
\eqref{eq:lower-reward-final} and
\eqref{eq:lower-constraint-final} hold for problem-dependent constants
$C_r,C_c>0$.
\begin{align}
\label{eq:lower-reward-final}
    \mathbb E_{Z\sim P_v}[J_r^\star(v)-J_r(\theta(Z),x)]
    &\leq\sqrt{\epsilon_{\mathrm{bias}}}+\frac{C_r}{\sqrt K},\\
\label{eq:lower-constraint-final}
    \mathbb E_{Z\sim P_v}[-J_c(\theta(Z),y)]
    &\leq C_{\mathrm{sl}}\sqrt{\epsilon_{\mathrm{bias}}}
    +\frac{C_c}{\sqrt K}.
\end{align}
Here $C_{\mathrm{sl}}=2/\lambda_{\max}$. The approximation error
$\epsilon_{\mathrm{bias}}$ may depend on $\lambda_{\max}$.
\end{theorem}
\noindent A detailed proof of the theorem is provided in
Appendix~\ref{app:proof-lower-convergence}.

\rev{For a prescribed lower-level accuracy $\delta_\ell\in(0,1]$, choose
$K=\lceil\delta_\ell^{-2}\rceil$} and let
$\epsilon_\ell:=\max\{\sqrt{\epsilon_{\mathrm{bias}}},\delta_\ell\}$.
With $\rev{D_r}:=1+C_r$ and
$\rev{D_c}:=C_{\mathrm{sl}}+C_c$, Theorem~\ref{thm:lower-convergence}
gives reward and constraint errors bounded by
$\rev{D_r}\epsilon_\ell$ and $\rev{D_c}\epsilon_\ell$.
Thus \rev{this choice of $K$} and
$N=\Theta(\log(\delta_\ell^{-1}))$ produce the required response; increasing
$K$ does not remove the approximation floor.

\subsection{Upper-Level Stationarity}
\label{subsec:upper-stationarity}

For the lower-level choices in Theorem~\ref{thm:lower-convergence}, define the
lower-run scale $Q_\ell:=K(N+1)/(1-\gamma)$.
Lemma~\ref{lem:upper-smoothness} (in Appendix) proves that
$\Phi_{\epsilon_\ell}(v)=\mathbb E_{Z\sim P_v}[G(v,\theta(Z))]$ is
$L_\Phi$-smooth with $L_\Phi=O(1+Q_\ell)$.
Lemma~\ref{lem:upper-estimator} (in Appendix) proves that the estimator in
\eqref{eq:upper-gradient-estimator} satisfies
\begin{equation}
\label{eq:upper-estimator-properties}
\mathbb E[\widehat g(v)]=\nabla\Phi_{\epsilon_\ell}(v),\qquad
\mathbb E[\|\widehat g(v)\|^2]\leq V_\Phi,\qquad
V_\Phi=O(1+Q_\ell).
\end{equation}
\rev{The bound $V_\Phi$ is uniform: the same constant holds for every $v$, since
Lemma~\ref{lem:upper-estimator} derives it from the uniform signal bounds in
Assumption~\ref{ass:main-outer-signal} and the uniform objective and oracle
bounds in Assumption~\ref{ass:outer-objective}. This uniformity is needed
because the upper-level parameter changes across iterations.}
Together, these bounds determine the stepsize and iteration complexity of the
upper-level update. Section~\ref{subsec:complexity-proof-sketch} explains why both
quantities inherit the same dependence on $Q_\ell$.

\begin{lemma}[Upper-level stationarity]
\label{lem:upper-convergence-main}
By Lemma~\ref{lem:upper-smoothness} (in Appendix),
$\Phi_{\epsilon_\ell}$ is $L_\Phi$-smooth. Suppose
\eqref{eq:upper-estimator-properties} holds uniformly in $v$. For
$\epsilon_o\in(0,1]$, define
$\overline L_\Phi:=\max\{1,L_\Phi\}$,
$\overline V_\Phi:=\max\{1,V_\Phi\}$, and
$\Delta_G:=G_{\max}-G_{\min}$. Choose
$\alpha=\epsilon_o/(\overline L_\Phi\overline V_\Phi)$ and
$T_o=\max\{1,\lceil2\Delta_G\overline L_\Phi\overline V_\Phi/
\epsilon_o^2\rceil\}$, and let
$R_o\sim\mathrm{Unif}\{0,\ldots,T_o-1\}$. The resulting stationarity bound is
given in \eqref{eq:upper-final}.
\begin{equation}
\label{eq:upper-final}
    \mathbb E[\|\nabla\Phi_{\epsilon_\ell}(v_{R_o})\|^2]\leq\epsilon_o.
\end{equation}
\end{lemma}
\noindent Lemma~\ref{lem:upper-convergence} (in Appendix) provides the
detailed proof of this result.

\subsection{End-to-End Constrained Bilevel Guarantee}
\label{subsec:end-to-end-guarantee}

Combining upper-level stationarity with an independent fresh lower-level run
ensures that the reported upper-level iterate is paired with a policy satisfying the
reward and feasibility guarantees.

\begin{theorem}[End-to-end constrained bilevel guarantee]
\label{thm:end-to-end}
Suppose the conditions of Theorem~\ref{thm:lower-convergence}
and Lemma~\ref{lem:upper-convergence-main} hold uniformly in $v$. For a
prescribed numerical lower-level accuracy $\delta_\ell\in(0,1]$, choose
$K=\lceil\delta_\ell^{-2}\rceil$ and choose $N$ according to
Theorem~\ref{thm:lower-convergence}. Define the effective response accuracy by
$\epsilon_\ell:=\max\{\sqrt{\epsilon_{\mathrm{bias}}},\delta_\ell\}$ and let
$(v_{R_o},\widehat\theta_{\mathrm{out}})$ be the final output, where
$\widehat\theta_{\mathrm{out}}$ is obtained from an independent \rev{lower-level} run at
$v_{R_o}$. Then
\begin{align}
\label{eq:e2e-upper}
    \mathbb E[\|\nabla\Phi_{\epsilon_\ell}(v_{R_o})\|^2]
    &\leq\epsilon_o,\\
\label{eq:e2e-reward}
    \mathbb E[J_r^\star(v_{R_o})
    -J_r(\widehat\theta_{\mathrm{out}},x_{R_o})]
    &\leq\rev{D_r}\epsilon_\ell,\\
\label{eq:e2e-constraint}
    \mathbb E[-J_c(\widehat\theta_{\mathrm{out}},y_{R_o})]
    &\leq\rev{D_c}\epsilon_\ell.
\end{align}
If each upper-gradient oracle call uses at most $2H$ environment transitions,
then, for fixed primitive problem constants, the expected total interaction
count satisfies
$N_{\mathrm{env}}=\widetilde{\mathcal{O}}
(\delta_\ell^{-6}\epsilon_o^{-2})$. In particular, setting
$\delta_\ell=\epsilon_o=\epsilon$ gives
$N_{\mathrm{env}}=\widetilde{\mathcal{O}}(\epsilon^{-8})$.
\end{theorem}
\noindent A detailed proof is provided in Appendix
Theorem~\ref{thm:end-to-end-app}.

\subsection{Proof Sketch of the End-to-End Sample Complexity}
\label{subsec:complexity-proof-sketch}

\paragraph{Step 1: Constructing an Approximate Lower-Level Response.}
Theorem~\ref{thm:lower-convergence} bounds both reward suboptimality and
constraint violation by the compatible-approximation floor plus
$O(K^{-1/2})$. Consequently, numerical accuracy $\delta_\ell$ requires
$K=\Theta(\delta_\ell^{-2})$ and
$N=\Theta(\log(\delta_\ell^{-1}))$,
up to fixed constants and logarithmic factors. \rev{The lower-run scale
$Q_\ell=K(N+1)/(1-\gamma)$ of Section~\ref{subsec:upper-stationarity}
therefore satisfies $Q_\ell=\widetilde{\mathcal O}(\delta_\ell^{-2})$.}
Here $K(N+1)$ counts the encoded \rev{oracle calls in one lower-level run ($N$ critic calls and one constraint-oracle call per iteration)}, while
$(1-\gamma)^{-1}$ is the expected trajectory-length scale of each call.
Thus $Q_\ell$ measures the expected scale of one \rev{lower-level run}:
producing one $\delta_\ell$-accurate response already costs
$\widetilde{\mathcal O}(\delta_\ell^{-2})$ environment interactions. The
primal--dual details establishing this response are deferred to Appendix
Theorem~\ref{thm:lower-sgd}.

\paragraph{Step 2: Replacing Solution-Map Sensitivity with \rev{Training-Distribution}
Sensitivity.}
Once the complete \rev{training} trace $Z$ is fixed, every sampled signal and update
is fixed, so the \rev{approximate response} $\theta(Z)$ has no additional explicit
dependence on $v$. The lower-level dependence on $v$ is instead carried by the
\rev{training distribution} $P_v$. Differentiating this \rev{distribution} represents the indirect upper
sensitivity through the \rev{training-distribution score $S(Z;v)$ in
\eqref{eq:complete-score}, which equals $\nabla_v\log P_v(Z)$},
rather than through a derivative of a constrained optimizer or
differentiation through the full sequence of lower-level updates.
Lemma~\ref{lem:training-gradient} (in Appendix) formally
establishes the resulting gradient identity for
$\Phi_{\epsilon_\ell}(v)=\mathbb E_{Z\sim P_v}[G(v,\theta(Z))]$.

\paragraph{Step 3: Tracing Lower-Training Length Through the Upper-Level
Analysis.}
The \rev{training-distribution score} is a sum of score increments from the encoded
signals generated throughout the \rev{lower-level} run. These increments accumulate over
all $K(N+1)$ calls, so the score second moment and the corresponding
log-density Hessian moment grow with $Q_\ell$;
Lemmas~\ref{lem:score-moments} and~\ref{lem:encoding-counts} (in Appendix) make this
accumulation precise. This growth has two distinct consequences. First, differentiating the
score-function representation of $\Phi_{\epsilon_\ell}$ introduces both score
products and log-density Hessian terms. Their moment bounds imply
$L_\Phi=O(1+Q_\ell)$,
as established by Lemma~\ref{lem:upper-smoothness} (in Appendix). Second, the
stochastic upper-gradient estimator multiplies a value-difference estimate by
the \rev{training-distribution score}. Its second moment therefore inherits the score
second moment, giving $V_\Phi=O(1+Q_\ell)$,
as established by Lemma~\ref{lem:upper-estimator} (in Appendix). Avoiding
differentiation through the full sequence of lower-level updates does not
remove dependence on the length of \rev{a lower-level run}; under the
training-distribution construction, that dependence reappears through the
moments of the \rev{training-distribution score}.

\paragraph{Step 4: Composing the Interaction Complexity.}
Using these smoothness and second-moment bounds,
Lemma~\ref{lem:upper-convergence} (in Appendix) yields
$T_o=O(L_\Phi V_\Phi\epsilon_o^{-2})
=O((1+Q_\ell)^2\epsilon_o^{-2})
=\widetilde{\mathcal O}(\delta_\ell^{-4}\epsilon_o^{-2})$.
Every upper-level iteration executes one \rev{lower-level run} costing $O(Q_\ell)$.
Including the upper-oracle calls and the final fresh \rev{lower-level} run, the detailed
accounting in Theorem~\ref{thm:end-to-end-app} (in Appendix) gives
$N_{\mathrm{env}}=O(Q_\ell^3\epsilon_o^{-2})
=\widetilde{\mathcal O}(\delta_\ell^{-6}\epsilon_o^{-2})$.
The factor $\delta_\ell^{-6}$ therefore decomposes into three factors of
$\delta_\ell^{-2}$: one for a \rev{lower-level run}, one for score-induced
smoothness, and one for the upper-gradient second moment.
This decomposition is not an information-theoretic lower bound. Within the
present construction, improving the $\delta_\ell$ dependence requires removing
at least one factor: reducing the cost of a \rev{lower-level run}, preventing upper
smoothness from growing with the \rev{training-distribution score}, or constructing an
upper-gradient estimator whose second moment has weaker score-length
dependence.

\section{Conclusion}
\label{sec:conclusion}

We developed a constrained bilevel RL framework that replaces the potentially
nonunique exact constrained solution map with a performance-controlled
approximate lower-level response. A stochastic primal--dual natural
policy-gradient \rev{algorithm} constructs this response, while \rev{randomized binary encoding}
yields an unbiased score-function upper gradient by differentiating the
\rev{lower-level training distribution} rather than unrolling the
critic, policy, or projected dual updates. We established lower-level reward
and feasibility guarantees, stationarity of the induced upper objective, and
an expected environment-interaction complexity of
$\widetilde{\mathcal{O}}(\delta_\ell^{-6}\epsilon_o^{-2})$, where $\delta_\ell$ is the
controllable numerical lower-level accuracy. The effective response accuracy
is $\epsilon_\ell=\max\{\sqrt{\epsilon_{\mathrm{bias}}},\delta_\ell\}$ and
therefore retains the compatible-function approximation floor. When
$\delta_\ell=\epsilon_o=\epsilon$, the interaction complexity becomes
$\widetilde{\mathcal{O}}(\epsilon^{-8})$.

\bibliography{iclr2027_conference}
\bibliographystyle{iclr2027_conference}

\appendix
\setcounter{assumption}{0}
\renewcommand{\theHassumption}{appendix.\arabic{assumption}}

\section{Extended Related Work}
\label{sc:related}

\subsection{Bilevel Optimization}

Bilevel optimization has been widely studied in the context of hyperparameter
optimization, meta-learning, and hierarchical decision-making.
\citet{ghadimi2018approximation} were among the first to
establish sample-complexity guarantees for hypergradient-based bilevel
algorithms under convexity assumptions on the lower-level problem.
\citet{chen2021closing} subsequently developed a general stochastic
bilevel optimization framework with applications to stochastic min--max
optimization and reinforcement learning, while
\citet{chen2022single} established nonasymptotic convergence guarantees for a
single-timescale stochastic bilevel method. Momentum-based approaches further
improved the iteration complexity of bilevel optimization
\citep{yang2021provably,khanduri2021near}. These methods, however, rely on
Hessian-based computations and can therefore be computationally intensive.

To alleviate the cost of Hessian inversion, several works developed
Hessian-free or fully first-order bilevel algorithms
\citep{li2022fully,sow2022convergence,kwon2023fully,liu2022bome}.
Penalty-based approaches further simplify algorithm design by avoiding
explicit hypergradient calculations; convergence guarantees under smoothness
and regularity assumptions were established in
\citep{shen2023penalty,kwon2024on}. An alternative line of work introduced
value-function-based reformulations of bilevel problems
\citep{liu2021value}. Collectively, these advances provide an important
foundation for the development and analysis of bilevel RL algorithms.

More recently, envelope-based techniques have been introduced to address
nonsmoothness and weak convexity. \citet{gao2023moreau} proposed a
Moreau-envelope-based reformulation under a convex lower-level assumption,
whereas \citet{liu2024moreau} used the Moreau envelope to design a
single-loop, Hessian-free method for nonconvex bilevel optimization. Despite
this progress, most existing analyses continue to require smoothness,
convexity, or other strong regularity conditions on the lower-level objective.

\subsection{Bilevel Reinforcement Learning}

Bilevel RL has recently emerged as a useful modeling paradigm for hierarchical
structures arising in meta-RL, incentive design, and reinforcement learning
from human feedback (RLHF). Early work developed stochastic approximation
frameworks for bilevel RL, often combining two-timescale updates with policy
evaluation and policy improvement. In particular, \citet{hong2023two}
established finite-time complexity guarantees for a
two-timescale bilevel framework and demonstrated its application to
actor--critic methods.

Penalty-based formulations tailored to RL and RLHF were developed by
\citet{shen2025principled}, who established convergence guarantees under
smoothness and regularity assumptions. \citet{gaur2025sample}
studied the sample complexity of bilevel RL, characterized the additional
difficulties caused by nested policy optimization, and provided improved
sample-complexity guarantees for unconstrained bilevel RL and bilevel
optimization. Contextual bilevel RL was explored in
\citep{NEURIPS2024_e66309ea}, which developed a hypergradient-based method
under a softmax policy parameterization.

Several recent works have sought to relax lower-level convexity assumptions.
\citet{yang2025bilevel} developed hypergradient-based bilevel RL
algorithms without lower-level convexity, while
\citet{chakraborty2024parl} proposed a unified bilevel framework for policy
alignment in RLHF. Nevertheless, most existing bilevel RL analyses assume an
unconstrained lower-level problem and do not establish how their methods or
guarantees extend to constrained policy optimization.

\subsection{Constrained Bilevel Optimization}

Another active line of research considers bilevel optimization with explicit
lower-level constraints. Early approaches used implicit hypergradients or
projected stochastic methods under strong convexity of the lower-level problem
\citep{tsaknakis2022implicit,xiao2023alternating}.
\citet{xu2023efficient} combined hypergradient approximations with the
Clarke subdifferential for problems with convex lower-level objectives and
equality and inequality constraints. \citet{khanduri2023linearly} also
developed a hypergradient-based approach
for linearly constrained bilevel optimization and established convergence
under weak convexity of the upper-level objective.

The preceding methods rely on computationally demanding Hessian-based
hypergradient calculations. \citet{kornowski2024first}
developed a first-order algorithm for linearly constrained bilevel problems
with a strongly convex lower-level objective. \citet{jiang2024primal}
proposed a first-order penalty method motivated by
a primal--dual formulation for convex lower-level constraints, while retaining
strong convexity of the lower-level objective. Proximal-Lagrangian and
regularized-gap approaches inspired by envelope constructions were introduced
in \citep{yao2024constrained,yao2025overcoming} to remove the lower-level
strong-convexity requirement.

Although these methods substantially broaden the scope of constrained bilevel
optimization, their analyses generally retain convexity assumptions on the
lower-level objective or constraints. \citet{liu2023value} removed
the lower-level convexity assumption through a value-function-based sequential
minimization method, but did not provide a nonasymptotic convergence guarantee.
Thus, existing constrained bilevel optimization theory does not directly yield
a generally parameterized policy-gradient algorithm for constrained bilevel
RL, whose lower-level policy-optimization problem is nonconvex. In this setting,
the constrained solution set can be nonunique and unstable as the upper-level variable changes, so the corresponding solution map need not satisfy the
regularity required by existing techniques. We therefore introduce an
approximate lower-level formulation characterized directly by reward
suboptimality and constraint violation. Our work develops a sample-efficient
constrained bilevel RL framework based on approximate lower-level responses and
a training-distribution score-function gradient, avoiding differentiation
through the constrained lower-level optimizer or its primal--dual updates.
\section{Supporting Algorithms}

\subsection{Lower-Level Primal-Dual Natural Policy Gradient}
\label{app:lower-level-algorithm}

For a fixed upper-level parameter $v=(x,y)$, the lower-level routine approximately
solves the CMDP by applying primal--dual updates to the Lagrangian
$J_L(\theta,\lambda;v)=J_r(\theta,x)+\lambda J_c(\theta,y)$. At primal--dual
iteration $k$, it approximates the natural-gradient direction
$F_\rho(\theta_k)^{-1}\nabla_\theta J_L(\theta_k,\lambda_k;v)$ by running $N$
vanilla-SGD iterations on the quadratic critic objective in
\eqref{eq:critic-objective}. The final critic iterate $\widehat\omega_k$ is
used directly in the policy update
$\theta_{k+1}=\theta_k+\eta\widehat\omega_k$.

The multiplier is updated by projected descent using an independent encoded
constraint-return estimate:
$\lambda_{k+1}=\Pi_{[0,\lambda_{\max}]}[\lambda_k-
\zeta\widehat J_{c,\rho}(\theta_k;y)]$. Every critic or constraint-oracle call
also returns the score contribution of its encoded reward and constraint
signals. These contributions are accumulated into the complete training
score $S$. After $K$ iterations, the routine samples
$R\sim\operatorname{Unif}\{0,\ldots,K-1\}$ and returns the stochastic policy
parameter $\widehat\theta=\theta_R=\theta(Z)$ associated with the realized
trace $Z$, together with $S$. Thus,
Algorithm~\ref{alg:lower-pd-npg-sgd} constructs the approximate lower-level
response required by the bilevel formulation and provides the score used by
the upper-level gradient estimator.

\begin{algorithm}[H]
\small
\setlength{\parskip}{0pt}
\caption{Lower-Level Primal-Dual Natural Policy Gradient with Vanilla SGD}
\label{alg:lower-pd-npg-sgd}
\begin{algorithmic}[1]
\Require Upper-level parameter $v=(x,y)$, initial primal parameter $\theta_0$,
dual initialization $\lambda_0=0$, initial distribution $\rho_0$, numbers of
primal-dual and critic iterations $K$ and $N$, stepsizes $\eta$, $\zeta$, and
$\beta$, and a fixed dual interval $\Lambda=[0,\lambda_{\max}]$ with
$\lambda_{\max}\geq\frac{4R}{(1-\gamma)\underline c}$.
\State Initialize the cumulative encoding score $S\gets0$.
\For{$k=0,\ldots,K-1$}
    \State Initialize $\omega_{k,0}=0$.
    \For{$n=0,\ldots,N-1$}
        \State Call Algorithm~\ref{alg:encoded-critic} with
        $(\theta_k,\omega_{k,n},\lambda_k,v)$ to obtain
        $\widehat\nabla_\omega\mathcal L_\rho
        (\omega_{k,n};\theta_k,\lambda_k,v)$ and score $S_{k,n}$.
        \State Accumulate $S\gets S+S_{k,n}$.
        \State Update
        $\omega_{k,n+1}=\omega_{k,n}-\beta
        \widehat\nabla_\omega\mathcal L_\rho
        (\omega_{k,n};\theta_k,\lambda_k,v)$.
    \EndFor
    \State Set $\widehat\omega_k=\omega_{k,N}$.
    \State Independently call Algorithm~\ref{alg:encoded-critic} at
    $(\theta_k,0,\lambda_k,v)$ to obtain
    $\widehat J_{c,\rho}(\theta_k;y)$ and score $S_k^c$.
    \State Accumulate $S\gets S+S_k^c$.
    \State Update
    $\theta_{k+1}=\theta_k+\eta\widehat\omega_k$ and
    $\lambda_{k+1}=\Pi_{[0,\lambda_{\max}]}[\lambda_k-
    \zeta\widehat J_{c,\rho}(\theta_k;y)]$.
\EndFor
\State Independently sample
$R\sim\operatorname{Unif}\{0,\ldots,K-1\}$.
\State \Return $\widehat\theta=\theta_R=\theta(Z)$ and the complete score $S$.
\end{algorithmic}
\end{algorithm}

\subsection{Encoded Critic and Constraint Oracle}
\label{app:encoded-critic}

Algorithm~\ref{alg:encoded-critic} uses geometric-horizon sampling. We assume
that the simulator can initialize an occupancy-sampling trajectory from
$\rho_0=\rho$ and reset to a state $\widehat s$ sampled by that trajectory.
Conditional on $(\widehat s,\widehat a)$ and the preceding history, it can
generate independent continuations with fresh transition, action-sampling, and
encoding randomness. The $Q$-continuation fixes its first action to
$\widehat a$, whereas the $V$-continuation draws a fresh first action from
$\pi_\theta(\cdot\mid\widehat s)$.

Let $\psi_\theta(s,a):=\nabla_\theta\log\pi_\theta(a\mid s)$. For each queried
reward and constraint signal, independently sample the \rev{binary encoding} signs
using \eqref{eq:encoding-probability} and decode
\begin{equation}
\label{eq:encoded-r-c}
    \widetilde r_j=2RB_j^r,
    \qquad
    \widetilde c_j=2CB_j^c.
\end{equation}
Their score increments, embedded in the $x$ and $y$ coordinates of $v$, are
\begin{align}
\label{eq:reward-score}
    s_{r,j}
    &=\begin{bmatrix}
        \dfrac{B_j^r\nabla_xr(s_j,a_j;x)}{2R+B_j^rr(s_j,a_j;x)}\\[1mm]
        0_{d_y}
      \end{bmatrix}
      \in\mathbb R^{d_v},\\
\label{eq:cost-score}
    s_{c,j}
    &=\begin{bmatrix}
        0_{d_x}\\[1mm]
        \dfrac{B_j^c\nabla_yc(s_j,a_j;y)}{2C+B_j^cc(s_j,a_j;y)}
      \end{bmatrix}
      \in\mathbb R^{d_v}.
\end{align}

\begin{algorithm}[H]
\small
\setlength{\parskip}{0pt}
\caption{Encoded Unbiased Critic and Constraint Sampling}
\label{alg:encoded-critic}
\begin{algorithmic}[1]
\Require Policy parameter $\theta$, critic parameter $\omega$, dual variable
$\lambda$, upper-level parameter $v=(x,y)$, discount factor $\gamma$, initial
distribution $\rho_0=\rho$, and simulator reset access as described above.
\State Initialize the local score $S_{\mathrm{loc}}\gets0$.
\State Sample $T_0\sim\operatorname{Geo}(1-\gamma)$,
$s_0\sim\rho_0$, and $a_0\sim\pi_\theta(\cdot\mid s_0)$.
\For{$j=0,\ldots,T_0$}
    \State Query $c_j=c(s_j,a_j;y)$ and sample
    $B_j^c\in\{-1,+1\}$ with
    $\mathbb P(B_j^c=b)=\frac12+\frac{b c_j}{4C}$,
    $b\in\{-1,+1\}$.
    \State Decode $\widetilde c_j=2CB_j^c$ and accumulate
    $S_{\mathrm{loc}}\gets S_{\mathrm{loc}}+s_{c,j}$.
    \If{$j<T_0$}
        \State Execute $a_j$, observe
        $s_{j+1}\sim P(\cdot\mid s_j,a_j)$, and sample
        $a_{j+1}\sim\pi_\theta(\cdot\mid s_{j+1})$.
    \EndIf
\EndFor
\State Set
$\widehat J_{c,\rho}(\theta;y)=\sum_{j=0}^{T_0}\widetilde c_j$ and
$(\widehat s,\widehat a)=(s_{T_0},a_{T_0})$.
\State Independently sample $T_Q\sim\operatorname{Geo}(1-\gamma)$, reset to
$\widehat s$, and set $s_0=\widehat s$ and $a_0=\widehat a$.
\For{$j=0,\ldots,T_Q$}
    \State Query $r_j=r(s_j,a_j;x)$ and $c_j=c(s_j,a_j;y)$.
    \State Independently encode them as
    $\widetilde r_j=2RB_j^r$ and $\widetilde c_j=2CB_j^c$.
    \State Accumulate
    $S_{\mathrm{loc}}\gets S_{\mathrm{loc}}+s_{r,j}+s_{c,j}$.
    \If{$j<T_Q$}
        \State Execute $a_j$, observe
        $s_{j+1}\sim P(\cdot\mid s_j,a_j)$, and sample
        $a_{j+1}\sim\pi_\theta(\cdot\mid s_{j+1})$.
    \EndIf
\EndFor
\State Set
$\widehat Q_r^{\pi_\theta}(\widehat s,\widehat a)
=\sum_{j=0}^{T_Q}\widetilde r_j$ and
$\widehat Q_c^{\pi_\theta}(\widehat s,\widehat a)
=\sum_{j=0}^{T_Q}\widetilde c_j$.
\State Independently sample $T_V\sim\operatorname{Geo}(1-\gamma)$, reset to
$\widehat s$, set $s_0=\widehat s$, and sample
$a_0\sim\pi_\theta(\cdot\mid s_0)$.
\For{$j=0,\ldots,T_V$}
    \State Query and independently encode the reward and constraint as
    $\widetilde r_j$ and $\widetilde c_j$.
    \State Accumulate
    $S_{\mathrm{loc}}\gets S_{\mathrm{loc}}+s_{r,j}+s_{c,j}$.
    \If{$j<T_V$}
        \State Execute $a_j$, observe
        $s_{j+1}\sim P(\cdot\mid s_j,a_j)$, and sample
        $a_{j+1}\sim\pi_\theta(\cdot\mid s_{j+1})$.
    \EndIf
\EndFor
\State Set
$\widehat V_r^{\pi_\theta}(\widehat s)=\sum_{j=0}^{T_V}\widetilde r_j$ and
$\widehat V_c^{\pi_\theta}(\widehat s)=\sum_{j=0}^{T_V}\widetilde c_j$.
\State Form
$\widehat A_r^{\pi_\theta}=\widehat Q_r^{\pi_\theta}-
\widehat V_r^{\pi_\theta}$,
$\widehat A_c^{\pi_\theta}=\widehat Q_c^{\pi_\theta}-
\widehat V_c^{\pi_\theta}$, and
$\widehat A_{L,\lambda}^{\pi_\theta}=\widehat A_r^{\pi_\theta}
+\lambda\widehat A_c^{\pi_\theta}$.
\State Compute
$\widehat F_\rho(\theta)=\psi_\theta(\widehat s,\widehat a)
\psi_\theta(\widehat s,\widehat a)^\top$ and
$\widehat H_\rho(\theta,\lambda)=\widehat A_{L,\lambda}^{\pi_\theta}
\psi_\theta(\widehat s,\widehat a)$.
\State Form
$\widehat\nabla_\omega\mathcal L_\rho(\omega;\theta,\lambda,v)
=\widehat F_\rho(\theta)\omega-(1-\gamma)^{-1}
\widehat H_\rho(\theta,\lambda)$.
\State \Return $\widehat J_{c,\rho}(\theta;y)$,
$\widehat\nabla_\omega\mathcal L_\rho(\omega;\theta,\lambda,v)$, and
$S_{\mathrm{loc}}$.
\end{algorithmic}
\end{algorithm}

The decoded signals are conditionally unbiased, so the critic-gradient
and constraint-return estimators retain the means required by the lower-level
analysis. Meanwhile, conditioned on the complete encoded trace, the returned
policy has no additional explicit dependence on $v$.

\section{Proof of Theorem \ref{thm:lower-convergence}}
\label{app:proof-lower-convergence}
\subsection{Justification for Assumptions for Theorem \ref{thm:lower-convergence}}
\label{app:lower-assumptions}

\subsubsection{Discussion on Assumption~\ref{ass:cmdp-regularity}}
\label{app:cmdp-assumption-justification}
Weak continuity of the controlled transition kernel and continuity of the
one-stage reward and constraint functions are standard regularity conditions
in continuous-space MDP and reinforcement-learning theory.
\citet{kara2023qlearning} study Q-learning on general state and action spaces
under weak continuity of the transition kernel together with continuity and
boundedness of the one-stage cost, while \citet{kara2026approximations} use
weakly continuous dynamics in approximation and learning results for
continuous-space average-cost MDPs. In the constrained discounted setting,
\citet{feinberg2020constrained} study Borel-state CMDPs under a
weak-continuity regime. These references provide precedent for the
weak-continuity and state-regularity component of our assumption, rather than
for every condition imposed here. Our Euclidean state space and finite action
space specialize the more general Borel-space formulations in these works.

The additional density condition is not fundamental to existence of an
optimal policy, but is imposed so that the initial \rev{distribution}, transition kernels,
and discounted occupancies can be represented by ordinary densities
\[
\rho(s),\qquad p(s'\mid s,a),\qquad d_\rho^\pi(s),
\]
and the expressions below can be written as Lebesgue integrals. For each
action, the state continuity imposed here and differentiability with respect to the upper-level variable imposed in Assumption~\ref{ass:outer-signal} make the signals
separately continuous in their Euclidean arguments and hence jointly Borel
measurable, as needed below.

As a concrete example, let
\[
\mathcal S=\mathbb R^{d_s},\qquad
\mathcal A=\{1,\ldots,m\},
\]
and consider the additive-noise dynamics
\[
s_{t+1}=f_a(s_t)+\xi_t^a,
\qquad
\xi_t^a\sim\mathcal N(0,\Sigma_a),
\qquad
\Sigma_a\succ0,
\]
where each $f_a:\mathbb R^{d_s}\to\mathbb R^{d_s}$ is continuous. The
transition kernel then has the Lebesgue density
\[
p(s'\mid s,a)=\mathcal N(s';f_a(s),\Sigma_a).
\]
If the initial state is also absolutely continuous, for example
\[
s_0\sim\mathcal N(m_0,\Sigma_0),
\qquad
\Sigma_0\succ0,
\]
then $\rho$ admits a Lebesgue density as well. To verify weak continuity, let
$h$ be bounded and continuous and let $s_n\to s$. Under a common realization
of $\xi^a$, continuity of $f_a$ gives
\[
f_a(s_n)+\xi^a\longrightarrow f_a(s)+\xi^a
\qquad\text{almost surely}.
\]
Bounded convergence therefore yields
\[
\mathbb E[h(f_a(s_n)+\xi^a)]
\longrightarrow
\mathbb E[h(f_a(s)+\xi^a)],
\]
or equivalently,
\[
\int_{\mathcal S}h(s')p(s'\mid s_n,a)\,ds'
\longrightarrow
\int_{\mathcal S}h(s')p(s'\mid s,a)\,ds'.
\]
Thus, if $r(\cdot,a;x)$ and $c(\cdot,a;y)$ are continuous in the state, all
state-regularity requirements in Assumption~\ref{ass:cmdp-regularity} hold.
The uniform boundedness and derivative requirements with respect to the upper-level parameter are
handled separately by Assumption~\ref{ass:outer-signal}.

\subsubsection{Discussion on Assumption~\ref{ass:score}}
\label{app:score-assumption-justification}

Uniform boundedness of the policy score is a common regularity condition in
finite-sample policy-gradient analysis; see, for example,
\citet{bedi2022hidden}, \citet{papini2022smoothing}, and
\citet{fatkhullin2023stochastic}. Fatkhullin et al. also impose Lipschitz
continuity of the score, as does Assumption~5.1 of
\citet{zhang2023modelbased}. These
references provide precedent specifically for the bounded- and
Lipschitz-score conditions, rather than for every component of
Assumption~\ref{ass:score}.

The uniform initialization bound is an explicit requirement on the chosen
policy class and initialization; it is not automatic on a noncompact state
space. A state-independent uniform initial policy satisfies it whenever
representable by the class.

\subsubsection{Discussion on Assumption~\ref{ass:fisher-lower}}
\label{app:fisher-assumption-justification}

Uniform positive definiteness of the Fisher information matrix is used in
global and finite-sample analyses of policy-gradient and natural
policy-gradient methods
\citep{liu2020improved,ding2022global,fatkhullin2023stochastic}. It is also
assumed in the general-parameterization CMDP
analysis of \citet{mondal2024constrained}. Besides ensuring that the natural
gradient $F_\rho(\theta)^{-1}\nabla J(\theta)$ is well-defined, the condition
makes the quadratic compatible-function objective $\mu_F$-strongly convex,
which yields the geometric critic convergence used below. Prior work also
discusses concrete Fisher-nondegenerate classes: Gaussian policies with fixed
positive-definite covariance and a suitably full-rank mean parameterization,
as well as broader subclasses of exponential-family policies; see
\citet{fatkhullin2023stochastic} and Section~8 of \citet{ding2022global}.
These examples require structural conditions on the parameterization and do
not imply Fisher non-degeneracy for arbitrary Gaussian or exponential-family
policies.

\subsubsection{Discussion on Assumption~\ref{ass:compatibility}}
\label{app:compatibility-assumption-justification}

The transferred form above evaluates the squared compatible-function residual
under the occupancy distribution of an optimal CMDP policy. Closely matching
assumptions appear in general-parameterization CMDP analyses
\citep{bai2023zero,mondal2024constrained,mondal2026last}, building on
compatible-function approximation conditions used in unconstrained natural
policy-gradient analyses \citep{agarwal2020optimality,liu2020improved}.
Accordingly, $\epsilon_{\rm bias}$ records the expressivity error of the
score-function feature class. In settings identified by prior work, the error
vanishes for complete parameterizations such as direct or softmax policies,
with analogous exact-representation results for linear MDPs; sufficiently rich
classes can instead make the error small
\citep{bai2023zero,mondal2024constrained}. For a restricted policy class,
$\epsilon_{\rm bias}>0$ explicitly retains the irreducible representational
mismatch rather than hiding it in the optimization error.

Since the Lagrangian advantage
$A_{L,\lambda}^{\pi_\theta}=A_r^{\pi_\theta}
+\lambda A_c^{\pi_\theta}$ depends explicitly on $\lambda$, enlarging the
dual interval may increase the compatible-function approximation constant
required to hold uniformly over $[0,\lambda_{\max}]$. Consequently, the
factor $2/\lambda_{\max}$ appearing later in the feasibility bound should be
interpreted as a normalization factor rather than as a mechanism for
arbitrarily suppressing the approximation floor. Because $\lambda_{\max}$ is
fixed through the uniform Slater condition, the resulting
$\epsilon_{\rm bias}$ is treated as a problem-dependent constant and may
inherit dependence on the Slater margin.

The loss in~\eqref{eq:compatible-transfer-loss} is a nonnegative squared
residual evaluated under the optimal policy's occupancy density. The loss
in~\eqref{eq:compatible-transfer-loss} is the transferred counterpart of the
critic objective in~\eqref{eq:critic-objective}, obtained by evaluating the
same squared residual under the optimal-policy occupancy rather than the
current-policy occupancy.
The factor $(1-\gamma)^{-1}$ is consistent with the normalized occupancy
convention and the natural-gradient definition in
Section~\ref{subsec:lower-level-npg}.

\subsubsection{Discussion on Assumption~\ref{ass:slater-lower}}
\label{app:slater-assumption-justification}

\paragraph{Ordinary Slater condition.}
For a fixed CMDP, existence of an interior feasible policy with a positive
margin is standard in primal--dual constrained-RL analyses
\citep{bai2023zero,mondal2024constrained,mondal2026last}. It supports strong
duality and yields a finite bound on the optimal Lagrange multiplier. These
references impose Slater conditions for fixed CMDPs; they do not require a
margin uniform over a bilevel upper-level parameter.

\paragraph{Uniformity over the upper-level parameter.}
Our assumption is stronger because the constraint function changes with $y$.
The required uniformity is nevertheless compatible with unconstrained upper-level updates when strict feasibility is enforced structurally through the
constraint parameterization rather than through projection of the upper-level variable. A simple sufficient construction is
\[
c(s,a;y)=c_0(s,a)+\Delta c(s,a;y),
\]
where the perturbation induced by the upper-level parameter is uniformly bounded:
\[
|\Delta c(s,a;y)|\le \delta,
\qquad
\forall (s,a,y).
\]
Suppose there exists a stationary policy $\pi_{\mathrm{safe}}$ such that
\[
J_{c_0}(\pi_{\mathrm{safe}})
\ge
\underline c+\frac{\delta}{1-\gamma}
\]
for some $\underline c>0$. Then, for every $y\in\mathbb R^{d_y}$,
\[
J_c(\pi_{\mathrm{safe}},y)
\ge
J_{c_0}(\pi_{\mathrm{safe}})
-
\frac{\delta}{1-\gamma}
\ge
\underline c.
\]
Hence the same policy is strictly feasible with a uniform margin over the
entire upper-level parameter space.

Such a construction is natural when $y$ is unconstrained but enters the
constraint function through a bounded map, such as a sigmoid or $\tanh$
parameterization. Assumption~\ref{ass:slater-lower} can therefore
be interpreted as a
structural requirement ensuring global strict feasibility, rather than as
requiring an explicit projection of the upper-level iterate onto a compact set.

\subsubsection{\texorpdfstring{\rev{A Finite-Action Policy Class Satisfying Assumptions~\ref{ass:score}--\ref{ass:compatibility}}}{A Finite-Action Policy Class Satisfying Assumptions 5-7}}
\label{app:joint-policy-construction}
\begingroup\color{black}
This subsection shows, by a sufficient construction, that
Assumptions~\ref{ass:cmdp-regularity}--\ref{ass:compatibility} can hold
simultaneously with a finite action space.

\paragraph{Environment.}
Take the additive-noise continuous-state model of
Appendix~\ref{app:cmdp-assumption-justification} with
$\mathcal S=\mathbb R^{d_s}$, a Gaussian initial distribution, and reward and
constraint functions that are continuous in the state and satisfy
Assumption~\ref{ass:outer-signal}. Assumption~\ref{ass:cmdp-regularity} then
holds. For a parameter dimension $d_\theta=d$, let the action space consist of
$4d$ actions,
\[
\mathcal A=\{a_{i,j}:i=1,\ldots,d,\ j=1,\ldots,4\}.
\]

\paragraph{Policy class.}
Let $\bar p:=1/(4d)$ and fix $0<\delta<\bar p$. For every state $s$ and every
$\theta\in\mathbb R^{d}$, define
\[
\begin{aligned}
\pi_\theta(a_{i,1}\mid s)&=\bar p+\delta\sin\theta_i, &
\pi_\theta(a_{i,2}\mid s)&=\bar p-\delta\sin\theta_i,\\
\pi_\theta(a_{i,3}\mid s)&=\bar p+\delta\cos\theta_i, &
\pi_\theta(a_{i,4}\mid s)&=\bar p-\delta\cos\theta_i.
\end{aligned}
\]
Each block of four probabilities sums to $4\bar p=1/d$, so the $4d$
probabilities sum to one. The policy is state-independent. For an action in block $i$, $\log\pi_\theta(a\mid s)$
depends only on $\theta_i$, so $\nabla_\theta\log\pi_\theta(a\mid s)$ has only
its $i$-th coordinate nonzero.

\paragraph{Assumption~\ref{ass:score}.}
Every probability satisfies $\pi_\theta(a\mid s)\geq\bar p-\delta>0$, so full
support holds for every $\theta$, and the initialization condition holds for
any $\theta_0$ with $p_0=\bar p-\delta$. For $a=a_{i,1}$,
\[
\frac{\partial}{\partial\theta_i}\log\pi_\theta(a_{i,1}\mid s)
=\frac{\delta\cos\theta_i}{\bar p+\delta\sin\theta_i},
\qquad
\frac{\partial^2}{\partial\theta_i^2}\log\pi_\theta(a_{i,1}\mid s)
=-\frac{\delta\sin\theta_i}{\bar p+\delta\sin\theta_i}
-\frac{\delta^2\cos^2\theta_i}{(\bar p+\delta\sin\theta_i)^2}.
\]
The other three actions in the block give analogous expressions, with
numerators equal up to sign and denominators of the form
$\bar p\pm\delta\sin\theta_i$ or $\bar p\pm\delta\cos\theta_i$. Since every
denominator is at least $\bar p-\delta$,
\[
\|\nabla_\theta\log\pi_\theta(a\mid s)\|\leq\frac{\delta}{\bar p-\delta}=:G_\pi,
\qquad
\|\nabla_\theta^2\log\pi_\theta(a\mid s)\|
\leq\frac{\delta}{\bar p-\delta}+\frac{\delta^2}{(\bar p-\delta)^2}=:L_\pi,
\]
where the Hessian is diagonal with a single nonzero entry. The mean value
theorem then gives the Lipschitz bound in~\eqref{eq:main-policy-regularity}.

\paragraph{Assumption~\ref{ass:fisher-lower}.}
Because the score of an action in block $i$ is supported on coordinate $i$,
$F_\rho(\theta)$ is diagonal. Since the policy does not depend on the state,
the state occupancy integrates out, and
\[
\begin{aligned}
F_{ii}(\theta)
&=\sum_{j=1}^4\frac{\big(\partial_{\theta_i}\pi_\theta(a_{i,j}\mid s)\big)^2}
{\pi_\theta(a_{i,j}\mid s)}\\
&=\delta^2\cos^2\theta_i\Big(\frac{1}{\bar p+\delta\sin\theta_i}
+\frac{1}{\bar p-\delta\sin\theta_i}\Big)
+\delta^2\sin^2\theta_i\Big(\frac{1}{\bar p+\delta\cos\theta_i}
+\frac{1}{\bar p-\delta\cos\theta_i}\Big).
\end{aligned}
\]
Every denominator is at most $\bar p+\delta$, so
$F_{ii}(\theta)\geq 2\delta^2/(\bar p+\delta)$. Hence
$F_\rho(\theta)\succeq\mu_FI$ uniformly over $\theta$ with
$\mu_F:=2\delta^2/(\bar p+\delta)>0$. Unlike the redundant tabular softmax
parameterization, whose Fisher matrix is singular, this class has a
nondegenerate Fisher matrix.

\paragraph{Assumption~\ref{ass:compatibility}.}
Let $B_L:=R+\lambda_{\max}C$. For every $\lambda\in[0,\lambda_{\max}]$, the
one-stage Lagrangian signal $r+\lambda c$ is bounded by $B_L$, so
$|A_{L,\lambda}^{\pi_\theta}(s,a;v)|\leq 2B_L/(1-\gamma)$. By the
policy-gradient identity
$\nabla_\theta J_L(\theta,\lambda;v)=(1-\gamma)^{-1}
\mathbb E_{\nu_\rho^{\pi_\theta}}[\psi_\theta A_{L,\lambda}^{\pi_\theta}]$,
\[
\|\nabla_\theta J_L(\theta,\lambda;v)\|\leq\frac{2B_LG_\pi}{(1-\gamma)^2},
\qquad\text{hence}\qquad
\|\omega_{\theta,\lambda}^\star\|
=\|F_\rho(\theta)^{-1}\nabla_\theta J_L(\theta,\lambda;v)\|
\leq\frac{2B_LG_\pi}{\mu_F(1-\gamma)^2}.
\]
Therefore
$|\psi_\theta(s,a)^\top\omega_{\theta,\lambda}^\star|
\leq 2B_LG_\pi^2/(\mu_F(1-\gamma)^2)$ and
$|A_{L,\lambda}^{\pi_\theta}(s,a;v)/(1-\gamma)|\leq 2B_L/(1-\gamma)^2$, so the
residual in~\eqref{eq:compatible-transfer-loss} is bounded at every $(s,a)$ by
$2B_L(1+G_\pi^2/\mu_F)/(1-\gamma)^2$. Since $d_\rho^{\pi^\star}$ and
$\pi^\star$ are probability distributions, Assumption~\ref{ass:compatibility}
holds with
\[
\epsilon_{\rm bias}
=\left[\frac{2B_L}{(1-\gamma)^2}\Big(1+\frac{G_\pi^2}{\mu_F}\Big)\right]^2.
\]

Together with the additive-noise continuous-state model of
Appendix~\ref{app:cmdp-assumption-justification}, this provides a finite-action
example satisfying
Assumptions~\ref{ass:cmdp-regularity}--\ref{ass:compatibility} simultaneously.
\endgroup

\subsection{Supporting Lemma for Theorem \ref{thm:lower-convergence}}
For a stationary policy $\pi$, define its normalized discounted state and
state-action occupancy densities by
\begin{equation}
\label{eq:continuous-occupancy}
d_\rho^\pi(s):=(1-\gamma)\sum_{t=0}^{\infty}\gamma^t
p_t^\pi(s),
\qquad
\nu_\rho^\pi(s,a):=d_\rho^\pi(s)\pi(a\mid s),
\end{equation}
where $p_t^\pi$ is the density of $s_t$ under $s_0\sim\rho$ and policy
$\pi$. These densities satisfy
\begin{equation}
\label{eq:occupancy-density-flow}
d_\rho^\pi(s')=(1-\gamma)\rho(s')
+\gamma\sum_{a\in\mathcal A}\int_{\mathcal S}
p(s'\mid s,a)\nu_\rho^\pi(s,a)\,ds.
\end{equation}
For the fixed upper-level parameter $v=(x,y)$, the returns are linear in the
occupancy density:
\begin{equation}
\label{eq:continuous-return}
\begin{aligned}
J_r(\pi,x)&=\frac{1}{1-\gamma}\int_{\mathcal S}
\sum_{a\in\mathcal A}\nu_\rho^\pi(s,a)r(s,a;x)\,ds,\\
J_c(\pi,y)&=\frac{1}{1-\gamma}\int_{\mathcal S}
\sum_{a\in\mathcal A}\nu_\rho^\pi(s,a)c(s,a;y)\,ds.
\end{aligned}
\end{equation}
When $v$ is fixed, $J_g(\pi)$ denotes the corresponding reward or constraint
return for $g\in\{r,c\}$.

\begin{lemma}[Continuous-state occupancy measures and CMDP duality]
\label{lem:continuous-cmdp-duality}
Under Assumptions~\ref{ass:cmdp-regularity} and
\ref{ass:slater-lower}, and the boundedness part of
Assumption~\ref{ass:outer-signal}, for each fixed $v$ the unparameterized CMDP
has an optimal stationary randomized policy $\pi^\star$ and a finite optimal
multiplier $\lambda^\star(v)\geq0$. Writing
$J_L(\pi,\lambda;v):=J_r(\pi,x)+\lambda J_c(\pi,y)$ for all
$\lambda\geq0$, one has
\begin{equation}
\label{eq:continuous-strong-duality}
J_r^\star(v)
=\min_{\lambda\geq0}\max_\pi J_L(\pi,\lambda;v)
=\max_\pi J_L(\pi,\lambda^\star(v);v),
\end{equation}
and the saddle inequalities
\begin{equation}
\label{eq:continuous-saddle}
J_L(\pi,\lambda^\star(v);v)
\leq J_L(\pi^\star,\lambda^\star(v);v)
=J_r^\star(v)
\leq J_L(\pi^\star,\lambda;v)
\end{equation}
hold for every stationary randomized $\pi$ and every $\lambda\geq0$.
Moreover, any average, including an expectation, of discounted occupancy
measures is the occupancy measure of a stationary randomized policy.
Neither that policy nor $\pi^\star$ is required to belong to the
parameterized policy class.
\end{lemma}

\begin{proof}
Only in this proof, we identify densities with their corresponding Borel
probability measures. In particular,
\[
\rho(E):=\int_E\rho(s)\,ds,\qquad
P(E\mid s,a):=\int_Ep(s'\mid s,a)\,ds',\qquad
\nu_\rho^\pi(E\times\{a\}):=\int_E\nu_\rho^\pi(s,a)\,ds
\]
for each Borel set $E\subseteq\mathcal S$.
Let $\mathcal D_\rho$ be the set of probability measures $\nu$ on
$\mathcal S\times\mathcal A$ whose state marginal $\nu_{\mathcal S}$
satisfies the discounted flow equation
\begin{equation}
\label{eq:continuous-flow}
\nu_{\mathcal S}(E)
=(1-\gamma)\rho(E)
+\gamma\sum_{a\in\mathcal A}
\int_{\mathcal S}P(E\mid s,a)\,\nu(\mathrm ds,\{a\})
\end{equation}
for every Borel $E\subseteq\mathcal S$.
Every probability measure with occupancy density given by
\eqref{eq:continuous-occupancy} satisfies this equation, by
\eqref{eq:occupancy-density-flow}. Conversely, disintegration on the Borel state space gives
$\nu(\mathrm ds,\{a\})=\nu_{\mathcal S}(\mathrm ds)\pi(a\mid s)$
for a Borel stationary randomized policy, defined arbitrarily on a
$\nu_{\mathcal S}$-null set. For this fixed policy, the right-hand side
of the state-marginal flow equation is a contraction with factor $\gamma$
in total variation. Its unique solution is $d_\rho^\pi$, so
$\nu=\nu_\rho^\pi$. Thus $\mathcal D_\rho$ is exactly the set of
stationary-policy occupancy measures. Its defining equations are linear,
so it is convex and is also closed under expectations of random occupancy
measures, by integration of~\eqref{eq:continuous-flow}.

We next prove weak compactness without assuming compactness of $\mathcal S$.
For each fixed time $t$, the family of state \rev{distributions} reachable from $\rho$
is uniformly tight over all policies. This follows by induction.
The initial \rev{distribution} $\rho$ is tight because it is a Borel probability measure on
\rev{a closed subset of Euclidean space, which is a Polish space}.
For the induction step, choose a compact set carrying arbitrarily high
probability uniformly for the time-$t$ \rev{distributions}. On this compact set, weak
continuity of $P$ and finiteness of $\mathcal A$ make the family of
next-state kernels weakly compact, hence uniformly tight. A compact set
carrying high probability for all these kernels then carries high
probability for all time-$(t+1)$ \rev{distributions}, up to the probability of leaving
the first compact set. This proves the induction claim.
The occupancy mass from times $t\geq m$ is $\gamma^m$, uniformly in the
policy. Combining this tail bound with tightness at the finitely many
times $0,\ldots,m-1$ proves uniform tightness of $\mathcal D_\rho$.
Finally, the flow equation tested against any bounded continuous $f$ is
preserved under weak limits, since $P f(s,a)$ is bounded continuous on
$\mathcal S\times\mathcal A$. Such test functions determine Borel
probability measures, so $\mathcal D_\rho$ is weakly closed. Prokhorov's
theorem now gives weak compactness.
Although densities need not be preserved under arbitrary weak limits,
they are preserved within this flow-constrained set: for any
Lebesgue-null set $E$, the density assumptions give
$\rho(E)=P(E\mid s,a)=0$ for all
$(s,a)$, and~\eqref{eq:continuous-flow} implies
$\nu_{\mathcal S}(E)=0$. Each action component is dominated by the state
marginal, so every $\nu\in\mathcal D_\rho$ admits a joint occupancy
density. Consequently, this measure-based compactness argument applies
to precisely the density-based policy model used elsewhere in the paper.

For fixed $v$, define the set of attainable constraint--reward pairs
\[
\mathcal K_v
:=\left\{
\left(\frac{\int c(s,a;y)\,\mathrm d\nu}{1-\gamma},
      \frac{\int r(s,a;x)\,\mathrm d\nu}{1-\gamma}\right)
:\nu\in\mathcal D_\rho
\right\}.
\]
State continuity from Assumption~\ref{ass:cmdp-regularity} and boundedness
from Assumption~\ref{ass:outer-signal} make $\mathcal K_v$ compact, and
linearity makes it convex. Its subset with nonnegative first coordinate
is nonempty by Slater feasibility and compact. Maximizing the second
coordinate therefore attains $J_r^\star(v)$; disintegration gives an
optimal stationary randomized policy $\pi^\star$.

For completeness, separate the convex set $\mathcal K_v$ from the disjoint
open convex set
$\{(q,z):q>0,\ z>J_r^\star(v)\}$.
The separating-hyperplane theorem gives coefficients
$b_c,b_r\geq0$, not both zero, such that
\[
b_c q+b_r z\leq b_r J_r^\star(v)
\qquad\text{for every }(q,z)\in\mathcal K_v.
\]
Nonnegativity follows because the open set is unbounded in each positive
coordinate direction. Slater feasibility provides a point of
$\mathcal K_v$ with $q>0$, which rules out $b_r=0$.
Thus $\lambda^\star(v):=b_c/b_r$ is finite and nonnegative, and
\[
J_r(\pi,x)+\lambda^\star(v)J_c(\pi,y)\leq J_r^\star(v)
\qquad\text{for every }\pi.
\]
At $\pi^\star$, feasibility and optimality give the reverse inequality,
hence equality and complementary slackness
$\lambda^\star(v)J_c(\pi^\star,y)=0$.
For any $\lambda\geq0$, feasibility gives
$\max_\pi J_L(\pi,\lambda;v)\geq J_L(\pi^\star,\lambda;v)
\geq J_r^\star(v)$.
These observations prove~\eqref{eq:continuous-strong-duality} and
\eqref{eq:continuous-saddle}, including dual attainment.
\end{proof}

The saddle point above belongs to the unparameterized policy--multiplier
problem. No existence or convexity assertion is made for a saddle point in
$(\theta,\lambda)$ coordinates.

The uniform margin allows one dual projection interval to be used throughout
the bilevel algorithm. Let $\lambda^\star(v)\geq0$ be an optimal multiplier
provided by Lemma~\ref{lem:continuous-cmdp-duality}. By
\eqref{eq:continuous-strong-duality},
\[
J_r^\star(v)
=\max_\pi\{J_r(\pi,x)+\lambda^\star(v)J_c(\pi,y)\}
\geq J_r(\pi_{\rm sl}(v),x)+\lambda^\star(v)c_0(v).
\]
Since $|r(s,a;x)|\leq R$ uniformly by
Assumption~\ref{ass:outer-signal}, this gives
\begin{equation}
\label{eq:uniform-dual-bound}
0\leq\lambda^\star(v)
\leq\frac{J_r^\star(v)-J_r(\pi_{\rm sl}(v),x)}{c_0(v)}
\leq\frac{2R}{(1-\gamma)c_0(v)}
\leq\frac{2R}{(1-\gamma)\underline c}.
\end{equation}
Accordingly, fix a finite endpoint satisfying
\begin{equation}
\label{eq:uniform-dual-cap}
\lambda_{\max}\geq\frac{4R}{(1-\gamma)\underline c}.
\end{equation}
Here $R>0$ is the reward-magnitude bound. Then
$\lambda_{\max}\geq2\lambda^\star(v)$ for every $v$, as required by the
Slater decoupling argument. Neither $\pi_{\rm sl}(v)$ nor $c_0(v)$ is used
in the lower updates; the algorithm uses only the fixed dual interval.
In particular, its endpoint has no explicit dependence on $v$.

For the remainder of the proof, let
\[
b_k
:=
\nabla_\theta
J_L(\theta_k,\lambda_k;v),
\qquad
F_k
:=
F_\rho(\theta_k),
\]
and
\[
\omega_k^\star
=
F_k^{-1}b_k.
\]
By Assumptions~\ref{ass:score}--\ref{ass:fisher-lower}, there exists a
constant $R_\omega<\infty$ such that
\begin{equation}
\label{eq:omega-star-bound}
\|\omega_k^\star\|
\leq
R_\omega
\end{equation}
uniformly in $k$.

\textbf{Properties of the Encoded Lower-Level Oracle.} We first show that the
signal encoding required by the upper-level analysis preserves the conditional
means of the population quantities used by the lower-level critic and dual
updates.

\begin{lemma}[Unbiasedness of \rev{randomized binary encoding}]
\label{lem:encoding-unbiased}
Let $a(v)$ be a scalar signal satisfying $|a(v)|\leq A$, and let
$B\in\{-1,+1\}$ satisfy
\[
\mathbb P_v(B=b)
=
\frac12+\frac{b\,a(v)}{4A},
\qquad b\in\{-1,+1\}.
\]
If $\widetilde a=2AB$, then
\[
\mathbb E[\widetilde a\mid a(v)]
=
a(v).
\]
Consequently, replacing every sampled reward and constraint signal in the
geometric-horizon lower-level oracle by its independently decoded \rev{binary-encoded}
signal preserves the conditional means of the constraint-return and
critic-gradient estimators:
\[
\mathbb E_k[
\widehat J_{c,\rho}(\theta_k;y)
]
=
J_{c,\rho}(\theta_k;y),
\]
and, for every critic vector $\omega$,
\[
\mathbb E_{k,n}
\left[
\widehat\nabla_\omega
\mathcal L_\rho
(\omega;\theta_k,\lambda_k,v)
\right]
=
\nabla_\omega
\mathcal L_\rho
(\omega;\theta_k,\lambda_k,v).
\]
Here $\mathbb E_k[\cdot]$ denotes expectation conditional on the lower
history up to policy iteration $k$, and $\mathbb E_{k,n}[\cdot]$ further
conditions on the critic history up to inner iteration $n$.
\end{lemma}

\begin{proof}
By definition,
\[
\mathbb E[\widetilde a]
=
2A
\left[
\mathbb P(B=1)-\mathbb P(B=-1)
\right].
\]
Substituting the encoding probabilities gives
\[
\mathbb E[\widetilde a]
=
2A
\left[
\left(
\frac12+\frac{a(v)}{4A}
\right)
-
\left(
\frac12-\frac{a(v)}{4A}
\right)
\right]
=
a(v).
\]

Condition on an arbitrary realized state-action trajectory. Every decoded
reward and constraint signal is conditionally unbiased for its
corresponding raw signal. Since the return, $Q$-function, $V$-function, and
advantage estimators used by the critic oracle are linear combinations of
these signals, iterated expectation gives
\[
\mathbb E[
\widehat A_r^{\pi_\theta}(s,a)
\mid s,a
]
=
A_r^{\pi_\theta}(s,a)
\]
and
\[
\mathbb E[
\widehat A_c^{\pi_\theta}(s,a)
\mid s,a
]
=
A_c^{\pi_\theta}(s,a).
\]
Therefore,
\[
\mathbb E[
\widehat A_{L,\lambda}^{\pi_\theta}(s,a)
\mid s,a
]
=
A_{L,\lambda}^{\pi_\theta}(s,a).
\]

The Fisher estimator does not depend on the reward or constraint signals.
Hence
\[
\mathbb E[
\widehat F_\rho(\theta)
]
=
F_\rho(\theta).
\]
Combining the two identities with
\[
\widehat\nabla_\omega\mathcal L_\rho
=
\widehat F_\rho(\theta)\omega
-
\frac{1}{1-\gamma}
\widehat H_\rho(\theta,\lambda)
\]
yields
\[
\mathbb E[
\widehat\nabla_\omega\mathcal L_\rho
(\omega;\theta,\lambda,v)
]
=
F_\rho(\theta)\omega
-
\nabla_\theta J_L(\theta,\lambda;v),
\]
which is precisely
$\nabla_\omega\mathcal L_\rho(\omega;\theta,\lambda,v)$.
The constraint-return identity follows analogously.
\end{proof}

The encoding increases the variance of the lower oracle but does not change
its mean. The following lemma records the moment property needed below.

\begin{lemma}[Second moment of the stochastic critic error]
\label{lem:critic-noise}
There exist finite constants $\sigma_0^2,\sigma_1^2$, independent of
$k,n$, such that
\[
\mathbb E_{k,n}
\left[
\left\|
\widehat\nabla_\omega\mathcal L_\rho
(\omega;\theta_k,\lambda_k,v)
-
\nabla_\omega\mathcal L_\rho
(\omega;\theta_k,\lambda_k,v)
\right\|^2
\right]
\leq
\sigma_0^2
+
\sigma_1^2
\|\omega-\omega_k^\star\|^2.
\]
Moreover, the encoded constraint-return estimator satisfies
\[
\mathbb E_k
\left[
\widehat J_{c,\rho}(\theta_k;y)^2
\right]
\leq
C_J^2
\]
for a finite constant $C_J$.
\end{lemma}

\begin{proof}
Write
\[
\widehat g_k(\omega)
=
\widehat F_k\omega-\widehat b_k,
\qquad
g_k(\omega)
=
F_k\omega-b_k.
\]
Since
\[
\omega
=
(\omega-\omega_k^\star)+\omega_k^\star,
\]
we have
\[
\widehat g_k(\omega)-g_k(\omega)
=
(\widehat F_k-F_k)(\omega-\omega_k^\star)
+
\left(
\widehat F_k\omega_k^\star-\widehat b_k
\right),
\]
where
\[
F_k\omega_k^\star-b_k=0.
\]
Therefore,
\[
\|
\widehat g_k(\omega)-g_k(\omega)
\|^2
\leq
2\|\widehat F_k-F_k\|^2
\|\omega-\omega_k^\star\|^2
+
2\|
\widehat F_k\omega_k^\star-\widehat b_k
\|^2.
\]

Assumption~\ref{ass:score} gives
\[
\|\widehat F_k\|
\leq
G_\pi^2,
\qquad
\|F_k\|
\leq
G_\pi^2,
\]
and consequently
\[
\|\widehat F_k-F_k\|
\leq
2G_\pi^2.
\]
Thus the first term is bounded by
\[
8G_\pi^4
\|\omega-\omega_k^\star\|^2.
\]

The decoded signals are bounded by $2R$ and $2C$, respectively. The
geometric horizon has finite moments of every fixed order for
$\gamma<1$. Together with the bounded score function and
\eqref{eq:omega-star-bound}, this implies the existence of a finite uniform
constant $\sigma_\star^2$ satisfying
\[
\mathbb E_k
\left[
\|
\widehat F_k\omega_k^\star-\widehat b_k
\|^2
\right]
\leq
\sigma_\star^2.
\]
Hence the claim follows with, for example,
\[
\sigma_1^2=8G_\pi^4,
\qquad
\sigma_0^2=2\sigma_\star^2.
\]

The same bounded-signal and geometric-horizon argument yields a finite
uniform second-moment bound $C_J^2$ for
$\widehat J_{c,\rho}(\theta_k;y)$.
\end{proof}

\textbf{Vanilla-SGD Natural-Gradient Estimation.} We now establish the
vanilla-SGD analogue of the accelerated critic analysis in \cite{mondal2024constrained}.

\begin{lemma}[Bias and mean-square error of the vanilla-SGD critic]
\label{lem:sgd-critic}
Fix a lower iteration $k$. Let
\[
\omega_{k,n+1}
=
\omega_{k,n}
-
\beta
\widehat\nabla_\omega
\mathcal L_\rho
(\omega_{k,n};\theta_k,\lambda_k,v),
\qquad
\omega_{k,0}=0.
\]
Suppose
\[
0<\beta
\leq
\min
\left\{
\frac{1}{G_\pi^2},
\frac{\mu_F}{\mu_F^2+\sigma_1^2}
\right\}.
\]
Then the last iterate
\[
\widehat\omega_k
=
\omega_{k,N}
\]
satisfies
\begin{equation}
\label{eq:sgd-bias}
\left\|
\mathbb E_k[\widehat\omega_k]
-
\omega_k^\star
\right\|
\leq
e^{-\beta\mu_F N}
R_\omega,
\end{equation}
and
\begin{equation}
\label{eq:sgd-mse}
\mathbb E_k
\left[
\|
\widehat\omega_k-\omega_k^\star
\|^2
\right]
\leq
(1-\beta\mu_F)^N R_\omega^2
+
\frac{\beta\sigma_0^2}{\mu_F}.
\end{equation}
Consequently,
\begin{equation}
\label{eq:omega-second-moment}
\mathbb E_k
\left[
\|\widehat\omega_k\|^2
\right]
\leq
W_\omega^2,
\end{equation}
where
\begin{equation}
\label{eq:W-omega}
W_\omega^2
:=
4R_\omega^2
+
\frac{2\beta\sigma_0^2}{\mu_F}.
\end{equation}
\end{lemma}

\begin{proof}
Define
\[
e_{k,n}
=
\omega_{k,n}-\omega_k^\star.
\]
Since
\[
\nabla_\omega
\mathcal L_\rho
(\omega;\theta_k,\lambda_k,v)
=
F_k(\omega-\omega_k^\star),
\]
the SGD recursion can be written as
\[
e_{k,n+1}
=
(I-\beta F_k)e_{k,n}
-
\beta\xi_{k,n},
\]
where
\[
\xi_{k,n}
=
\widehat\nabla_\omega
\mathcal L_\rho
(\omega_{k,n};\theta_k,\lambda_k,v)
-
\nabla_\omega
\mathcal L_\rho
(\omega_{k,n};\theta_k,\lambda_k,v).
\]
By Lemma~\ref{lem:encoding-unbiased},
\[
\mathbb E_{k,n}[\xi_{k,n}]
=
0.
\]

We first analyze the conditional mean. Since $e_{k,n}$ is measurable with
respect to the critic history before drawing the fresh sample at iteration
$n$,
\[
\mathbb E_k[e_{k,n+1}]
=
(I-\beta F_k)
\mathbb E_k[e_{k,n}].
\]
Iterating gives
\[
\mathbb E_k[e_{k,N}]
=
(I-\beta F_k)^N e_{k,0}.
\]
Since
\[
\mu_F I
\preceq
F_k
\preceq
G_\pi^2I
\]
and $\beta\leq G_\pi^{-2}$,
\[
\|I-\beta F_k\|
\leq
1-\beta\mu_F.
\]
Therefore,
\[
\|
\mathbb E_k[e_{k,N}]
\|
\leq
(1-\beta\mu_F)^N
\|e_{k,0}\|.
\]
Using $\omega_{k,0}=0$ and
$\|\omega_k^\star\|\leq R_\omega$,
\[
\|
\mathbb E_k[\widehat\omega_k]
-
\omega_k^\star
\|
\leq
(1-\beta\mu_F)^N R_\omega
\leq
e^{-\beta\mu_FN}R_\omega,
\]
which proves~\eqref{eq:sgd-bias}.

For the mean-square error,
\[
\mathbb E_{k,n}
\|e_{k,n+1}\|^2
=
\|(I-\beta F_k)e_{k,n}\|^2
+
\beta^2
\mathbb E_{k,n}\|\xi_{k,n}\|^2,
\]
because the cross term vanishes conditionally. Applying
Lemma~\ref{lem:critic-noise},
\[
\mathbb E_{k,n}
\|e_{k,n+1}\|^2
\leq
(1-\beta\mu_F)^2
\|e_{k,n}\|^2
+
\beta^2
\left(
\sigma_0^2+\sigma_1^2\|e_{k,n}\|^2
\right).
\]
Hence
\[
\mathbb E_{k,n}
\|e_{k,n+1}\|^2
\leq
\left[
1-2\beta\mu_F
+
\beta^2(\mu_F^2+\sigma_1^2)
\right]
\|e_{k,n}\|^2
+
\beta^2\sigma_0^2.
\]
The stepsize condition
\[
\beta(\mu_F^2+\sigma_1^2)
\leq
\mu_F
\]
implies
\[
\mathbb E_{k,n}
\|e_{k,n+1}\|^2
\leq
(1-\beta\mu_F)
\|e_{k,n}\|^2
+
\beta^2\sigma_0^2.
\]
Taking total expectation and iterating yields
\[
\mathbb E_k
\|e_{k,N}\|^2
\leq
(1-\beta\mu_F)^N
\|e_{k,0}\|^2
+
\beta^2\sigma_0^2
\sum_{j=0}^{N-1}
(1-\beta\mu_F)^j.
\]
Using
\[
\sum_{j=0}^{N-1}
(1-\beta\mu_F)^j
\leq
\frac{1}{\beta\mu_F}
\]
and $\|e_{k,0}\|\leq R_\omega$ gives
\[
\mathbb E_k
\|e_{k,N}\|^2
\leq
(1-\beta\mu_F)^N R_\omega^2
+
\frac{\beta\sigma_0^2}{\mu_F},
\]
which proves~\eqref{eq:sgd-mse}.

Finally,
\[
\|\widehat\omega_k\|^2
\leq
2\|
\widehat\omega_k-\omega_k^\star
\|^2
+
2\|\omega_k^\star\|^2.
\]
Using~\eqref{eq:sgd-mse} and
$\|\omega_k^\star\|\leq R_\omega$ yields
\[
\mathbb E_k\|\widehat\omega_k\|^2
\leq
4R_\omega^2
+
\frac{2\beta\sigma_0^2}{\mu_F}
=
W_\omega^2.
\]
\end{proof}

\begin{remark}
\label{rem:sgd-key}
The distinction between
\eqref{eq:sgd-bias} and~\eqref{eq:sgd-mse} is important. The conditional
bias decreases geometrically with the number $N$ of critic iterations,
whereas the mean-square error may converge only to a constant noise floor.
This is sufficient for the lower-level CMDP analysis because the conditional
bias appears as a first-order error term, while the second moment of
$\widehat\omega_k$ is multiplied by the policy stepsize $\eta$.
\end{remark}

\textbf{Local-to-Global Lagrangian Comparison}

Define
\[
D_k
=
\int_{\mathcal S}d_\rho^{\pi^\star}(s)
D_{\rm KL}
\left(
\pi^\star(\cdot\mid s)
\|
\pi_{\theta_k}(\cdot\mid s)
\right)\,ds.
\]

Since $\mathcal A$ is finite, the KL divergence here is the finite action
sum, weighted by the state occupancy density and integrated with $ds$.
Assumption~\ref{ass:score}
ensures
\[
0\leq D_0\leq\log(1/p_0),
\]
uniformly in $v$ and in the optimal comparator $\pi^\star$. The bounded
log-policy gradients also keep these integrals finite at every finite
iterate.

\begin{lemma}[Local-to-global comparison]
\label{lem:local-global-sgd}
Under Assumptions~\ref{ass:score}--\ref{ass:compatibility}, the iterates of
Algorithm~\ref{alg:lower-pd-npg-sgd} satisfy
\begin{equation}
\label{eq:local-global-sgd}
\frac1K
\sum_{k=0}^{K-1}
\mathbb E
\left[
J_L(\pi^\star,\lambda_k;v)
-
J_L(\theta_k,\lambda_k;v)
\right]
\leq
\sqrt{\epsilon_{\rm bias}}
+
\frac{G_\pi}{K}
\sum_{k=0}^{K-1}
\mathbb E
\left[
\left\|
\mathbb E_k[\widehat\omega_k]
-
\omega_k^\star
\right\|
\right]
+
\frac{L_\pi\eta}{2K}
\sum_{k=0}^{K-1}
\mathbb E\|\widehat\omega_k\|^2
+
\frac{D_0}{\eta K}.
\end{equation}
\end{lemma}

\begin{proof}
The $L_\pi$-smoothness of
$\log\pi_\theta(a|s)$ gives
\[
\log\pi_{\theta_{k+1}}(a|s)
\geq
\log\pi_{\theta_k}(a|s)
+
\eta
\nabla_\theta\log\pi_{\theta_k}(a|s)^\top
\widehat\omega_k
-
\frac{L_\pi\eta^2}{2}
\|\widehat\omega_k\|^2.
\]
Taking expectation over
$s\sim d_\rho^{\pi^\star}$ and
$a\sim\pi^\star(\cdot|s)$ yields
\[
D_{k+1}
\leq
D_k
-
\eta
\mathbb E_{\nu_\rho^{\pi^\star}}
\left[
\nabla_\theta\log\pi_{\theta_k}(a|s)^\top
\widehat\omega_k
\right]
+
\frac{L_\pi\eta^2}{2}
\|\widehat\omega_k\|^2.
\]
Conditioning on the history up to iteration $k$,
\[
\mathbb E_k[D_{k+1}]
\leq
D_k
-
\eta
\mathbb E_{\nu_\rho^{\pi^\star}}
\left[
\nabla_\theta\log\pi_{\theta_k}(a|s)^\top
\mathbb E_k[\widehat\omega_k]
\right]
+
\frac{L_\pi\eta^2}{2}
\mathbb E_k
\|\widehat\omega_k\|^2.
\]

Add and subtract $\omega_k^\star$. Assumption~\ref{ass:score} gives
\[
\left|
\mathbb E_{\nu_\rho^{\pi^\star}}
\left[
\nabla_\theta\log\pi_{\theta_k}(a|s)^\top
\left(
\mathbb E_k[\widehat\omega_k]
-
\omega_k^\star
\right)
\right]
\right|
\leq
G_\pi
\left\|
\mathbb E_k[\widehat\omega_k]
-
\omega_k^\star
\right\|.
\]

The performance-difference identity together with the compatible-function
approximation assumption implies
\[
J_L(\pi^\star,\lambda_k;v)
-
J_L(\theta_k,\lambda_k;v)
\leq
\mathbb E_{\nu_\rho^{\pi^\star}}
\left[
\nabla_\theta\log\pi_{\theta_k}(a|s)^\top
\omega_k^\star
\right]
+
\sqrt{\epsilon_{\rm bias}}.
\]
Substituting the previous two inequalities gives
\[
\eta
\mathbb E_k
\left[
J_L(\pi^\star,\lambda_k;v)
-
J_L(\theta_k,\lambda_k;v)
\right]
\leq
D_k-\mathbb E_k[D_{k+1}]
+
\eta\sqrt{\epsilon_{\rm bias}}
\]
\[
\qquad
+
\eta G_\pi
\left\|
\mathbb E_k[\widehat\omega_k]
-
\omega_k^\star
\right\|
+
\frac{L_\pi\eta^2}{2}
\mathbb E_k
\|\widehat\omega_k\|^2.
\]
Summing over $k=0,\ldots,K-1$, taking total expectations, using
$D_K\geq0$, and dividing by $\eta K$ proves
\eqref{eq:local-global-sgd}.
\end{proof}

Combining Lemmas~\ref{lem:sgd-critic} and
\ref{lem:local-global-sgd} gives the following immediate result.

\begin{corollary}[Average Lagrangian convergence]
\label{cor:lagrangian-sgd}
Under the assumptions above,
\begin{equation}
\label{eq:lagrangian-gap-sgd}
\frac1K
\sum_{k=0}^{K-1}
\mathbb E
\left[
J_L(\pi^\star,\lambda_k;v)
-
J_L(\theta_k,\lambda_k;v)
\right]
\leq
\sqrt{\epsilon_{\rm bias}}
+
G_\pi R_\omega e^{-\beta\mu_FN}
+
\frac{L_\pi\eta}{2}W_\omega^2
+
\frac{D_0}{\eta K}.
\end{equation}
\end{corollary}

\begin{proof}
Apply~\eqref{eq:sgd-bias} and
\eqref{eq:omega-second-moment} uniformly in $k$ to
\eqref{eq:local-global-sgd}.
\end{proof}

For convenience, define
\begin{equation}
\label{eq:Delta-lower}
\Delta_{K,N}
:=
\sqrt{\epsilon_{\rm bias}}
+
G_\pi R_\omega e^{-\beta\mu_FN}
+
\frac{L_\pi\eta}{2}W_\omega^2
+
\frac{D_0}{\eta K}.
\end{equation}

\textbf{Projected Dual Recursion}

The projected Lagrange-multiplier update is
\[
\lambda_{k+1}
=
\Pi_{[0,\lambda_{\max}]}
\left[
\lambda_k
-
\zeta
\widehat J_{c,\rho}(\theta_k;y)
\right].
\]

\begin{lemma}[Dual cross-term bounds]
\label{lem:dual-bounds}
Suppose $\lambda_0=0$. Then
\begin{equation}
\label{eq:dual-zero-bound}
\frac1K
\sum_{k=0}^{K-1}
\mathbb E
\left[
\lambda_k J_c(\theta_k,y)
\right]
\leq
\frac{\zeta C_J^2}{2},
\end{equation}
and
\begin{equation}
\label{eq:dual-max-bound}
\frac1K
\sum_{k=0}^{K-1}
\mathbb E
\left[
(\lambda_k-\lambda_{\max})
J_c(\theta_k,y)
\right]
\leq
\frac{\lambda_{\max}^2}{2\zeta K}
+
\frac{\zeta C_J^2}{2}.
\end{equation}
\end{lemma}

\begin{proof}
By non-expansiveness of Euclidean projection and $\lambda_0=0$,
\[
\lambda_{k+1}^2
\leq
\left(
\lambda_k
-
\zeta\widehat J_{c,\rho}(\theta_k;y)
\right)^2.
\]
Therefore,
\[
\lambda_{k+1}^2-\lambda_k^2
\leq
-2\zeta
\lambda_k
\widehat J_{c,\rho}(\theta_k;y)
+
\zeta^2
\widehat J_{c,\rho}(\theta_k;y)^2.
\]
Summing and using $\lambda_K^2\geq0$ gives
\[
\frac1K
\sum_{k=0}^{K-1}
\mathbb E
\left[
\lambda_k
\widehat J_{c,\rho}(\theta_k;y)
\right]
\leq
\frac{\zeta}{2K}
\sum_{k=0}^{K-1}
\mathbb E
\left[
\widehat J_{c,\rho}(\theta_k;y)^2
\right].
\]
Lemma~\ref{lem:encoding-unbiased} and conditional independence imply
\[
\mathbb E[
\lambda_k
\widehat J_{c,\rho}(\theta_k;y)
]
=
\mathbb E[
\lambda_kJ_c(\theta_k,y)
].
\]
Lemma~\ref{lem:critic-noise} then yields
\eqref{eq:dual-zero-bound}.

For the second claim, use non-expansiveness relative to the feasible point
$\lambda_{\max}$:
\[
|\lambda_{k+1}-\lambda_{\max}|^2
\leq
|
\lambda_k-\zeta\widehat J_{c,\rho}(\theta_k;y)
-\lambda_{\max}
|^2.
\]
Expanding gives
\[
|\lambda_{k+1}-\lambda_{\max}|^2
\leq
|\lambda_k-\lambda_{\max}|^2
-
2\zeta
(\lambda_k-\lambda_{\max})
\widehat J_{c,\rho}(\theta_k;y)
+
\zeta^2
\widehat J_{c,\rho}(\theta_k;y)^2.
\]
Summing, using
$|\lambda_0-\lambda_{\max}|^2=\lambda_{\max}^2$ and dropping the
nonnegative terminal distance yields
\[
\frac1K
\sum_{k=0}^{K-1}
\mathbb E
[
(\lambda_k-\lambda_{\max})
J_c(\theta_k,y)
]
\leq
\frac{\lambda_{\max}^2}{2\zeta K}
+
\frac{\zeta C_J^2}{2}.
\]
\end{proof}

\textbf{Reward and Constraint Convergence}

\begin{lemma}[Reward suboptimality]
\label{lem:reward-gap}
The lower iterates satisfy
\begin{equation}
\label{eq:reward-gap}
\frac1K
\sum_{k=0}^{K-1}
\mathbb E
\left[
J_r^\star(v)
-
J_r(\theta_k,x)
\right]
\leq
\Delta_{K,N}
+
\frac{\zeta C_J^2}{2}.
\end{equation}
\end{lemma}

\begin{proof}
Since $\pi^\star$ is feasible,
\[
J_c(\pi^\star,y)\geq0.
\]
Therefore,
\[
J_L(\pi^\star,\lambda_k;v)
-
J_L(\theta_k,\lambda_k;v)
\geq
J_r^\star(v)
-
J_r(\theta_k,x)
-
\lambda_kJ_c(\theta_k,y).
\]
Rearranging, averaging, and applying
Corollary~\ref{cor:lagrangian-sgd} gives
\[
\frac1K
\sum_{k=0}^{K-1}
\mathbb E
[
J_r^\star(v)-J_r(\theta_k,x)
]
\leq
\Delta_{K,N}
+
\frac1K
\sum_{k=0}^{K-1}
\mathbb E[
\lambda_kJ_c(\theta_k,y)
].
\]
Apply~\eqref{eq:dual-zero-bound}.
\end{proof}

We next recall the standard Slater-based decoupling result.

\begin{lemma}[Slater decoupling]
\label{lem:slater-decoupling}
For the fixed upper-level parameter $v$, write $\lambda^\star=\lambda^\star(v)$
for an optimal multiplier supplied by
Lemma~\ref{lem:continuous-cmdp-duality}.
Suppose $C_{\rm pen}>0$, $C_{\rm pen}\geq2\lambda^\star$, and
$\delta\geq0$.
If a stationary randomized policy $\bar\pi$ satisfies
\[
J_r^\star(v)
-
J_r(\bar\pi,x)
+
C_{\rm pen}[-J_c(\bar\pi,y)]
\leq
\delta,
\]
then
\[
-J_c(\bar\pi,y)
\leq
\frac{2\delta}{C_{\rm pen}}.
\]
\end{lemma}

\begin{proof}
If $J_c(\bar\pi,y)\geq0$, the claim is immediate. Suppose
$J_c(\bar\pi,y)<0$ and set
\[
\tau
=
J_c(\bar\pi,y).
\]
Define
\[
V(\tau)
=
\max_\pi
\left\{
J_r(\pi,x):
J_c(\pi,y)\geq\tau
\right\}.
\]
Since $\bar\pi$ is feasible for this relaxed problem,
\[
J_r(\bar\pi,x)
\leq
V(\tau).
\]
By weak duality for the relaxed problem and the attained strong-duality
identity~\eqref{eq:continuous-strong-duality} for the original CMDP,
\[
V(\tau)
\leq
J_r^\star(v)
-
\lambda^\star\tau.
\]
Therefore,
\[
J_r^\star(v)-J_r(\bar\pi,x)
\geq
\lambda^\star\tau.
\]
Since $\tau<0$,
\[
(C_{\rm pen}-\lambda^\star)(-\tau)
\leq
J_r^\star(v)
-
J_r(\bar\pi,x)
+
C_{\rm pen}(-\tau)
\leq
\delta.
\]
Because $C_{\rm pen}\geq2\lambda^\star$,
\[
C_{\rm pen}-\lambda^\star
\geq
\frac{C_{\rm pen}}{2},
\]
and hence
\[
-J_c(\bar\pi,y)
=
-\tau
\leq
\frac{2\delta}{C_{\rm pen}}.
\]
\end{proof}

\subsection{Proof of Theorem \ref{thm:lower-convergence}}
\begin{theorem}[Approximate lower-level response with vanilla SGD]
\label{thm:lower-sgd}
Suppose Assumptions~\ref{ass:cmdp-regularity}--\ref{ass:compatibility}
and~\ref{ass:slater-lower}
hold, and suppose the boundedness part of
Assumption~\ref{ass:outer-signal} holds. Fix
$\lambda_{\max}$ according to~\eqref{eq:uniform-dual-cap}. Let
Algorithm~\ref{alg:lower-pd-npg-sgd} use
\[
\eta
=
\frac{\eta_0}{\sqrt K},
\qquad
\zeta
=
\frac{\zeta_0}{\sqrt K},
\]
for positive constants $\eta_0,\zeta_0$, choose the critic stepsize $\beta$
to satisfy the condition in Lemma~\ref{lem:sgd-critic}, and choose the number
of critic iterations so that
\begin{equation}
\label{eq:N-choice}
N
\geq
\max
\left\{
1,
\left\lceil
\frac{1}{\beta\mu_F}
\log
\left(
R_\omega\sqrt K
\right)
\right\rceil
\right\}.
\end{equation}
Then there exist problem-dependent constants $C_r,C_c>0$, independent of
$K$ and $N$, such that
\begin{equation}
\label{eq:lower-reward-average-app}
\frac1K
\sum_{k=0}^{K-1}
\mathbb E
\left[
J_r^\star(v)
-
J_r(\theta_k,x)
\right]
\leq
\sqrt{\epsilon_{\rm bias}}
+
\frac{C_r}{\sqrt K},
\end{equation}
and
\begin{equation}
\label{eq:lower-constraint-average-app}
\frac1K
\sum_{k=0}^{K-1}
\mathbb E
\left[
-J_c(\theta_k,y)
\right]
\leq
C_{\rm sl}
\sqrt{\epsilon_{\rm bias}}
+
\frac{C_c}{\sqrt K},
\end{equation}
where one may take $C_{\rm sl}=2/\lambda_{\max}$. Thus $C_{\rm sl}$ is
independent of $v$; the dependence of the bounds on Slater feasibility enters
through the common lower bound $\underline c$ and the fixed dual endpoint
$\lambda_{\max}$, rather than a separately chosen endpoint for each $v$.

If
\[
R
\sim
\operatorname{Unif}\{0,\ldots,K-1\}
\]
is selected independently after the lower run and $Z\sim P_v$ denotes the
resulting complete trace, then $\theta_R=\theta(Z)$ and
\begin{equation}
\label{eq:approx-response-reward}
\mathbb E
\left[
J_r^\star(v)
-
J_r(\theta_R,x)
\right]
\leq
\sqrt{\epsilon_{\rm bias}}
+
\frac{C_r}{\sqrt K},
\end{equation}
and
\begin{equation}
\label{eq:approx-response-constraint}
\mathbb E
\left[
-J_c(\theta_R,y)
\right]
\leq
C_{\rm sl}
\sqrt{\epsilon_{\rm bias}}
+
\frac{C_c}{\sqrt K}.
\end{equation}
\end{theorem}

\begin{proof}
By~\eqref{eq:N-choice},
\[
R_\omega e^{-\beta\mu_FN}
\leq
\frac{1}{\sqrt K}.
\]
Substituting
$\eta=\eta_0/\sqrt K$ into~\eqref{eq:Delta-lower} gives
\[
\Delta_{K,N}
\leq
\sqrt{\epsilon_{\rm bias}}
+
\frac{G_\pi}{\sqrt K}
+
\frac{L_\pi\eta_0W_\omega^2}{2\sqrt K}
+
\frac{D_0}{\eta_0\sqrt K}.
\]
Hence there exists $C_L>0$ such that
\[
\Delta_{K,N}
\leq
\sqrt{\epsilon_{\rm bias}}
+
\frac{C_L}{\sqrt K}.
\]

Lemma~\ref{lem:reward-gap} and
$\zeta=\zeta_0/\sqrt K$ imply
\[
\frac1K
\sum_{k=0}^{K-1}
\mathbb E
[
J_r^\star(v)-J_r(\theta_k,x)
]
\leq
\sqrt{\epsilon_{\rm bias}}
+
\frac{1}{\sqrt K}
\left(
C_L+\frac{\zeta_0C_J^2}{2}
\right).
\]
This proves~\eqref{eq:lower-reward-average-app}.

For feasibility, combining
Corollary~\ref{cor:lagrangian-sgd} with
\eqref{eq:dual-max-bound} gives
\[
\frac1K
\sum_{k=0}^{K-1}
\mathbb E
\left[
J_r^\star(v)-J_r(\theta_k,x)
+
\lambda_{\max}(-J_c(\theta_k,y))
\right]
\leq
\Delta_{K,N}
+
\frac{\lambda_{\max}^2}{2\zeta K}
+
\frac{\zeta C_J^2}{2}.
\]
With $\zeta=\zeta_0/\sqrt K$,
\[
\frac{\lambda_{\max}^2}{2\zeta K}
+
\frac{\zeta C_J^2}{2}
=
\frac{1}{\sqrt K}
\left(
\frac{\lambda_{\max}^2}{2\zeta_0}
+
\frac{\zeta_0C_J^2}{2}
\right).
\]
Hence the right-hand side is
\[
\sqrt{\epsilon_{\rm bias}}
+
O(K^{-1/2}).
\]

For this fixed $v$, form the deterministic expected average occupancy
density
\[
\bar\nu(s,a):=\frac1K\sum_{k=0}^{K-1}
\mathbb E[\nu_\rho^{\pi_{\theta_k}}(s,a)],
\qquad
\bar d(s):=\sum_{a\in\mathcal A}\bar\nu(s,a).
\]
Define $\bar\pi(a\mid s)=\bar\nu(s,a)/\bar d(s)$ wherever
$\bar d(s)>0$, and choose the uniform action distribution elsewhere.
The averaged density satisfies the linear flow
identity~\eqref{eq:occupancy-density-flow}. By
Lemma~\ref{lem:continuous-cmdp-duality}, it is the occupancy density of
$\bar\pi$. Boundedness, Fubini's theorem, and linearity of the returns
in~\eqref{eq:continuous-return} give
\[
J_g(\bar\pi)
=
\frac1K
\sum_{k=0}^{K-1}
\mathbb E[J_g(\theta_k)],
\qquad
g\in\{r,c\}.
\]
Thus the preceding expected combined-gap bound is a deterministic bound
for $\bar\pi$, to which the decoupling lemma applies directly.
By~\eqref{eq:uniform-dual-bound}--\eqref{eq:uniform-dual-cap}, the fixed
endpoint satisfies $\lambda_{\max}\geq2\lambda^\star(v)$.
Applying Lemma~\ref{lem:slater-decoupling} with
$C_{\rm pen}=\lambda_{\max}$ yields
\[
\frac1K
\sum_{k=0}^{K-1}
\mathbb E[-J_c(\theta_k,y)]
\leq
C_{\rm sl}\sqrt{\epsilon_{\rm bias}}
+
\frac{C_c}{\sqrt K}
\]
for a finite problem-dependent $C_c$. This proves
\eqref{eq:lower-constraint-average-app}.

Finally, because $R$ is uniform and independent,
\[
\mathbb E[
J_r^\star(v)-J_r(\theta_R,x)
]
=
\frac1K
\sum_{k=0}^{K-1}
\mathbb E[
J_r^\star(v)-J_r(\theta_k,x)
],
\]
and
\[
\mathbb E[
-J_c(\theta_R,y)
]
=
\frac1K
\sum_{k=0}^{K-1}
\mathbb E[
-J_c(\theta_k,y)
].
\]
The bounds~\eqref{eq:approx-response-reward} and
\eqref{eq:approx-response-constraint} follow immediately.
\end{proof}

\begin{corollary}[Lower-level iteration and sample complexity]
\label{cor:lower-complexity}
For any target numerical accuracy $\delta_\ell>0$, choose
\[
K
=
\Theta(\delta_\ell^{-2})
\]
and
\[
N
=
\Theta(\log(\delta_\ell^{-1})).
\]
Define the effective response accuracy by
$\epsilon_\ell:=\max\{\sqrt{\epsilon_{\rm bias}},\delta_\ell\}$ and let
$\rev{D_r}:=1+C_r$ and
$\rev{D_c}:=C_{\rm sl}+C_c$. Then
\[
\mathbb E
\left[
J_r^\star(v)-J_r(\theta_R,x)
\right]
\leq
\rev{D_r}\epsilon_\ell,
\]
and
\[
\mathbb E
\left[
-J_c(\theta_R,y)
\right]
\leq
\rev{D_c}\epsilon_\ell.
\]

Since each stochastic critic call uses a geometric-horizon trajectory with
expected length $O((1-\gamma)^{-1})$, the total number of lower-level
environment transitions is
\[
\widetilde{\mathcal{O}}
\left(
\frac{KN}{1-\gamma}
\right)
=
\widetilde{\mathcal{O}}
\left(
\frac{1}{(1-\gamma)\delta_\ell^2}
\right),
\]
up to the additional polynomial dependence on $(1-\gamma)^{-1}$ contained
in the problem-dependent constants and stepsizes.
\end{corollary}

\begin{proof}
The iteration choices follow directly from
Theorem~\ref{thm:lower-sgd}. Since
$K=\Theta(\delta_\ell^{-2})$,
we have $K^{-1/2}=O(\delta_\ell)$, while
the concrete choice
$N=\Theta(\log K)=\Theta(\log(\delta_\ell^{-1}))$ satisfies
\eqref{eq:N-choice}. The effective
bounds follow because both $\sqrt{\epsilon_{\rm bias}}$ and
$\delta_\ell$ are bounded by $\epsilon_\ell$.

Each critic iteration invokes one geometric-horizon stochastic oracle whose
expected trajectory length is proportional to $(1-\gamma)^{-1}$. Summing
over the $N$ critic calls at each of the $K$ primal-dual iterations gives
the stated order, with logarithmic factors absorbed into
$\widetilde{\mathcal{O}}(\cdot)$.
\end{proof}

\section{Proof of Theorem \ref{thm:training-distribution-gradient}}
\label{app:proof-training-gradient}
\textbf{Exact Gradient of the Approximate-Response Objective.}

The lower initialization, iteration counts, step sizes, dual interval,
and decoding bounds are fixed independently of $v$ for each prescribed
numerical accuracy $\delta_\ell$ and corresponding effective response accuracy
$\epsilon_\ell=\max\{\sqrt{\epsilon_{\rm bias}},\delta_\ell\}$.

\begin{lemma}[Training-distribution gradient]
\label{lem:training-gradient}
Under Assumptions~\ref{ass:outer-signal} and
\ref{ass:upper-objective}, the objective
\[
\Phi_{\epsilon_\ell}(v)
=
\mathbb E_{Z\sim P_v}
\left[
G(v,\theta(Z))
\right]
\]
is differentiable and, for any deterministic scalar baseline $c(v)$,
satisfies
\begin{equation}
\label{eq:training-gradient}
\nabla\Phi_{\epsilon_\ell}(v)
=
\mathbb E_v
\left[
\nabla_vG(v,\theta(Z))
+
\bigl(
G(v,\theta(Z))-c(v)
\bigr)
S(Z;v)
\right].
\end{equation}
Here $S(Z;v)$ is the sum of the score functions of all reward and
constraint encoding variables generated during the lower-level training
run.
\end{lemma}

\begin{proof}
Let
\[
\mathcal T
=
(T_1,\ldots,T_{3K(N+1)})
\]
collect all geometric horizons sampled during one complete execution of
the lower-level algorithm. The number of horizon variables is deterministic,
and every $T_i$ is sampled independently of the upper-level parameter $v$.
Therefore,
\begin{equation}
\label{eq:condition-horizon}
\Phi_{\epsilon_\ell}(v)
=
\mathbb E_{\mathcal T}
\left[
\Phi_{\mathcal T}(v)
\right],
\end{equation}
where
\[
\Phi_{\mathcal T}(v)
:=
\mathbb E_v
\left[
G(v,\theta(Z))
\mid \mathcal T
\right].
\]

Fix an arbitrary realization
\[
\mathcal T=t.
\]
Conditional on this realization, the lower-level training procedure has a
finite number of state-action samples and therefore a finite number of
reward and constraint encoding events.

Let $N_r(t)$ and $N_c(t)$ denote the resulting numbers of reward and
constraint encoding events, respectively, and define
\[
M(t):=N_r(t)+N_c(t).
\]
Enumerate these $M(t)$ encoding variables as
\[
B_1,\ldots,B_{M(t)}
\]
in chronological order. For each encoding event $j$, let
\[
q_j(B_j;v)
\]
denote the conditional probability of the realized encoding sign given the
training history immediately before that encoding event.

By construction, the lower-level algorithm receives only the decoded
signals
\[
\widetilde r=2RB^r,
\qquad
\widetilde c=2CB^c.
\]
Hence, once the complete sampled trace $Z$ is fixed, all lower-level critic,
policy, and dual iterates are fixed functions of that trace. In particular,
$\theta(Z)$ has no additional explicit dependence on $v$.

The initial-state distribution, transition kernel, geometric-horizon
distribution, final random-index distribution, and all other auxiliary
randomness are independent of $v$. Furthermore, conditional on a fixed
trace, the policy probabilities are evaluated at policy parameters
reconstructed from the fixed decoded signals and are therefore also
fixed with respect to $v$.

Consequently, the dependence of the conditional \rev{training distribution} on $v$
arises only through the encoding probabilities. Thus its score is
\begin{equation}
\label{eq:conditional-trace-score}
\nabla_v\log P_v(Z\mid\mathcal T=t)
=
\sum_{j=1}^{M(t)}
\nabla_v\log q_j(B_j;v)
=
S(Z;v).
\end{equation}

Because $M(t)<\infty$ for every fixed $t$, standard likelihood-ratio
differentiation gives
\begin{align}
\nabla\Phi_t(v)
&=
\nabla_v
\mathbb E_v
\left[
G(v,\theta(Z))
\mid\mathcal T=t
\right]
\\
&=
\mathbb E_v
\left[
\nabla_vG(v,\theta(Z))
+
G(v,\theta(Z))
S(Z;v)
\mid\mathcal T=t
\right].
\label{eq:conditional-gradient}
\end{align}

It remains to remove the conditioning on the geometric horizons. By
Assumption~\ref{ass:upper-objective},
\[
\|\nabla_vG(v,\theta)\|
\leq
L_{G,0},
\qquad
|G(v,\theta)|
\leq
C_G,
\]
where
\[
C_G
:=
\max\{|G_{\min}|,|G_{\max}|\}.
\]
Moreover, each encoding score satisfies
\[
\|s_{r,j}\|
\leq
\frac{R_1}{R},
\qquad
\|s_{c,j}\|
\leq
\frac{C_1}{C}.
\]
Hence, for
\[
\kappa
:=
\max
\left\{
\frac{R_1}{R},
\frac{C_1}{C}
\right\},
\]
we have
\[
\|S(Z;v)\|
\leq
\kappa M(\mathcal T).
\]
Therefore,
\[
\left\|
\nabla\Phi_{\mathcal T}(v)
\right\|
\leq
L_{G,0}
+
C_G\kappa M(\mathcal T).
\]

The total number of encoding events is bounded by
\[
M(\mathcal T)
\leq
2\sum_{i=1}^{3K(N+1)}(T_i+1).
\]
Since
\[
\mathbb E[T_i+1]
=
\frac{1}{1-\gamma},
\]
we have
\[
\mathbb E[M(\mathcal T)]
<
\infty.
\]
Thus
\[
L_{G,0}
+
C_G\kappa M(\mathcal T)
\]
is integrable with respect to the distribution of $\mathcal T$.
We may therefore differentiate
\eqref{eq:condition-horizon} under the expectation over $\mathcal T$.
Using~\eqref{eq:conditional-gradient},
\begin{align}
\nabla\Phi_{\epsilon_\ell}(v)
&=
\mathbb E_{\mathcal T}
\left[
\nabla\Phi_{\mathcal T}(v)
\right]
\\
&=
\mathbb E_v
\left[
\nabla_vG(v,\theta(Z))
+
G(v,\theta(Z))S(Z;v)
\right].
\label{eq:unconditional-score-gradient}
\end{align}

To make the last expectation explicit, set $m:=3K(N+1)$. The unconditional
\rev{training distribution} is the total-probability mixture
\[
P_v(Z)
=
\sum_{t\in\mathbb N_0^m}
\mathbb P(\mathcal T=t)P_v(Z\mid\mathcal T=t),
\qquad
\mathbb P(\mathcal T=t)
=
\prod_{i=1}^m(1-\gamma)\gamma^{t_i},
\]
where the first equality is understood as an equality of probability
measures.
Thus the second equality in
\eqref{eq:unconditional-score-gradient} is precisely the law of total
expectation applied to the conditional expression in
\eqref{eq:conditional-gradient}. Since the distribution of $\mathcal T$ is
independent of $v$, it contributes no term to the score of $P_v$.

Finally, every individual encoding score has conditional mean zero.
Therefore, by iterated expectation,
\[
\mathbb E_v[S(Z;v)]
=
0.
\]
Hence, for any deterministic scalar $c(v)$,
\[
\mathbb E_v[c(v)S(Z;v)]
=
c(v)\mathbb E_v[S(Z;v)]
=
0.
\]
Subtracting this zero term from
\eqref{eq:unconditional-score-gradient} yields
\[
\nabla\Phi_{\epsilon_\ell}(v)
=
\mathbb E_v
\left[
\nabla_vG(v,\theta(Z))
+
\bigl(
G(v,\theta(Z))-c(v)
\bigr)
S(Z;v)
\right],
\]
which proves the result.
\end{proof}

\section{Proof of Theorem \ref{thm:end-to-end}}
\label{app:proof-approx-response}
We now analyze the convergence of
Algorithm~\ref{alg:bilevel-nonpert}. The upper-level objective is induced by
the distribution of the complete stochastic lower-level training procedure.
In contrast to an iterative-differentiation approach, we do not differentiate
through the policy, critic, or projected dual updates of the lower-level
algorithm.

For a fixed upper-level parameter $v$, let $Z\sim P_v$ denote the complete random
trace of Algorithm~\ref{alg:lower-pd-npg-sgd}, including all environment
trajectories, encoded reward and constraint signals, critic samples, and
the final random lower-iterate selection. Here $P_v$ is the \rev{training distribution} induced
by the upper-level parameter. Let
\[
\theta(Z)
\]
denote the resulting policy parameter. The upper objective is
\begin{equation}
\label{eq:approx-response-objective}
\Phi_{\epsilon_\ell}(v)
:=
\mathbb E_{Z\sim P_v}
\left[
G(v,\theta(Z))
\right],
\end{equation}
where the parameters of the lower-level algorithm are chosen to achieve lower
numerical accuracy $\delta_\ell$, inducing effective response accuracy
$\epsilon_\ell=\max\{\sqrt{\epsilon_{\rm bias}},\delta_\ell\}$.

Once the complete encoded trace $Z$ is fixed, the output $\theta(Z)$ is
fixed. All dependence on $v$ through the lower-level algorithm therefore enters
through the probabilities of the encoded reward and constraint signals.

\subsection{Justification for Assumptions for Theorem \ref{thm:end-to-end}}
\label{app:end-to-end-assumptions}

\subsubsection{Discussion on Assumption~\ref{ass:outer-signal}}
\label{app:outer-signal-assumption-justification}

\begin{remark}[A normalized architecture satisfying Assumption~\ref{ass:outer-signal}]
\label{rem:normalized-signal-architecture}
The signal conditions above are stated for a general parameterized
function-approximation class. As one concrete sufficient example, let the
unconstrained parameter be $x\in\mathbb R^{d_x}$, define the bounded
effective parameter
\[
w_r=\varphi_r(x),
\qquad
\varphi_r(x)=M_r\tanh(x),
\qquad M_r>0,
\]
where the hyperbolic tangent is applied componentwise, and set
\[
r(s,a;x)=\overline r(s,a;\varphi_r(x)).
\]
Suppose that
$\overline r(s,a;w_r)=f_r(\phi(s,a);w_r)$ is a finite-depth feedforward
neural network with twice continuously differentiable activations, with
all trainable weights and biases included in $w_r$. Here $\phi(s,a)$ is
a fixed, finite-dimensional bounded feature map, continuous in $s$ for
each $a\in\mathcal A$.

Although the state space $\mathcal S=\mathbb R^{d_s}$ is not compact,
the feature-range closure
\[
\mathcal K_{\rm feat}
:=\overline{\{\phi(s,a):s\in\mathcal S,\ a\in\mathcal A\}}
\]
is compact. Moreover,
\[
\overline{\varphi_r(\mathbb R^{d_x})}
\subseteq[-M_r,M_r]^{d_x}.
\]
Joint continuity of $f_r$, $\nabla_{w_r}f_r$, and
$\nabla_{w_rw_r}^2f_r$ on the compact feature--parameter domain
$\mathcal K_{\rm feat}\times[-M_r,M_r]^{d_x}$ therefore yields uniform
bounds, valid for all states, actions, and effective parameters in this
cube:
\[
|\overline r|\leq\overline C_{r,0},
\qquad
\|\nabla_{w_r}\overline r\|\leq\overline C_{r,1},
\qquad
\|\nabla_{w_rw_r}^2\overline r\|\leq\overline C_{r,2},
\]
for finite positive constants $\overline C_{r,0},\overline C_{r,1},
\overline C_{r,2}$. In addition,
\[
\|D\varphi_r(x)\|\leq M_r,
\qquad
\|D^2\varphi_r(x)\|_{\rm op}\leq2M_r,
\]
where the second norm is the bilinear operator norm induced by Euclidean
norms. The chain rule consequently gives, uniformly in $(s,a,x)$,
\[
|r(s,a;x)|\leq\overline C_{r,0},
\qquad
\|\nabla_xr(s,a;x)\|\leq M_r\overline C_{r,1},
\]
and
\[
\|\nabla_{xx}^2r(s,a;x)\|
\leq M_r^2\overline C_{r,2}+2M_r\overline C_{r,1}.
\]
An identical construction
$w_c=\varphi_c(y)=M_c\tanh(y)$ with $M_c>0$ and
$c(s,a;y)=\overline c(s,a;\varphi_c(y))$ yields the analogous constants
$\overline C_{c,0},\overline C_{c,1},\overline C_{c,2}$.
Thus valid choices in Assumption~\ref{ass:outer-signal} are
\[
\begin{aligned}
R&=\overline C_{r,0},
& R_1&=M_r\overline C_{r,1},
& R_2&=M_r^2\overline C_{r,2}+2M_r\overline C_{r,1},\\
C&=\overline C_{c,0},
& C_1&=M_c\overline C_{c,1},
& C_2&=M_c^2\overline C_{c,2}+2M_c\overline C_{c,1}.
\end{aligned}
\]
The assumption does not restrict the analysis to neural networks, while
smooth networks with bounded effective parameters and bounded features
provide a non-vacuous class satisfying its signal bounds without
projecting the unconstrained parameters $x$ and $y$. The continuous
feature map also ensures state continuity of these signals. This
construction justifies signal regularity only; the uniform Slater
condition remains a separate requirement on the constraint family.
\end{remark}

\subsubsection{Discussion on Assumption~\ref{ass:upper-objective}}
\label{app:upper-objective-assumption-justification}

The boundedness part of Assumption~\ref{ass:upper-objective} follows from
Assumption~\ref{ass:outer-signal} when $G$ averages Bradley--Terry negative
log-likelihoods for uniformly bounded-length reward and constraint sums, as
shown next.

\begin{lemma}[Bounded Bradley--Terry objective for finite reward and constraint sums]
\label{lem:bt-bounded-objective}
Suppose $|r(s,a;x)|\leq R$ and $|c(s,a;y)|\leq C$ uniformly, and every compared
trajectory segment $\tau=((s_t,a_t))_{t=0}^{h(\tau)-1}$ has length
$h(\tau)\leq H$, where $H$ is a fixed finite positive integer independent
of $(v,\theta)$. Define its reward and constraint preference scores by
\[
u_x^r(\tau)
:=
\sum_{t=0}^{h(\tau)-1}r(s_t,a_t;x),
\qquad
u_y^c(\tau)
:=
\sum_{t=0}^{h(\tau)-1}c(s_t,a_t;y).
\]
For a compared pair $(\tau_1,\tau_2)$ and labels
$Y_r,Y_c\in\{0,1\}$, where a value of one means that the first trajectory is
preferred under the corresponding criterion, let
\[
\Delta_r:=u_x^r(\tau_1)-u_x^r(\tau_2),
\qquad
\Delta_c:=u_y^c(\tau_1)-u_y^c(\tau_2),
\]
and define
\[
\ell_r:=\log(1+e^{\Delta_r})-Y_r\Delta_r,
\qquad
\ell_c:=\log(1+e^{\Delta_c})-Y_c\Delta_c,
\qquad
\ell_{\rm BT}:=\frac12(\ell_r+\ell_c).
\]
If $G(v,\theta)$ is an empirical average or an expectation of
$\ell_{\rm BT}$ over such jointly labeled pairs, then, uniformly in
$(v,\theta)$,
\begin{equation}
\label{eq:bt-objective-bounds}
\frac12\left[
\log\!\left(1+e^{-2HR}\right)
+\log\!\left(1+e^{-2HC}\right)
\right]
\leq
G(v,\theta)
\leq
\frac12\left[
\log\!\left(1+e^{2HR}\right)
+\log\!\left(1+e^{2HC}\right)
\right].
\end{equation}
The distribution of the compared pairs may depend on $(v,\theta)$,
provided the same length, reward, and constraint bounds hold throughout its
support.
In particular, valid constants for the boundedness part of
Assumption~\ref{ass:upper-objective} are
\begin{equation}
\label{eq:bt-G-constants}
G_{\min}=0,
\qquad
G_{\max}=\frac12\left[
\log\!\left(1+e^{2HR}\right)
+\log\!\left(1+e^{2HC}\right)
\right],
\qquad
\Delta_G=G_{\max}.
\end{equation}
\end{lemma}

\begin{proof}
The triangle inequality and the uniform signal bounds give
\[
|u_x^r(\tau)|
\leq
\sum_{t=0}^{h(\tau)-1}|r(s_t,a_t;x)|
\leq
h(\tau)R
\leq HR,
\qquad
|u_y^c(\tau)|
\leq
h(\tau)C
\leq HC.
\]
Consequently,
$|\Delta_r|\leq2HR$ and $|\Delta_c|\leq2HC$. For either value of its
binary label, each logistic loss equals $\log(1+e^z)$ for some
$z\in[-|\Delta|,|\Delta|]$. Monotonicity of this function gives the two
componentwise bounds, whose average is~\eqref{eq:bt-objective-bounds}.
Empirical averaging or expectation preserves them. The lower bound is
positive, so replacing it by zero yields~\eqref{eq:bt-G-constants}.
\end{proof}

Lemma~\ref{lem:bt-bounded-objective} alone justifies only the boundedness
part of Assumption~\ref{ass:upper-objective}. The smoothness and
gradient-oracle conditions are established next for a specified
policy-generated sampling model. The common length bound $H$ is
essential to this argument: an undiscounted sum over an almost surely
finite but unbounded random horizon, such as the geometric horizons used
by the lower-level critic, need not have a uniformly bounded score.
If $G$ is an unnormalized sum of a fixed number of preference losses
rather than their average, its bounds are multiplied by that number.

\paragraph{Policy-generated preference samples.}
Fix a positive integer $H$. At any queried $(v,\theta)$, suppose we can
sample two independent trajectories $\tau_1,\tau_2$ under $\pi_\theta$,
starting from the fixed initial distribution $\rho$. Each length
$h_i\in\{1,\ldots,H\}$ is fixed in advance or sampled independently from
a fixed \rev{distribution} independent of $(v,\theta)$. The transition kernel is the
parameter-independent kernel $P$ of the problem formulation.
For each pair, suppose we can sample labels
$(Y_r,Y_c)\in\{0,1\}^2$ from a fixed Borel measurable conditional \rev{distribution}
$\mathcal Q(\cdot\mid\tau_1,\tau_2)$, where $Y_r=1$ means that the first
trajectory is reward-preferred and $Y_c=1$ means that it is preferred under
the constraint or safety criterion. This joint label \rev{distribution} has no additional
dependence on $(v,\theta)$ and need not itself follow a Bradley--Terry model.
Each oracle call uses fresh trajectories and fresh labels, conditionally
independently of all previous calls.

Using the reward- and constraint-sum scores from
Lemma~\ref{lem:bt-bounded-objective}, define
\[
\Delta_r=u_x^r(\tau_1)-u_x^r(\tau_2),
\qquad
p_r=\frac{1}{1+e^{-\Delta_r}},
\qquad
\ell_r=\log(1+e^{\Delta_r})-Y_r\Delta_r,
\]
and
\[
\Delta_c=u_y^c(\tau_1)-u_y^c(\tau_2),
\qquad
p_c=\frac{1}{1+e^{-\Delta_c}},
\qquad
\ell_c=\log(1+e^{\Delta_c})-Y_c\Delta_c.
\]
Let $\ell_{\rm BT}:=(\ell_r+\ell_c)/2$.
Let
\begin{equation}
\label{eq:bt-policy-objective}
G(v,\theta)
:=\mathbb E_{\tau_1,\tau_2\sim\pi_\theta,\,
(Y_r,Y_c)\sim\mathcal Q(\cdot\mid\tau_1,\tau_2)}
[\ell_{\rm BT}].
\end{equation}
Thus, at fixed $\theta$, the sampling \rev{distribution} is independent of $v$.
Define the pair policy score
\[
S_\theta^{\rm pair}
:=\sum_{i=1}^{2}\sum_{t=0}^{h_i-1}
\nabla_\theta\log\pi_\theta(a_t^{(i)}\mid s_t^{(i)}).
\]
This is the score of the upper trajectory pair, not the lower training
score $S(Z;v)$. Assume the reward and constraint values and their respective
$x$- and $y$-derivatives can be evaluated on the sampled trajectories. Define
the single-pair estimators
\begin{equation}
\label{eq:bt-gradient-estimators}
\widehat\nabla_vG(v,\theta)
=\frac12\begin{bmatrix}
(p_r-Y_r)\nabla_x\Delta_r\\
(p_c-Y_c)\nabla_y\Delta_c
\end{bmatrix},
\qquad
\widehat\nabla_\theta G(v,\theta)
=\ell_{\rm BT}\,S_\theta^{\rm pair}
\end{equation}

Using the policy-score bound $G_\pi$ from
Assumption~\ref{ass:score}, let
\[
L_{\rm BT}
:=\frac12\left[
\log(1+e^{2HR})+\log(1+e^{2HC})
\right].
\]

\begin{lemma}[Unbiasedness and bounded variance of Bradley--Terry gradients]
\label{lem:bt-gradient-oracles}
Under Assumptions~\ref{ass:score} and~\ref{ass:outer-signal} and the
sampling model above, the estimators in~\eqref{eq:bt-gradient-estimators}
satisfy, uniformly in $(v,\theta)$,
\begin{equation}
\label{eq:bt-gradient-unbiasedness}
\begin{aligned}
\mathbb E[\widehat\nabla_vG(v,\theta)\mid v,\theta]
&=\nabla_vG(v,\theta),\\
\mathbb E[\widehat\nabla_\theta G(v,\theta)\mid v,\theta]
&=\nabla_\theta G(v,\theta),
\end{aligned}
\end{equation}
and
\begin{equation}
\label{eq:bt-gradient-variance}
\begin{aligned}
\mathbb E\!\left[
\|\widehat\nabla_vG(v,\theta)-\nabla_vG(v,\theta)\|^2
\mid v,\theta\right]
&\leq H^2(R_1^2+C_1^2),\\
\mathbb E\!\left[
\|\widehat\nabla_\theta G(v,\theta)-\nabla_\theta G(v,\theta)\|^2
\mid v,\theta\right]
&\leq(2HG_\pi L_{\rm BT})^2.
\end{aligned}
\end{equation}
\end{lemma}

\begin{proof}
Condition on the queried $(v,\theta)$ and the history before the oracle
call. The new data have the sampling \rev{distribution} specified above. For a
fixed labeled pair, direct differentiation gives
\[
\nabla_x\ell_{\rm BT}=\frac12(p_r-Y_r)\nabla_x\Delta_r,
\qquad
\nabla_y\ell_{\rm BT}=\frac12(p_c-Y_c)\nabla_y\Delta_c.
\]
Since the pair and label \rev{distribution} is independent of $v$ at fixed $\theta$,
differentiating its expectation gives the first estimator
in~\eqref{eq:bt-gradient-estimators}. This interchange is justified by
$|p_r-Y_r|,|p_c-Y_c|\leq1$,
$\|\nabla_x\Delta_r\|\leq2HR_1$, and
$\|\nabla_y\Delta_c\|\leq2HC_1$.

Let $f_\theta(\tau_1,\tau_2,Y_r,Y_c)$ denote the joint density/mass function
of the sampled trajectories, lengths, and labels. On its support, all factors
except the action probabilities are independent of $\theta$, so
\[
\nabla_\theta\log f_\theta(\tau_1,\tau_2,Y_r,Y_c)
=S_\theta^{\rm pair}.
\]
The loss has no explicit $\theta$-dependence at a fixed labeled pair.
Consequently,
\[
\nabla_\theta G(v,\theta)
=\mathbb E[\ell_{\rm BT}S_\theta^{\rm pair}],
\]
which proves unbiasedness of the second estimator. For completeness,
full action support gives a common support, and the bounded policy score
implies the local likelihood-ratio bound
\[
\frac{f_{\theta'}(\tau_1,\tau_2,Y_r,Y_c)}
{f_\theta(\tau_1,\tau_2,Y_r,Y_c)}
\leq\exp(2HG_\pi\|\theta'-\theta\|)
\]
on that support. Together with $\ell_{\rm BT}\leq L_{\rm BT}$ and
$\|S_\theta^{\rm pair}\|\leq2HG_\pi$, this supplies local domination
for the differentiation step. The conditional label probabilities need
not be evaluated or differentiated to implement either estimator.

The preceding componentwise gradient bounds,
$\ell_{\rm BT}\leq L_{\rm BT}$, and
$\|S_\theta^{\rm pair}\|\leq2HG_\pi$ give
\begin{equation}
\label{eq:bt-gradient-sample-bounds}
\|\widehat\nabla_vG(v,\theta)\|
\leq H\sqrt{R_1^2+C_1^2},
\qquad
\|\widehat\nabla_\theta G(v,\theta)\|
\leq2HG_\pi L_{\rm BT}.
\end{equation}
Taking expectations bounds the true gradient norms. For either unbiased vector estimator $X$,
\[
\mathbb E[\|X-\mathbb E[X\mid v,\theta]\|^2\mid v,\theta]
\leq\mathbb E[\|X\|^2\mid v,\theta],
\]
which proves~\eqref{eq:bt-gradient-variance}.

Finally, at a fixed labeled pair,
\[
\nabla_{xx}^2\ell_{\rm BT}
=\frac12\left[
p_r(1-p_r)\nabla_x\Delta_r\nabla_x\Delta_r^\top
+(p_r-Y_r)\nabla_{xx}^2\Delta_r
\right].
\]
Using $p_r(1-p_r)\leq1/4$ and
$\|\nabla_{xx}^2\Delta_r\|\leq2HR_2$ gives
\[
\|\nabla_{xx}^2\ell_{\rm BT}\|
\leq\frac12\left(H^2R_1^2+2HR_2\right).
\]
Analogously,
\[
\|\nabla_{yy}^2\ell_{\rm BT}\|
\leq\frac12\left(H^2C_1^2+2HC_2\right),
\qquad
\nabla_{xy}^2\ell_{\rm BT}=0.
\]
The sampling \rev{distribution} is independent of $v$ at fixed $\theta$. Averaging these
uniform block-Hessian bounds establishes the claimed Lipschitz constant for
$\nabla_vG$.
\end{proof}

The proof also shows that the following are valid constants in
Assumption~\ref{ass:upper-objective}:
\begin{equation}
\label{eq:bt-oracle-constants}
\begin{aligned}
L_{G,0}&=H\sqrt{R_1^2+C_1^2},
& M_G&=2HG_\pi L_{\rm BT},\\
\sigma_v^2&=H^2(R_1^2+C_1^2),
& \sigma_\theta^2&=(2HG_\pi L_{\rm BT})^2,\\
L_G&=\frac12\max\left\{
H^2R_1^2+2HR_2,
H^2C_1^2+2HC_2
\right\}.
\end{aligned}
\end{equation}
Together with the value bounds in
Lemma~\ref{lem:bt-bounded-objective}, this verifies
Assumption~\ref{ass:upper-objective} for~\eqref{eq:bt-policy-objective}.

Lemma~\ref{lem:bt-gradient-oracles} applies at every queried policy
parameter, including the interpolated parameters in
\eqref{eq:fundamental-value-difference}. Each oracle call in this implementation
uses two fresh rollouts of length at most $H$ and obtains reward-preference
and constraint-preference labels; separate calls use independent data as
required by the algorithm.
The lemma concerns the policy-generated expectation
\eqref{eq:bt-policy-objective}. For a fixed, parameter-independent dataset
with reward and constraint scores depending on $(x,y)$, minibatch
$v$-gradients instead estimate the empirical objective, whose
$\theta$-partial gradient is zero.
If pair selection or the conditional label \rev{distribution} has additional parameter
dependence, its derivative contributions must also be included; the
estimators above do not automatically cover that case.

Define
\begin{equation}
\label{eq:Vtheta-def}
V_\theta
:=
M_G^2+\sigma_\theta^2.
\end{equation}

\subsubsection{\texorpdfstring{\rev{Application Examples}}{Application Examples}}
\label{app:examples}
\begingroup\color{black}
We state the three examples from Section~\ref{sec:introduction} in the
notation of Section~\ref{sec:problem-formulation}. In each case, the lower level
is the CMDP $\max_\theta\{J_r(\theta,x):J_c(\theta,y)\geq0\}$, and the upper
objective is evaluated at the returned policy parameter through
$\Phi_{\epsilon_\ell}(v)=\mathbb E_{Z\sim P_v}[G(v,\theta(Z))]$. The examples
illustrate the formulation; the first is checked against
Assumption~\ref{ass:upper-objective} in detail in
Appendix~\ref{app:upper-objective-assumption-justification}.

\paragraph{Safe alignment from human feedback.}
The upper-level variable $x$ parameterizes a reward model $r(s,a;x)$, and $y$
parameterizes a constraint signal $c(s,a;y)$, for example a budget offset minus
a learned safety cost, so that $J_c(\theta,y)\geq0$ encodes a cost budget. The
lower level maximizes the learned reward subject to the learned safety
constraint. The upper objective is the Bradley--Terry loss of reward and safety
preference labels on trajectory pairs sampled from $\pi_\theta$,
\[
G(v,\theta)=\mathbb E_{\tau_1,\tau_2\sim\pi_\theta,\,
(Y_r,Y_c)\sim\mathcal Q(\cdot\mid\tau_1,\tau_2)}[\ell_{\rm BT}],
\]
as in~\eqref{eq:bt-policy-objective}. Because the preference pairs are
generated by the returned policy, the reward and safety models are fitted to
the behavior they induce, as in policy-aligned reward learning
\citep{chakraborty2024parl}, here with the constrained lower level of Safe RLHF
\citep{dai2024safe}.

\paragraph{Inverse constrained RL.}
The reward $r(s,a;x)$ is known or held fixed, and $y$ parameterizes an unknown
constraint signal $c(s,a;y)$, for example a penalty on an unsafe region of the
state space. Given expert demonstrations $\mathcal D$, the upper objective
measures how well the returned policy explains them, for example through the
negative log-likelihood
\[
G(v,\theta)=-\frac{1}{|\mathcal D|}\sum_{(s,a)\in\mathcal D}
\log\pi_\theta(a\mid s).
\]
Minimizing the upper objective over $y$ selects a constraint under which the
lower-level behavior matches the expert, as in
\citet{scobee2020maximum} and \citet{malik2021inverse}.

\paragraph{Constrained reward design.}
The constraint parameter $y$ is held fixed, so that $c(s,a;y)$ encodes a fixed
safety or resource requirement, while $x$ parameterizes the reward
$r(s,a;x)=r_0(s,a)+\Delta r(s,a;x)$ optimized by the lower-level agent, where
$r_0$ is a task reward. The upper objective is the designer-specified task
performance of the returned policy, possibly with a regularizer on the reward
parameters,
\[
G(v,\theta)=-J_{r_0}(\theta)+\kappa\|x\|^2,
\qquad
J_{r_0}(\theta):=\mathbb E_{\pi_\theta}\Big[\sum_{t=0}^\infty\gamma^tr_0(s_t,a_t)\Big].
\]
Because the upper objective is evaluated at the policy returned by a finite
lower-level run, the choice of $x$ can improve it even when it does not change
the exact constrained optimum.
\endgroup

\subsection{Supporting Lemmas for Theorem \ref{thm:end-to-end}}
\textbf{Score Bounds for the Random-Length Lower Trace.}
For a reward encoding, define the score already embedded in the
$x$-coordinate block of $v=(x,y)$ by
\[
s_r
=
\begin{bmatrix}
\dfrac{B^r\nabla_xr(s,a;x)}{2R+B^rr(s,a;x)}\\[1mm]
0_{d_y}
\end{bmatrix}
\in\mathbb R^{d_v},
\]
and for a constraint encoding define the score already embedded in the
$y$-coordinate block by
\[
s_c
=
\begin{bmatrix}
0_{d_x}\\[1mm]
\dfrac{B^c\nabla_yc(s,a;y)}{2C+B^cc(s,a;y)}
\end{bmatrix}
\in\mathbb R^{d_v}.
\]
The corresponding log-probability Hessians, embedded in the same coordinate
blocks, are denoted by $H_r,H_c\in\mathbb R^{d_v\times d_v}$.

Let
\[
N_r(Z),
\qquad
N_c(Z)
\]
be the total numbers of reward and constraint encoding events in one complete
lower-level training run. Define
\begin{equation}
\label{eq:expected-counts}
\bar n_r
:=
\sup_v
\mathbb E_v[N_r(Z)],
\qquad
\bar n_c
:=
\sup_v
\mathbb E_v[N_c(Z)].
\end{equation}

Define
\begin{equation}
\label{eq:Ibar}
\bar I
:=
\frac{\bar n_r}{3}
\left(
\frac{R_1}{R}
\right)^2
+
\frac{\bar n_c}{3}
\left(
\frac{C_1}{C}
\right)^2,
\end{equation}
and
\begin{equation}
\label{eq:Jbar}
\bar J
:=
\bar n_r
\left[
\frac{R_2}{R}
+
\left(
\frac{R_1}{R}
\right)^2
\right]
+
\bar n_c
\left[
\frac{C_2}{C}
+
\left(
\frac{C_1}{C}
\right)^2
\right].
\end{equation}

\begin{lemma}[Moments of the complete training score]
\label{lem:score-moments}
Under Assumption~\ref{ass:outer-signal}, define
\[
S(Z;v)
:=
\sum_{j=1}^{N_r(Z)}s_{r,j}
+
\sum_{j=1}^{N_c(Z)}s_{c,j},
\qquad
H_{\log}(Z;v)
:=
\sum_{j=1}^{N_r(Z)}H_{r,j}
+
\sum_{j=1}^{N_c(Z)}H_{c,j}.
\]
Then
\begin{equation}
\label{eq:score-mean-zero}
\mathbb E_v[S(Z;v)]
=
0,
\end{equation}
and
\begin{equation}
\label{eq:score-second-moment-app}
\mathbb E_v
\left[
\|S(Z;v)\|^2
\right]
\leq
\bar I.
\end{equation}
Moreover,
\begin{equation}
\label{eq:hlog-bound}
\mathbb E_v
\left[
\|H_{\log}(Z;v)\|
\right]
\leq
\bar J.
\end{equation}
\end{lemma}

\begin{proof}
Fix $v$. Enumerate all reward and constraint encoding events in one lower
training run in chronological order, with reward encodings preceding
constraint encodings when both are queried at the same step. Let
$N_{\rm enc}:=N_r(Z)+N_c(Z)$ be the total number of such events. For
$j\geq1$, define the vector $D_j\in\mathbb R^{d_v}$ by
\[
D_j:=
\begin{cases}
s_r^{(j)},
&\text{if event $j$ is a reward encoding},\\[3mm]
s_c^{(j)},
&\text{if event $j$ is a constraint encoding},\\[3mm]
0_{d_v},&\text{if $j>N_{\rm enc}$}.
\end{cases}
\]
Here $s_r^{(j)}$ and $s_c^{(j)}$ are the already embedded scores in
\eqref{eq:reward-score} and~\eqref{eq:cost-score}, evaluated at the
state--action pair and encoded sign of event $j$. Thus $j$ indexes the
combined sequence of encoding events, rather than policy or critic
iterations. With the coordinate embedding above and zero padding after
termination, the complete score is
\[
S(Z;v)=\sum_{j=1}^{N_{\rm enc}}D_j=\sum_{j\geq1}D_j.
\]

Let $\mathcal F_{j-1}$ contain the training history immediately before
encoding event $j$, including whether the run has already terminated and,
if it has not, the event type and queried state--action pair. After
termination, retain the terminal history. These sigma-algebras form an
increasing filtration, and $D_i$ is $\mathcal F_{j-1}$-measurable whenever
$i<j$. The occurrence and type of event $j$ are known before drawing its
encoded sign; no conditioning on future encoding outcomes is involved.

\emph{Conditional mean of each increment.}
For an active event $j$, write $a(v)$ for its scalar signal and $A>0$ for
its fixed encoding bound, with $|a(v)|\leq A$. Conditional on
$\mathcal F_{j-1}$, the queried pair is fixed and
\[
\mathbb P_v(B=b\mid\mathcal F_{j-1})
=\frac{2A+b\,a(v)}{4A},
\qquad
s_a(b;v)=\frac{b\nabla_v a(v)}{2A+b\,a(v)},
\qquad b\in\{-1,+1\}.
\]
The derivative holds the queried pair and preceding trace fixed, and the
bound $A$ is independent of $v$. Since $2A+b\,a(v)\geq A>0$, direct
summation over the two possible signs gives
\[
\begin{aligned}
\mathbb E_v[s_a(B;v)\mid\mathcal F_{j-1}]
&=\sum_{b\in\{-1,+1\}}
\frac{2A+b\,a(v)}{4A}
\frac{b\nabla_v a(v)}{2A+b\,a(v)}\\
&=\frac{\nabla_v a(v)}{4A}\sum_{b\in\{-1,+1\}}b
=0.
\end{aligned}
\]
The coordinate-embedded reward and constraint score definitions above
preserve this identity. On terminated runs $D_j=0$, so in all cases
\begin{equation}
\label{eq:encoding-increment-conditional-mean}
\mathbb E_v[D_j\mid\mathcal F_{j-1}]=0,
\qquad
\mathbb E_v[D_j]
=\mathbb E_v\!\left[\mathbb E_v[D_j\mid\mathcal F_{j-1}]\right]=0.
\end{equation}

\emph{Mean of the complete random-length score.}
There are $K(N+1)$ oracle calls, each using finitely many geometric-horizon
trajectories with finite expected lengths. Thus
$\mathbb E_v[N_r]+\mathbb E_v[N_c]<\infty$; the precise counts are given
in Lemma~\ref{lem:encoding-counts}. Assumption~\ref{ass:outer-signal} and
the positive denominator bounds imply
\[
\|s_r^{(j)}\|\leq\frac{R_1}{R},
\qquad
\|s_c^{(j)}\|\leq\frac{C_1}{C}.
\]
Consequently, Tonelli's theorem gives
\[
\sum_{j\geq1}\mathbb E_v\|D_j\|
=\mathbb E_v\!\left[\sum_{j\geq1}\|D_j\|\right]
\leq\frac{R_1}{R}\mathbb E_v[N_r]
+\frac{C_1}{C}\mathbb E_v[N_c]
<\infty.
\]
This absolute integrability justifies interchanging expectation and the
infinite, zero-padded sum. Using
\eqref{eq:encoding-increment-conditional-mean}, we obtain
\[
\mathbb E_v[S(Z;v)]
=\sum_{j\geq1}\mathbb E_v[D_j]
=\sum_{j\geq1}
\mathbb E_v\!\left[\mathbb E_v[D_j\mid\mathcal F_{j-1}]\right]
=0,
\]
which proves~\eqref{eq:score-mean-zero}.

\emph{Second moment of the complete score.}
For an active scalar encoding with $\|\nabla_v a(v)\|\leq A_1$, the
conditional second moment is
\[
\begin{aligned}
\mathbb E_v[\|s_a(B;v)\|^2\mid\mathcal F_{j-1}]
&=\sum_{b\in\{-1,+1\}}
\frac{2A+b\,a(v)}{4A}
\frac{\|\nabla_v a(v)\|^2}{(2A+b\,a(v))^2}\\
&=\frac{\|\nabla_v a(v)\|^2}{4A}
\left(\frac{1}{2A+a(v)}+\frac{1}{2A-a(v)}\right)\\
&=\frac{\|\nabla_v a(v)\|^2}{4A^2-a(v)^2}
\leq\frac{A_1^2}{3A^2}.
\end{aligned}
\]
Let $I_j^r$ and $I_j^c$ be the indicators that event $j$ is a reward or
constraint encoding, respectively, with both zero after termination.
They are $\mathcal F_{j-1}$-measurable. Applying the preceding bound to
each event type gives
\[
\mathbb E_v[\|D_j\|^2\mid\mathcal F_{j-1}]
\leq\frac{I_j^r}{3}\left(\frac{R_1}{R}\right)^2
+\frac{I_j^c}{3}\left(\frac{C_1}{C}\right)^2.
\]
Since $\sum_j I_j^r=N_r$ and $\sum_j I_j^c=N_c$, taking expectations and
summing yields
\begin{equation}
\label{eq:encoding-increment-second-moments}
\sum_{j\geq1}\mathbb E_v\|D_j\|^2
\leq\frac{\mathbb E_v[N_r]}{3}\left(\frac{R_1}{R}\right)^2
+\frac{\mathbb E_v[N_c]}{3}\left(\frac{C_1}{C}\right)^2
\leq\bar I<\infty.
\end{equation}

For $i<j$, the earlier increment $D_i$ is known given
$\mathcal F_{j-1}$. The tower property and
\eqref{eq:encoding-increment-conditional-mean} therefore imply
\[
\mathbb E_v[D_i^\top D_j]
=\mathbb E_v\!\left[
D_i^\top\mathbb E_v[D_j\mid\mathcal F_{j-1}]
\right]=0.
\]
Thus the increments are orthogonal in expectation; independence is not
required. For the finite partial sum $S_m:=\sum_{j=1}^mD_j$, expansion
of the squared norm gives
\[
\mathbb E_v\|S_m\|^2
=\sum_{j=1}^m\mathbb E_v\|D_j\|^2
+2\sum_{1\leq i<j\leq m}\mathbb E_v[D_i^\top D_j]
=\sum_{j=1}^m\mathbb E_v\|D_j\|^2.
\]
To pass to the full random-length score, note that for $n>m$,
\[
\mathbb E_v\|S_n-S_m\|^2
=\sum_{j=m+1}^n\mathbb E_v\|D_j\|^2
\leq\sum_{j>m}\mathbb E_v\|D_j\|^2\longrightarrow0
\quad\text{as }m\to\infty,
\]
by~\eqref{eq:encoding-increment-second-moments}. Hence $S_m$ converges
in $L^2$. Since $N_{\rm enc}<\infty$ almost surely, these partial sums
also eventually equal $S(Z;v)$ on each realized finite run, identifying
the $L^2$ limit with $S(Z;v)$. Therefore,
\[
\mathbb E_v\|S(Z;v)\|^2
=\lim_{m\to\infty}\mathbb E_v\|S_m\|^2
=\sum_{j\geq1}\mathbb E_v\|D_j\|^2
\leq\bar I,
\]
which proves~\eqref{eq:score-second-moment-app}.

\emph{Log-Hessian bound.}
Consider one active encoding event with scalar signal $a(v)$, encoding
bound $A$, and realized sign $B\in\{-1,+1\}$. Its conditional probability
given the preceding history is
\[
p_v(B)
=
\frac{2A+Ba(v)}{4A}.
\]
Holding the preceding history, event type, and queried state--action pair
fixed, its log-probability Hessian is
\[
H_a(B;v)
=
\nabla_v^2\log p_v(B)
=
\frac{B\nabla^2a(v)}
{2A+Ba(v)}
-
\frac{
\nabla a(v)\nabla a(v)^\top
}{
(2A+Ba(v))^2
}.
\]
The second term follows because differentiating the denominator contributes
$B\nabla a(v)$ and $B^2=1$. Under
$|a(v)|\leq A$, $\|\nabla_v a(v)\|\leq A_1$, and
$\|\nabla_v^2a(v)\|\leq A_2$, we have
\[
2A+Ba(v)
\geq A,
\]
and
\[
\|\nabla a(v)\nabla a(v)^\top\|
=
\|\nabla a(v)\|^2.
\]
The triangle inequality therefore gives the deterministic eventwise bound
\[
\|H_a(B;v)\|
\leq
\frac{A_2}{A}
+
\left(
\frac{A_1}{A}
\right)^2.
\]
For a reward encoding, substitute
$(A,A_1,A_2)=(R,R_1,R_2)$; for a constraint encoding, substitute
$(A,A_1,A_2)=(C,C_1,C_2)$. Embedding these Hessians in the corresponding
$x$- or $y$-coordinate block does not change their operator norms. Hence,
on every realized lower trace,
\[
\begin{aligned}
\|H_{\log}(Z;v)\|
&\leq
\sum_{j=1}^{N_r(Z)}\|H_{r,j}\|
+
\sum_{j=1}^{N_c(Z)}\|H_{c,j}\|\\
&\leq
N_r(Z)
\left[
\frac{R_2}{R}
+
\left(\frac{R_1}{R}\right)^2
\right]
+
N_c(Z)
\left[
\frac{C_2}{C}
+
\left(\frac{C_1}{C}\right)^2
\right].
\end{aligned}
\]
Taking expectations, using
$\mathbb E_v[N_r(Z)]\leq\bar n_r$ and
$\mathbb E_v[N_c(Z)]\leq\bar n_c$, and recalling the definition of
$\bar J$ gives
\[
\begin{aligned}
\mathbb E_v\|H_{\log}(Z;v)\|
&\leq
\bar n_r
\left[
\frac{R_2}{R}
+
\left(\frac{R_1}{R}\right)^2
\right]\\
&\quad+
\bar n_c
\left[
\frac{C_2}{C}
+
\left(\frac{C_1}{C}\right)^2
\right]
=\bar J,
\end{aligned}
\]
which proves~\eqref{eq:hlog-bound}.
\end{proof}

\begin{lemma}[Encoding counts for the geometric-horizon lower-level algorithm]
\label{lem:encoding-counts}
Suppose the geometric random variables in
Algorithm~\ref{alg:encoded-critic} have distribution
\[
\mathbb P(T=t)
=
(1-\gamma)\gamma^t,
\qquad
t=0,1,\ldots.
\]
One call to Algorithm~\ref{alg:encoded-critic} uses, in expectation,
\[
\frac{2}{1-\gamma}
\]
reward encodings and
\[
\frac{3}{1-\gamma}
\]
constraint encodings.

Since Algorithm~\ref{alg:lower-pd-npg-sgd} makes $N+1$ such calls at each
of its $K$ primal-dual iterations,
\begin{equation}
\label{eq:nr-explicit}
\bar n_r
=
\frac{2K(N+1)}{1-\gamma},
\end{equation}
and
\begin{equation}
\label{eq:nc-explicit}
\bar n_c
=
\frac{3K(N+1)}{1-\gamma}.
\end{equation}
In particular,
\begin{equation}
\label{eq:I-J-order}
\bar I
=
O
\left(
\frac{K(N+1)}{1-\gamma}
\right),
\qquad
\bar J
=
O
\left(
\frac{K(N+1)}{1-\gamma}
\right),
\end{equation}
where the hidden constants depend only on the signal derivative bounds.
\end{lemma}

\begin{proof}
For $T\sim\operatorname{Geo}(1-\gamma)$,
\[
\mathbb E[T+1]
=
\frac{1}{1-\gamma}.
\]
The constraint-return trajectory generates one constraint encoding at every
sampled state-action pair, hence contributes
$(1-\gamma)^{-1}$ constraint encodings in expectation.

The $Q$-trajectory generates both one reward and one constraint encoding at
each sampled state-action pair, and the $V$-trajectory does the same.
Therefore a complete oracle call produces
\[
\frac{2}{1-\gamma}
\]
reward encodings and
\[
\frac{3}{1-\gamma}
\]
constraint encodings in expectation.

There are $N$ critic calls and one independent constraint/critic call per
primal-dual iteration. Multiplying by $K(N+1)$ proves
\eqref{eq:nr-explicit}--\eqref{eq:nc-explicit}.
Substitution into~\eqref{eq:Ibar}--\eqref{eq:Jbar} gives
\eqref{eq:I-J-order}.
\end{proof}

\subsection{Upper-Gradient Estimation Using Only Gradient Oracles}

We next construct an unbiased estimator of
\eqref{eq:training-gradient} without assuming access to exact values of
$G$.

We choose the reference policy parameter to be the lower initialization
$\theta_0$. Let
\[
d(Z)
=
\theta(Z)-\theta_0.
\]

\begin{lemma}[Mean-square displacement of the approximate lower-level response]
\label{lem:response-displacement}
Suppose Assumptions~\ref{ass:score} and~\ref{ass:fisher-lower} hold and
let $\beta$ be the inner vanilla-SGD critic stepsize used in
Algorithm~\ref{alg:lower-pd-npg-sgd}, chosen as required by
Lemma~\ref{lem:sgd-critic}. Let $Z\sim P_v$ be the complete lower trace and
let $\theta_R=\theta(Z)$ be the lower-level response returned by the same
algorithm. Then
\begin{equation}
\label{eq:D2bar}
\mathbb E
\left[
\|\theta_R-\theta_0\|^2
\right]
\leq
\bar D^2,
\end{equation}
where
\begin{equation}
\label{eq:Dbar-definition}
\bar D^2
:=
\eta^2K^2W_\omega^2.
\end{equation}
Here $W_\omega^2$ is the constant defined in
Lemma~\ref{lem:sgd-critic}.

In particular, if
\begin{equation}
\label{eq:eta-displacement-choice}
\eta
=
\frac{\eta_0}{\sqrt K},
\end{equation}
then
\begin{equation}
\label{eq:Dbar-sqrtK}
\bar D^2
\leq
\eta_0^2
K
W_\omega^2.
\end{equation}
\end{lemma}

\begin{proof}
By Lemma~\ref{lem:sgd-critic},
\[
\mathbb E_k\|\widehat\omega_k\|^2
\leq
W_\omega^2.
\]
Taking expectation over the lower-level history gives
\begin{equation}
\label{eq:unconditional-Womega-bound}
\mathbb E\|\widehat\omega_k\|^2
\leq
W_\omega^2
\end{equation}
uniformly in $k$.

The policy update in Algorithm~\ref{alg:lower-pd-npg-sgd} gives, for any
$r\in\{0,\ldots,K\}$,
\begin{equation}
\label{eq:theta-telescope-displacement}
\theta_r-\theta_0
=
\eta
\sum_{k=0}^{r-1}
\widehat\omega_k.
\end{equation}
Therefore,
\begin{equation}
\label{eq:theta-square-before-CS}
\|\theta_r-\theta_0\|^2
=
\eta^2
\left\|
\sum_{k=0}^{r-1}
\widehat\omega_k
\right\|^2.
\end{equation}

By the Cauchy--Schwarz inequality,
\begin{equation}
\label{eq:sum-CS-displacement}
\left\|
\sum_{k=0}^{r-1}
\widehat\omega_k
\right\|^2
\leq
r
\sum_{k=0}^{r-1}
\|\widehat\omega_k\|^2.
\end{equation}
Substituting~\eqref{eq:sum-CS-displacement} into
\eqref{eq:theta-square-before-CS} gives
\begin{equation}
\label{eq:theta-r-pre-expectation}
\|\theta_r-\theta_0\|^2
\leq
\eta^2
r
\sum_{k=0}^{r-1}
\|\widehat\omega_k\|^2.
\end{equation}

Taking expectation of~\eqref{eq:theta-r-pre-expectation},
\begin{align}
\mathbb E
\left[
\|\theta_r-\theta_0\|^2
\right]
&\leq
\eta^2
r
\sum_{k=0}^{r-1}
\mathbb E
\left[
\|\widehat\omega_k\|^2
\right].
\label{eq:theta-r-expectation-step}
\end{align}
Applying~\eqref{eq:unconditional-Womega-bound},
\begin{align}
\mathbb E
\left[
\|\theta_r-\theta_0\|^2
\right]
&\leq
\eta^2
r
\sum_{k=0}^{r-1}
W_\omega^2
\nonumber\\
&=
\eta^2
r^2
W_\omega^2.
\label{eq:theta-r-rsquare}
\end{align}
Since $r\leq K$,
\begin{equation}
\label{eq:theta-r-K-bound}
\mathbb E
\left[
\|\theta_r-\theta_0\|^2
\right]
\leq
\eta^2
K^2
W_\omega^2.
\end{equation}

We finally average over the sampled output index
\[
R
\sim
\operatorname{Unif}\{0,\ldots,K-1\}.
\]
Because $R$ is independent of the lower-level optimization randomness,
\begin{align}
\mathbb E
\left[
\|\theta_R-\theta_0\|^2
\right]
&=
\frac{1}{K}
\sum_{r=0}^{K-1}
\mathbb E
\left[
\|\theta_r-\theta_0\|^2
\right]
\nonumber\\
&\leq
\frac{1}{K}
\sum_{r=0}^{K-1}
\eta^2
K^2
W_\omega^2
\nonumber\\
&=
\eta^2
K^2
W_\omega^2
\nonumber\\
&=
\bar D^2.
\label{eq:approx-response-displacement-final}
\end{align}
This proves~\eqref{eq:D2bar}.

Finally, if
\[
\eta
=
\frac{\eta_0}{\sqrt K},
\]
then
\begin{align}
\bar D^2
&=
\eta^2
K^2
W_\omega^2
\nonumber\\
&=
\frac{\eta_0^2}{K}
K^2
W_\omega^2
\nonumber\\
&=
\eta_0^2
K
W_\omega^2,
\end{align}
which proves~\eqref{eq:Dbar-sqrtK}.
\end{proof}

\begin{lemma}[Unbiased value-difference estimator]
\label{lem:value-difference}
Under Assumption~\ref{ass:upper-objective}, for a complete lower trace $Z$
with output $\theta(Z)$, define
\[
d:=\theta(Z)-\theta_0,
\qquad
V_\theta:=M_G^2+\sigma_\theta^2,
\]
and, conditionally on $Z$, set
\begin{equation}
\label{eq:adaptive-BG}
B_G(Z):=\max\{1,\lceil\|d\|^2\rceil\}.
\end{equation}
Draw conditionally independent $U_m\sim\operatorname{Unif}[0,1]$ and fresh
upper-gradient oracle samples, and define
\begin{equation}
\label{eq:Y-upper}
Y_m:=\left\langle
\widehat\nabla_\theta G(v,\theta_0+U_md),d
\right\rangle,
\qquad m=1,\ldots,B_G(Z),
\end{equation}
\begin{equation}
\label{eq:A-upper}
\widehat{\delta}_G(Z;v)
:=\frac{1}{B_G(Z)}\sum_{m=1}^{B_G(Z)}Y_m.
\end{equation}
Then, conditionally on $Z$,
\begin{equation}
\label{eq:A-unbiased}
\mathbb E[\widehat{\delta}_G(Z;v)\mid Z]
=G(v,\theta(Z))-G(v,\theta_0),
\end{equation}
\begin{equation}
\label{eq:delta-variance}
\operatorname{Var}(\widehat{\delta}_G(Z;v)\mid Z)\leq V_\theta.
\end{equation}
\end{lemma}

\begin{proof}
Fix $v$ and let $Z\sim P_v$ be a complete lower training trace.
Conditionally on $Z$, both $d$ and $B_G(Z)$ are fixed. Conditional
unbiasedness of each fresh oracle sample and averaging over $U_m$ give
\[
\begin{aligned}
\mathbb E[Y_m\mid Z]
&=\int_0^1\left\langle
\nabla_\theta G(v,\theta_0+ud),d\right\rangle\,du\\
&=G(v,\theta_0+d)-G(v,\theta_0)
=G(v,\theta(Z))-G(v,\theta_0),
\end{aligned}
\]
where the second equality is the fundamental theorem of calculus along
the segment. Averaging the $B_G(Z)$ terms proves~\eqref{eq:A-unbiased}.

Conditioning additionally on $U_m$ fixes the oracle's evaluation point.
The unbiasedness and variance bound in
Assumption~\ref{ass:upper-objective} imply
\[
\begin{aligned}
&\mathbb E\left[
\|\widehat\nabla_\theta G(v,\theta_0+U_md)\|^2\mid Z,U_m
\right]\\
&\quad=\|\nabla_\theta G(v,\theta_0+U_md)\|^2
+\mathbb E\left[
\|\widehat\nabla_\theta G(v,\theta_0+U_md)
-\nabla_\theta G(v,\theta_0+U_md)\|^2\mid Z,U_m
\right]\\
&\quad\leq M_G^2+\sigma_\theta^2=V_\theta.
\end{aligned}
\]
Cauchy--Schwarz and the tower property therefore yield
\[
\begin{aligned}
\mathbb E[Y_m^2\mid Z]
&\leq\|d\|^2\mathbb E\left[
\mathbb E\left[
\|\widehat\nabla_\theta G(v,\theta_0+U_md)\|^2\mid Z,U_m
\right]\mid Z\right]\\
&\leq V_\theta\|d\|^2.
\end{aligned}
\]
Since the $Y_m$ are independent conditionally on $Z$,
\[
\operatorname{Var}(\widehat{\delta}_G\mid Z)
=\frac{1}{B_G(Z)^2}\sum_{m=1}^{B_G(Z)}
\operatorname{Var}(Y_m\mid Z)
\leq\frac{V_\theta\|d\|^2}{B_G(Z)}\leq V_\theta,
\]
proving~\eqref{eq:delta-variance}. The bound
$|G(v,\theta(Z))-G(v,\theta_0)|\leq\Delta_G$ then gives
\begin{equation}
\label{eq:A-second}
\mathbb E[\widehat{\delta}_G^2\mid Z]
=\operatorname{Var}(\widehat{\delta}_G\mid Z)
+\bigl(\mathbb E[\widehat{\delta}_G\mid Z]\bigr)^2
\leq V_\theta+\Delta_G^2.
\end{equation}
\end{proof}

\begin{lemma}[Smoothness of the approximate-response objective]
\label{lem:upper-smoothness}
Suppose Assumptions~\ref{ass:outer-signal} and
\ref{ass:upper-objective} hold. Let
\[
\Phi_{\epsilon_\ell}(v)
=
\mathbb E_{Z\sim P_v}
\left[
G(v,\theta(Z))
\right]
\]
be the approximate-response objective defined in
\eqref{eq:approx-response-objective}.

Let
\[
S(Z;v)
=
\sum_{j=1}^{N_r(Z)} s_{r,j}(Z;v)
+
\sum_{j=1}^{N_c(Z)} s_{c,j}(Z;v)
\]
denote the complete training score, and define
\[
H_{\log}(Z;v)
:=
\nabla_v^2
\log P_v(Z\mid\mathcal T),
\]
where $\mathcal T$ denotes the collection of geometric horizons generated
during the complete lower-level training run.

By Lemma~\ref{lem:score-moments}, the score moments satisfy
\[
\mathbb E_v
\left[
\|S(Z;v)\|^2
\right]
\leq
\bar I,
\qquad
\mathbb E_v
\left[
\|H_{\log}(Z;v)\|
\right]
\leq
\bar J,
\]
where $\bar I$ and $\bar J$ are defined in
\eqref{eq:Ibar} and~\eqref{eq:Jbar}, respectively.

Then $\Phi_{\epsilon_\ell}$ has globally Lipschitz gradient. In particular,
for every $v,v'\in\mathbb R^{d_v}$,
\begin{equation}
\label{eq:Phi-gradient-Lipschitz}
\left\|
\nabla\Phi_{\epsilon_\ell}(v)
-
\nabla\Phi_{\epsilon_\ell}(v')
\right\|
\leq
L_\Phi
\|v-v'\|,
\end{equation}
where a valid Lipschitz constant is
\begin{equation}
\label{eq:L-Phi}
L_\Phi
=
L_G
+
2L_{G,0}\sqrt{\bar I}
+
\frac{\Delta_G}{2}
\left(
\bar I+\bar J
\right).
\end{equation}

Moreover, for Algorithm~\ref{alg:lower-pd-npg-sgd},
\begin{equation}
\label{eq:Lphi-order}
L_\Phi
=
O
\left(
1+
\frac{K(N+1)}{1-\gamma}
\right).
\end{equation}
\end{lemma}

\begin{proof}
We first prove the result under the additional assumption that
$G(\cdot,\theta)$ is twice continuously differentiable for every $\theta$.
We subsequently explain how to extend the result to the $C^{1,1}$ regularity in
Assumption~\ref{ass:upper-objective}.

Recall that the complete training score is
\begin{equation}
\label{eq:recall-complete-score}
S(Z;v)
=
\sum_{j=1}^{N_r(Z)} s_{r,j}(Z;v)
+
\sum_{j=1}^{N_c(Z)} s_{c,j}(Z;v),
\end{equation}
where $s_{r,j}$ and $s_{c,j}$ are, respectively, the score increments
corresponding to the $j$-th encoded reward and constraint signals.

We also define the complete log-likelihood Hessian by
\begin{equation}
\label{eq:complete-log-hessian}
H_{\log}(Z;v)
:=
\nabla_v^2
\log P_v(Z\mid \mathcal T),
\end{equation}
where
\[
\mathcal T
=
(T_1,\ldots,T_{3K(N+1)})
\]
collects all geometric horizons appearing in one complete lower-level
training run.

The geometric horizons are independent of the upper-level parameter $v$.
Conditional on a realization $\mathcal T=t$, the lower training trace
contains finitely many encoding events. Therefore, likelihood
differentiation can first be performed conditional on $\mathcal T=t$.

\paragraph{Step 1: Second derivative of the conditional \rev{training distribution}.}

Fix an arbitrary realization $\mathcal T=t$. Let
\[
p_v(z\mid t)
\]
denote the conditional probability density or probability mass function of
the complete lower-level trace $Z=z$ given $\mathcal T=t$, with respect to
a common reference measure.

By construction of the encoded lower-level algorithm, once the complete trace $z$ is
fixed, the decoded reward and constraint signals are fixed.
Consequently, all policy, critic, and dual iterates, as well as the final
output $\theta(z)$, are fixed functions of $z$ and have no additional
explicit dependence on $v$.

The score of the conditional \rev{training distribution} is
\begin{equation}
\label{eq:conditional-trace-score-smooth}
S(z;v)
=
\nabla_v
\log p_v(z\mid t).
\end{equation}
Therefore,
\begin{equation}
\label{eq:first-density-derivative}
\nabla_v p_v(z\mid t)
=
p_v(z\mid t)S(z;v).
\end{equation}

Differentiating~\eqref{eq:first-density-derivative} once more gives
\begin{align}
\nabla_v^2 p_v(z\mid t)
&=
\nabla_v
\left(
p_v(z\mid t)S(z;v)
\right)
\nonumber\\
&=
\nabla_v p_v(z\mid t)S(z;v)^\top
+
p_v(z\mid t)
\nabla_v S(z;v)
\nonumber\\
&=
p_v(z\mid t)
S(z;v)S(z;v)^\top
+
p_v(z\mid t)
H_{\log}(z;v),
\end{align}
where we used
\[
\nabla_v S(z;v)
=
\nabla_v^2
\log p_v(z\mid t)
=
H_{\log}(z;v).
\]
Hence,
\begin{equation}
\label{eq:second-density-derivative-full}
\nabla_v^2 p_v(z\mid t)
=
p_v(z\mid t)
\left(
S(z;v)S(z;v)^\top
+
H_{\log}(z;v)
\right).
\end{equation}

\paragraph{Step 2: Conditional Hessian of the approximate-response objective.}

For a fixed horizon realization $t$, define
\begin{equation}
\label{eq:conditional-Phi-def}
\Phi_{\epsilon_\ell,t}(v)
:=
\mathbb E_v
\left[
G(v,\theta(Z))
\mid
\mathcal T=t
\right].
\end{equation}
Equivalently,
\begin{equation}
\label{eq:conditional-Phi-integral}
\Phi_{\epsilon_\ell,t}(v)
=
\int
G(v,\theta(z))
p_v(z\mid t)
\,dz.
\end{equation}

Differentiating~\eqref{eq:conditional-Phi-integral} once yields
\begin{align}
\nabla_v
\Phi_{\epsilon_\ell,t}(v)
=
\int
\Big[
&
\nabla_vG(v,\theta(z))
p_v(z\mid t)
\nonumber\\
&+
G(v,\theta(z))
\nabla_vp_v(z\mid t)
\Big]
\,dz.
\end{align}
Using~\eqref{eq:first-density-derivative},
\begin{equation}
\label{eq:conditional-first-gradient}
\nabla_v
\Phi_{\epsilon_\ell,t}(v)
=
\mathbb E_v
\left[
\nabla_vG(v,\theta(Z))
+
G(v,\theta(Z))S(Z;v)
\mid
\mathcal T=t
\right].
\end{equation}

We now differentiate once more. Since $\theta(z)$ is fixed when $z$ is
fixed, the derivative of
$\nabla_vG(v,\theta(z))$ with respect to $v$ is simply
$\nabla_{vv}^2G(v,\theta(z))$. Applying the product rule gives
\begin{align}
\nabla_v^2
\Phi_{\epsilon_\ell,t}(v)
=
\mathbb E_v
\Big[
&
\nabla_{vv}^2G(v,\theta(Z))
\nonumber\\
&+
\nabla_vG(v,\theta(Z))
S(Z;v)^\top
\nonumber\\
&+
S(Z;v)
\nabla_vG(v,\theta(Z))^\top
\nonumber\\
&+
G(v,\theta(Z))
\big(
S(Z;v)S(Z;v)^\top
+
H_{\log}(Z;v)
\big)
\mid
\mathcal T=t
\Big].
\label{eq:conditional-Phi-hessian}
\end{align}

Taking expectation with respect to $\mathcal T$ gives
\begin{align}
\nabla_v^2
\Phi_{\epsilon_\ell}(v)
=
\mathbb E_v
\Big[
&
\nabla_{vv}^2G(v,\theta(Z))
\nonumber\\
&+
\nabla_vG(v,\theta(Z))
S(Z;v)^\top
\nonumber\\
&+
S(Z;v)
\nabla_vG(v,\theta(Z))^\top
\nonumber\\
&+
G(v,\theta(Z))
\big(
S(Z;v)S(Z;v)^\top
+
H_{\log}(Z;v)
\big)
\Big].
\label{eq:unconditional-Phi-hessian}
\end{align}

We now justify the interchange of differentiation and expectation over the
geometric horizons. Each horizon $T_i\sim\operatorname{Geo}(1-\gamma)$ has
finite moments of every order. Since the complete lower-level run contains
only the finite collection
$\mathcal T=(T_1,\ldots,T_{3K(N+1)})$ and its total number of encoding events
satisfies
\[
M(\mathcal T)
\leq
2\sum_{i=1}^{3K(N+1)}(T_i+1),
\]
we have
\[
\mathbb E[M(\mathcal T)^2]<\infty.
\]
Let
\[
\kappa
:=
\max\left\{\frac{R_1}{R},\frac{C_1}{C}\right\}.
\]
The uniform bound $\|s_j\|\leq\kappa$ on each encoding-score increment gives
\[
\|S(Z;v)\|
\leq
\kappa M(\mathcal T),
\qquad
\|S(Z;v)S(Z;v)^\top\|
\leq
\kappa^2M(\mathcal T)^2.
\]
Assumption~\ref{ass:outer-signal} also gives the uniform eventwise
log-likelihood Hessian bound, where $H_j$ denotes the reward or constraint
log-likelihood Hessian increment at event $j$,
\[
\|H_j\|
\leq
\kappa_H
:=
\max\left\{
\frac{R_2}{R}+\left(\frac{R_1}{R}\right)^2,
\frac{C_2}{C}+\left(\frac{C_1}{C}\right)^2
\right\},
\]
and hence
$\|H_{\log}(Z;v)\|\leq\kappa_HM(\mathcal T)$. Writing
$C_G:=\max\{|G_{\min}|,|G_{\max}|\}$, the norm of the integrand in
\eqref{eq:conditional-Phi-hessian} is therefore bounded by
\[
L_G
+2L_{G,0}\kappa M(\mathcal T)
+C_G\left(
\kappa^2M(\mathcal T)^2
+\kappa_HM(\mathcal T)
\right).
\]
This dominating function is integrable with respect to $\mathcal T$ because
$\mathbb E[M(\mathcal T)^2]<\infty$. Dominated convergence therefore
justifies interchanging $\nabla_v^2$ and the expectation over the geometric
horizons in~\eqref{eq:unconditional-Phi-hessian}.

\paragraph{Step 3: A normalization identity for the score Hessian term.}

For every fixed realization $\mathcal T=t$,
\begin{equation}
\label{eq:conditional-prob-normalization}
\int
p_v(z\mid t)
\,dz
=
1.
\end{equation}
Differentiating~\eqref{eq:conditional-prob-normalization} twice and using
\eqref{eq:second-density-derivative-full} yields
\begin{equation}
\label{eq:conditional-second-normalization}
\mathbb E_v
\left[
S(Z;v)S(Z;v)^\top
+
H_{\log}(Z;v)
\mid
\mathcal T=t
\right]
=
0.
\end{equation}
Taking expectation over $\mathcal T$ gives
\begin{equation}
\label{eq:unconditional-second-normalization}
\mathbb E_v
\left[
S(Z;v)S(Z;v)^\top
+
H_{\log}(Z;v)
\right]
=
0.
\end{equation}

Define the midpoint of the range of the upper objective by
\begin{equation}
\label{eq:G-midpoint}
G_{\mathrm{mid}}
:=
\frac{G_{\min}+G_{\max}}{2}.
\end{equation}
Since
\eqref{eq:unconditional-second-normalization} is zero,
\begin{align}
&
\mathbb E_v
\left[
G(v,\theta(Z))
\left(
S(Z;v)S(Z;v)^\top
+
H_{\log}(Z;v)
\right)
\right]
\nonumber\\
&=
\mathbb E_v
\left[
\left(
G(v,\theta(Z))-G_{\mathrm{mid}}
\right)
\left(
S(Z;v)S(Z;v)^\top
+
H_{\log}(Z;v)
\right)
\right].
\label{eq:center-upper-objective}
\end{align}

Furthermore,
\begin{equation}
\label{eq:G-midpoint-bound}
\left|
G(v,\theta(Z))-G_{\mathrm{mid}}
\right|
\leq
\frac{\Delta_G}{2}.
\end{equation}

Substituting~\eqref{eq:center-upper-objective} into
\eqref{eq:unconditional-Phi-hessian}, we obtain
\begin{align}
\nabla_v^2
\Phi_{\epsilon_\ell}(v)
=
\mathbb E_v
\Big[
&
\nabla_{vv}^2G(v,\theta(Z))
\nonumber\\
&+
\nabla_vG(v,\theta(Z))
S(Z;v)^\top
\nonumber\\
&+
S(Z;v)
\nabla_vG(v,\theta(Z))^\top
\nonumber\\
&+
\left(
G(v,\theta(Z))-G_{\mathrm{mid}}
\right)
\left(
S(Z;v)S(Z;v)^\top
+
H_{\log}(Z;v)
\right)
\Big].
\label{eq:centered-Phi-hessian}
\end{align}

\paragraph{Step 4: Bound each term in the Hessian.}

Taking the operator norm on both sides of
\eqref{eq:centered-Phi-hessian} and using the triangle inequality gives
\begin{align}
\left\|
\nabla_v^2
\Phi_{\epsilon_\ell}(v)
\right\|
\leq\;&
\mathbb E_v
\left[
\left\|
\nabla_{vv}^2
G(v,\theta(Z))
\right\|
\right]
\nonumber\\
&+
\mathbb E_v
\left[
\left\|
\nabla_vG(v,\theta(Z))
S(Z;v)^\top
\right\|
\right]
\nonumber\\
&+
\mathbb E_v
\left[
\left\|
S(Z;v)
\nabla_vG(v,\theta(Z))^\top
\right\|
\right]
\nonumber\\
&+
\mathbb E_v
\Big[
\left|
G(v,\theta(Z))-G_{\mathrm{mid}}
\right|
\nonumber\\
&\hspace{22mm}\times
\left\|
S(Z;v)S(Z;v)^\top
+
H_{\log}(Z;v)
\right\|
\Big].
\label{eq:Phi-hessian-four-terms}
\end{align}

Because $\nabla_vG(\cdot,\theta)$ is $L_G$-Lipschitz in $v$, wherever
$\nabla_{vv}^2G$ exists,
\begin{equation}
\label{eq:upper-G-Hessian-bound}
\left\|
\nabla_{vv}^2G(v,\theta(Z))
\right\|
\leq
L_G.
\end{equation}
Therefore,
\begin{equation}
\label{eq:first-Hessian-term-bound}
\mathbb E_v
\left[
\left\|
\nabla_{vv}^2G(v,\theta(Z))
\right\|
\right]
\leq
L_G.
\end{equation}

For any two vectors $a$ and $b$,
\[
\|ab^\top\|
=
\|a\|\|b\|.
\]
Using
\[
\left\|
\nabla_vG(v,\theta(Z))
\right\|
\leq
L_{G,0},
\]
we obtain
\begin{equation}
\label{eq:second-Hessian-term-bound}
\mathbb E_v
\left[
\left\|
\nabla_vG(v,\theta(Z))
S(Z;v)^\top
\right\|
\right]
\leq
L_{G,0}
\mathbb E_v
\left[
\|S(Z;v)\|
\right],
\end{equation}
and similarly,
\begin{equation}
\label{eq:third-Hessian-term-bound}
\mathbb E_v
\left[
\left\|
S(Z;v)
\nabla_vG(v,\theta(Z))^\top
\right\|
\right]
\leq
L_{G,0}
\mathbb E_v
\left[
\|S(Z;v)\|
\right].
\end{equation}

For the fourth term, using
\eqref{eq:G-midpoint-bound} and the triangle inequality gives
\begin{align}
&
\mathbb E_v
\Big[
\left|
G(v,\theta(Z))-G_{\mathrm{mid}}
\right|
\left\|
S(Z;v)S(Z;v)^\top
+
H_{\log}(Z;v)
\right\|
\Big]
\nonumber\\
&\leq
\frac{\Delta_G}{2}
\mathbb E_v
\left[
\left\|
S(Z;v)S(Z;v)^\top
\right\|
+
\left\|
H_{\log}(Z;v)
\right\|
\right].
\label{eq:fourth-Hessian-term-intermediate}
\end{align}
Since
\begin{equation}
\label{eq:score-outer-product-norm}
\left\|
S(Z;v)S(Z;v)^\top
\right\|
=
\|S(Z;v)\|^2,
\end{equation}
we have
\begin{align}
&
\mathbb E_v
\Big[
\left|
G(v,\theta(Z))-G_{\mathrm{mid}}
\right|
\left\|
S(Z;v)S(Z;v)^\top
+
H_{\log}(Z;v)
\right\|
\Big]
\nonumber\\
&\leq
\frac{\Delta_G}{2}
\left(
\mathbb E_v
\left[
\|S(Z;v)\|^2
\right]
+
\mathbb E_v
\left[
\|H_{\log}(Z;v)\|
\right]
\right).
\label{eq:fourth-Hessian-term-bound}
\end{align}

Combining
\eqref{eq:first-Hessian-term-bound},
\eqref{eq:second-Hessian-term-bound},
\eqref{eq:third-Hessian-term-bound}, and
\eqref{eq:fourth-Hessian-term-bound} in
\eqref{eq:Phi-hessian-four-terms} yields
\begin{align}
\left\|
\nabla_v^2
\Phi_{\epsilon_\ell}(v)
\right\|
\leq\;&
L_G
+
2L_{G,0}
\mathbb E_v
\left[
\|S(Z;v)\|
\right]
\nonumber\\
&+
\frac{\Delta_G}{2}
\left(
\mathbb E_v
\left[
\|S(Z;v)\|^2
\right]
+
\mathbb E_v
\left[
\|H_{\log}(Z;v)\|
\right]
\right).
\label{eq:Phi-Hessian-before-score-bounds}
\end{align}

\paragraph{Step 5: Substitute the score-moment bounds.}

By Jensen's inequality,
\begin{equation}
\label{eq:score-first-moment-Jensen}
\mathbb E_v
\left[
\|S(Z;v)\|
\right]
\leq
\sqrt{
\mathbb E_v
\left[
\|S(Z;v)\|^2
\right]
}.
\end{equation}
Lemma~\ref{lem:score-moments} gives
\begin{equation}
\label{eq:score-second-moment-recall}
\mathbb E_v
\left[
\|S(Z;v)\|^2
\right]
\leq
\bar I,
\end{equation}
and
\begin{equation}
\label{eq:log-Hessian-moment-recall}
\mathbb E_v
\left[
\|H_{\log}(Z;v)\|
\right]
\leq
\bar J.
\end{equation}
Combining
\eqref{eq:score-first-moment-Jensen} and
\eqref{eq:score-second-moment-recall},
\begin{equation}
\label{eq:score-first-moment-final}
\mathbb E_v
\left[
\|S(Z;v)\|
\right]
\leq
\sqrt{\bar I}.
\end{equation}

Substituting
\eqref{eq:score-second-moment-recall},
\eqref{eq:log-Hessian-moment-recall}, and
\eqref{eq:score-first-moment-final} into
\eqref{eq:Phi-Hessian-before-score-bounds} yields
\begin{equation}
\label{eq:uniform-Phi-Hessian-bound}
\left\|
\nabla_v^2
\Phi_{\epsilon_\ell}(v)
\right\|
\leq
L_G
+
2L_{G,0}\sqrt{\bar I}
+
\frac{\Delta_G}{2}
\left(
\bar I+\bar J
\right)
=
L_\Phi.
\end{equation}

The bound in~\eqref{eq:uniform-Phi-Hessian-bound} holds uniformly in $v$.
Hence, for any $v,v'\in\mathbb R^{d_v}$,
\begin{align}
\nabla\Phi_{\epsilon_\ell}(v)
-
\nabla\Phi_{\epsilon_\ell}(v')
&=
\int_0^1
\nabla_v^2
\Phi_{\epsilon_\ell}
\left(
v'+\tau(v-v')
\right)
(v-v')
\,d\tau.
\end{align}
Taking norms and applying
\eqref{eq:uniform-Phi-Hessian-bound} gives
\begin{equation}
\label{eq:gradient-Lipschitz-conclusion}
\left\|
\nabla\Phi_{\epsilon_\ell}(v)
-
\nabla\Phi_{\epsilon_\ell}(v')
\right\|
\leq
L_\Phi
\|v-v'\|.
\end{equation}
Thus $\nabla\Phi_{\epsilon_\ell}$ is globally
$L_\Phi$-Lipschitz.

\paragraph{Step 6: Extension to $C^{1,1}$ regularity.}

Assumption~\ref{ass:upper-objective} requires only that
$\nabla_vG(\cdot,\theta)$ be globally $L_G$-Lipschitz. Such a function is
twice differentiable almost everywhere, with
\[
\left\|
\nabla_{vv}^2G(v,\theta)
\right\|
\leq
L_G
\]
at every point at which the Hessian exists.

To obtain the global result without requiring pointwise twice
differentiability, let $G_\delta$ be a standard smooth mollification of
$G$ in the upper-level variable $v$. The mollified function satisfies the same
uniform gradient-Lipschitz bound,
\[
\left\|
\nabla_{vv}^2G_\delta(v,\theta)
\right\|
\leq
L_G,
\]
and the same uniform bounds on
$G$ and $\nabla_vG$, up to an arbitrarily small approximation error.
Applying the preceding argument to $G_\delta$ gives
\[
\left\|
\nabla
\Phi_{\epsilon_\ell,\delta}(v)
-
\nabla
\Phi_{\epsilon_\ell,\delta}(v')
\right\|
\leq
L_\Phi
\|v-v'\|
\]
with a constant $L_\Phi$ independent of $\delta$. Passing to the limit
$\delta\downarrow0$ yields
\eqref{eq:gradient-Lipschitz-conclusion} for
$\Phi_{\epsilon_\ell}$.

\paragraph{Step 7: Dependence on the lower-level training length.}

By Lemma~\ref{lem:encoding-counts},
\begin{equation}
\label{eq:I-training-order}
\bar I
=
O
\left(
\frac{K(N+1)}{1-\gamma}
\right),
\end{equation}
and
\begin{equation}
\label{eq:J-training-order}
\bar J
=
O
\left(
\frac{K(N+1)}{1-\gamma}
\right).
\end{equation}
Consequently,
\begin{equation}
\label{eq:sqrt-I-training-order}
\sqrt{\bar I}
=
O
\left(
\sqrt{
\frac{K(N+1)}{1-\gamma}
}
\right).
\end{equation}
Substituting
\eqref{eq:I-training-order}--\eqref{eq:sqrt-I-training-order}
into~\eqref{eq:L-Phi} gives
\begin{align}
L_\Phi
&=
L_G
+
O
\left(
\sqrt{
\frac{K(N+1)}{1-\gamma}
}
\right)
+
O
\left(
\frac{K(N+1)}{1-\gamma}
\right)
\nonumber\\
&=
O
\left(
1+
\frac{K(N+1)}{1-\gamma}
\right),
\end{align}
which proves~\eqref{eq:Lphi-order}.
\end{proof}

\begin{lemma}[Unbiased upper-gradient estimator]
\label{lem:upper-estimator}
Under Assumption~\ref{ass:upper-objective} and the hypotheses of
Lemmas~\ref{lem:training-gradient}, \ref{lem:score-moments}, and
\ref{lem:encoding-counts}, let
$\widehat{\delta}_G(Z;v)$ be the value-difference estimator defined in
Lemma~\ref{lem:value-difference}, and let $S(Z;v)$ be the complete score
defined in Lemma~\ref{lem:score-moments}. Using an independent fresh
$v$-gradient oracle call, define
\begin{equation}
\label{eq:ghat-upper}
\widehat g(v):=\widehat\nabla_vG(v,\theta(Z))
+\widehat{\delta}_G(Z;v)S(Z;v).
\end{equation}
Let $\bar I$ be the uniform score second-moment bound from
Lemma~\ref{lem:score-moments}, let
$V_\theta:=M_G^2+\sigma_\theta^2$, and define
\begin{equation}
\label{eq:V-Phi}
V_\Phi:=2(L_{G,0}^2+\sigma_v^2)
+2(\Delta_G^2+V_\theta)\bar I.
\end{equation}
Then
\begin{equation}
\label{eq:ghat-unbiased}
\mathbb E[\widehat g(v)]=\nabla\Phi_{\epsilon_\ell}(v),
\end{equation}
\begin{equation}
\label{eq:ghat-second}
\mathbb E\|\widehat g(v)\|^2\leq V_\Phi,
\end{equation}
and
\begin{equation}
\label{eq:Vphi-order}
V_\Phi=O\left(1+\frac{K(N+1)}{1-\gamma}\right),
\end{equation}
where
$\Phi_{\epsilon_\ell}(v)=\mathbb E_{Z\sim P_v}[G(v,\theta(Z))]$.
Expectations include the lower trace and the fresh upper-oracle randomness.
\end{lemma}

\begin{proof}
Fix $v$ and let $Z\sim P_v$ be a complete lower training trace.
For the vector estimator, $S(Z;v)$ is fixed conditionally on $Z$.
The fresh $v$-gradient oracle and
Lemma~\ref{lem:value-difference} give
\[
\mathbb E[\widehat g(v)\mid Z]
=\nabla_vG(v,\theta(Z))
+\bigl(G(v,\theta(Z))-G(v,\theta_0)\bigr)S(Z;v).
\]
Taking expectation over $Z$ and applying
Lemma~\ref{lem:training-gradient} with baseline $c(v)=G(v,\theta_0)$
proves~\eqref{eq:ghat-unbiased}.

The inequality $\|u+w\|^2\leq2\|u\|^2+2\|w\|^2$ gives
\[
\begin{aligned}
\mathbb E[\|\widehat g(v)\|^2\mid Z]
&\leq2\mathbb E[\|\widehat\nabla_vG(v,\theta(Z))\|^2\mid Z]
+2\mathbb E[\widehat{\delta}_G^2\mid Z]\|S(Z;v)\|^2\\
&\leq2(L_{G,0}^2+\sigma_v^2)
+2(\Delta_G^2+V_\theta)\|S(Z;v)\|^2.
\end{aligned}
\]
Here the first oracle second moment is bounded by its squared mean plus
its variance. The scalar second-moment bound follows
from~\eqref{eq:A-second}. Taking expectation
and using $\mathbb E\|S(Z;v)\|^2\leq\bar I$ proves the second-moment
bound in~\eqref{eq:ghat-second}. Unbiasedness then implies
\[
\mathbb E\|\widehat g(v)-\nabla\Phi_{\epsilon_\ell}(v)\|^2
=\mathbb E\|\widehat g(v)\|^2-\|\nabla\Phi_{\epsilon_\ell}(v)\|^2
\leq V_\Phi.
\]
Finally, Lemma~\ref{lem:encoding-counts} gives
$\bar I=O(K(N+1)/(1-\gamma))$. Substitution into the definition of
$V_\Phi$ proves~\eqref{eq:Vphi-order}.
\end{proof}

\subsection{Upper-Level Stationarity and End-to-End Guarantee}
\label{app:proof-upper-end}
\textbf{Upper-Level Convergence}

\begin{lemma}[Convergence to upper-level stationarity]
\label{lem:upper-convergence}
Suppose the hypotheses of Lemmas~\ref{lem:upper-smoothness} and
\ref{lem:upper-estimator} hold uniformly in $v$. Let $L_\Phi$ and
$V_\Phi$ be the smoothness and estimator second-moment constants from these
lemmas, respectively, and define
\[
\bar{L}_\Phi:=\max\{1,L_\Phi\},
\qquad
\bar{V}_\Phi:=\max\{1,V_\Phi\}.
\]
Starting from $v_0$, generate
\[
v_{t+1}:=v_t-\alpha\widehat g_t,
\]
where, conditionally on $v_t$, $\widehat g_t$ is an independent realization
of the upper-gradient estimator defined in
Lemma~\ref{lem:upper-estimator}.
For a desired upper accuracy $\epsilon_o\in(0,1]$, choose
\begin{equation}
\label{eq:outer-alpha}
\alpha
=
\frac{
\epsilon_o
}{
\bar{L}_\Phi\bar{V}_\Phi
},
\end{equation}
and
\begin{equation}
\label{eq:outer-T}
T_o
=
\max
\left\{
1,
\left\lceil
\frac{
2\Delta_G
\bar{L}_\Phi
\bar{V}_\Phi
}{
\epsilon_o^2
}
\right\rceil
\right\}.
\end{equation}
Let
\[
R_o
\sim
\operatorname{Unif}
\{0,\ldots,T_o-1\}
\]
be selected independently of all previous randomness. Then
\begin{equation}
\label{eq:outer-stationarity}
\mathbb E
\left[
\|
\nabla
\Phi_{\epsilon_\ell}(v_{R_o})
\|^2
\right]
\leq
\epsilon_o.
\end{equation}
\end{lemma}

\begin{proof}
By Lemma~\ref{lem:upper-smoothness},
$\Phi_{\epsilon_\ell}$ is $L_\Phi$-smooth. Therefore,
\[
\Phi_{\epsilon_\ell}(v_{t+1})
\leq
\Phi_{\epsilon_\ell}(v_t)
+
\left\langle
\nabla\Phi_{\epsilon_\ell}(v_t),
v_{t+1}-v_t
\right\rangle
+
\frac{L_\Phi}{2}
\|v_{t+1}-v_t\|^2.
\]
Using
\[
v_{t+1}-v_t
=
-\alpha\widehat g_t,
\]
we obtain
\[
\Phi_{\epsilon_\ell}(v_{t+1})
\leq
\Phi_{\epsilon_\ell}(v_t)
-
\alpha
\left\langle
\nabla\Phi_{\epsilon_\ell}(v_t),
\widehat g_t
\right\rangle
+
\frac{L_\Phi\alpha^2}{2}
\|\widehat g_t\|^2.
\]
Taking conditional expectation given $v_t$ and applying
Lemma~\ref{lem:upper-estimator},
\[
\mathbb E_t[
\Phi_{\epsilon_\ell}(v_{t+1})
]
\leq
\Phi_{\epsilon_\ell}(v_t)
-
\alpha
\|
\nabla\Phi_{\epsilon_\ell}(v_t)
\|^2
+
\frac{
L_\Phi\alpha^2
}{
2
}
V_\Phi.
\]
Taking total expectation and summing over
$t=0,\ldots,T_o-1$ gives
\[
\alpha
\sum_{t=0}^{T_o-1}
\mathbb E
\|
\nabla\Phi_{\epsilon_\ell}(v_t)
\|^2
\leq
\mathbb E[
\Phi_{\epsilon_\ell}(v_0)
-
\Phi_{\epsilon_\ell}(v_{T_o})
]
+
\frac{
L_\Phi\alpha^2V_\Phi T_o
}{
2
}.
\]
Since
\[
G_{\min}
\leq
\Phi_{\epsilon_\ell}(v)
\leq
G_{\max},
\]
we have
\[
\Phi_{\epsilon_\ell}(v_0)
-
\Phi_{\epsilon_\ell}(v_{T_o})
\leq
\Delta_G.
\]
Consequently,
\begin{equation}
\label{eq:average-upper-gradient}
\frac1{T_o}
\sum_{t=0}^{T_o-1}
\mathbb E
\|
\nabla\Phi_{\epsilon_\ell}(v_t)
\|^2
\leq
\frac{
\Delta_G
}{
\alpha T_o
}
+
\frac{
L_\Phi\alpha V_\Phi
}{
2
}.
\end{equation}

By~\eqref{eq:outer-T},
\[
\frac{
\Delta_G
}{
\alpha T_o
}
\leq
\frac{\epsilon_o}{2}.
\]
Furthermore, using~\eqref{eq:outer-alpha},
\[
\frac{
L_\Phi\alpha V_\Phi
}{
2
}
\leq
\frac{\epsilon_o}{2}.
\]
Thus
\[
\frac1{T_o}
\sum_{t=0}^{T_o-1}
\mathbb E
\|
\nabla\Phi_{\epsilon_\ell}(v_t)
\|^2
\leq
\epsilon_o.
\]
Uniform independent selection of $R_o$ yields
\[
\mathbb E
\|
\nabla\Phi_{\epsilon_\ell}(v_{R_o})
\|^2
=
\frac1{T_o}
\sum_{t=0}^{T_o-1}
\mathbb E
\|
\nabla\Phi_{\epsilon_\ell}(v_t)
\|^2,
\]
which proves~\eqref{eq:outer-stationarity}.
\end{proof}

\textbf{End-to-End Guarantee and Sample Complexity}

We combine the upper stationarity guarantee, the approximate lower-level
response guarantee, and the resulting sample requirements in one statement.

\begin{theorem}[End-to-end constrained bilevel guarantee and complexity]
\label{thm:end-to-end-app}
Suppose the hypotheses of Theorem~\ref{thm:lower-sgd} and
Lemma~\ref{lem:upper-convergence} hold uniformly in $v$. For
$\delta_\ell,\epsilon_o\in(0,1]$, set
$K=\lceil\delta_\ell^{-2}\rceil$ and
$\epsilon_\ell:=\max\{\sqrt{\epsilon_{\rm bias}},\delta_\ell\}$. Take $N$
to be the smallest positive
integer satisfying the critic-iteration condition in
Theorem~\ref{thm:lower-sgd}, and use the remaining lower and upper parameter
choices stated in Theorem~\ref{thm:lower-sgd} and
Lemma~\ref{lem:upper-convergence}. Let
$(v_{R_o},\widehat\theta_{\rm out})$ be the output of
Algorithm~\ref{alg:bilevel-nonpert}, including its final independent lower
run. Conditional on $v_{R_o}$, denote the trace of this fresh run by
$Z_{\rm out}\sim P_{v_{R_o}}$, so that
$\widehat\theta_{\rm out}=\theta(Z_{\rm out})$. Then
\begin{equation}
\label{eq:end-upper}
\mathbb E\|\nabla\Phi_{\epsilon_\ell}(v_{R_o})\|^2
\leq
\epsilon_o,
\end{equation}
\begin{equation}
\label{eq:end-reward}
\mathbb E\left[
J_r^\star(v_{R_o})
-J_r(\widehat\theta_{\rm out},x_{R_o})
\right]
\leq
\rev{D_r}\epsilon_\ell,
\end{equation}
and
\begin{equation}
\label{eq:end-constraint}
\mathbb E\left[-J_c(\widehat\theta_{\rm out},y_{R_o})\right]
\leq
\rev{D_c}\epsilon_\ell.
\end{equation}
Here $\rev{D_r}:=1+C_r$ and
$\rev{D_c}:=C_{\rm sl}+C_c$, where $C_r$, $C_c$, and $C_{\rm sl}$ are
the constants in
Theorem~\ref{thm:lower-sgd}, and $\epsilon_{\rm bias}$ is defined in
Assumption~\ref{ass:compatibility}.

Suppose additionally that each upper-gradient oracle call uses at most $2H$
environment transitions. Let $N_{\rm env}$ denote the expected total number of environment
transitions used by Algorithm~\ref{alg:bilevel-nonpert}, including both lower
training and upper-gradient sampling. Then, for fixed primitive problem
constants,
\begin{equation}
\label{eq:final-env-complexity}
N_{\rm env}
=
\widetilde{\mathcal{O}}(\delta_\ell^{-6}\epsilon_o^{-2}).
\end{equation}
\end{theorem}

\begin{proof}
The choices of $\eta$, $\zeta$, $N$, and $\lambda_{\max}$ satisfy the
conditions of Theorem~\ref{thm:lower-sgd}, while the choices of $\alpha$ and
$T_o$ satisfy the conditions of Lemma~\ref{lem:upper-convergence}.
Therefore Lemma~\ref{lem:upper-convergence} gives
\eqref{eq:end-upper}.

Condition on $v_{R_o}$. The deployment parameter
$\widehat\theta_{\rm out}=\theta(Z_{\rm out})$ is generated by an independent
fresh trace $Z_{\rm out}\sim P_{v_{R_o}}$ at this fixed upper-level parameter.
Uniform application of
Theorem~\ref{thm:lower-sgd} gives
\[
\mathbb E\left[
J_r^\star(v_{R_o})-J_r(\widehat\theta_{\rm out},x_{R_o})
\mid v_{R_o}
\right]
\leq
\sqrt{\epsilon_{\rm bias}}+\frac{C_r}{\sqrt K},
\]
and
\[
\mathbb E\left[
-J_c(\widehat\theta_{\rm out},y_{R_o})
\mid v_{R_o}
\right]
\leq
C_{\rm sl}\sqrt{\epsilon_{\rm bias}}+\frac{C_c}{\sqrt K}.
\]
Because $K^{-1/2}\leq\delta_\ell\leq\epsilon_\ell$ and
$\sqrt{\epsilon_{\rm bias}}\leq\epsilon_\ell$, taking expectation over
$v_{R_o}$ proves
\eqref{eq:end-reward}--\eqref{eq:end-constraint}.

For the environment-interaction count, define
\[
Q_\ell:=\frac{K(N+1)}{1-\gamma}.
\]
Let $M_\ell$ be a uniform upper bound on the expected number of environment
transitions in one complete lower run. Also let $N_{\rm lower-env}$ denote
the expected transitions used by all lower runs and let $N_{\rm upper}$
denote the expected number of upper-gradient oracle calls.
Each complete lower run makes $K(N+1)$ encoded critic calls. Each call uses
a constant number of geometric-horizon trajectories whose expected lengths
are $O((1-\gamma)^{-1})$, so $M_\ell=O(Q_\ell)$. Lemmas
\ref{lem:upper-smoothness}, \ref{lem:upper-estimator}, and
\ref{lem:encoding-counts} give
\[
L_\Phi=O(1+Q_\ell),
\qquad
V_\Phi=O(1+Q_\ell).
\]
Substitution into the definition of $T_o$ in
Lemma~\ref{lem:upper-convergence} gives
\[
T_o=O\left((1+Q_\ell)^2\epsilon_o^{-2}\right).
\]

The algorithm performs one lower run at each of the $T_o$ upper-level iterations
and one fresh deployment run after selecting $R_o$. Hence
\[
N_{\rm lower-env}\leq M_\ell(T_o+1),
\]
and therefore
\[
N_{\rm lower-env}
=O\left(Q_\ell+Q_\ell(1+Q_\ell)^2\epsilon_o^{-2}\right).
\]

For $a\geq0$, $\max\{1,\lceil a\rceil\}\leq1+a$. The stated adaptive batch
size and Lemma~\ref{lem:response-displacement} therefore give
\[
\mathbb E[B_G(Z)]
\leq
1+\mathbb E\|\theta(Z)-\theta_0\|^2
\leq
1+\bar D^2.
\]
Each upper-level iteration uses $B_G(Z)$ fresh $\theta$-gradient calls and one
fresh $v$-gradient call. Therefore the expected number of calls per upper-level iteration is at most $2+\bar D^2$, and hence
\[
N_{\rm upper}
\leq
(2+\bar D^2)T_o.
\]
By the assumed upper-oracle cost bound, each such call uses at most $2H$
environment transitions. Lemma~\ref{lem:bt-gradient-oracles} (in Appendix)
provides one concrete realization through policy-generated Bradley--Terry
samples formed from two trajectories of length at most $H$. Consequently,
\begin{equation}
\label{eq:total-env-decomposition}
N_{\rm env}
\leq
N_{\rm lower-env}+2H N_{\rm upper}.
\end{equation}

Finally, the stated choices give
\[
K=\Theta(\delta_\ell^{-2}),
\qquad
N=\Theta(\log(\delta_\ell^{-1})),
\qquad
Q_\ell=\widetilde{\mathcal{O}}(\delta_\ell^{-2}).
\]
Lemma~\ref{lem:response-displacement} gives
$\bar D^2\leq\eta_0^2KW_\omega^2=O(\delta_\ell^{-2})$.
Thus both $N_{\rm lower-env}$ and $2H N_{\rm upper}$ are
$\widetilde{\mathcal{O}}(\delta_\ell^{-6}\epsilon_o^{-2})$. Combining these bounds
in~\eqref{eq:total-env-decomposition} proves
\eqref{eq:final-env-complexity}.
\end{proof}

\begin{remark}[No differentiation through the lower-level algorithm]
\label{rem:no-unrolling}
None of the bounds above requires differentiation through the full sequence of
lower-level updates, including the projected multiplier update, or an inverse
Hessian of a lower-level saddle system. The dependence on the length of lower training
appears only through the score moments
$\bar I$ and $\bar J$, which grow linearly with the expected number of
encoded signals.
\end{remark}

\end{document}